%% file: main.tex
\documentclass[runningheads]{llncs}
\usepackage[T1]{fontenc}
\usepackage{graphicx}
\usepackage{booktabs}
\usepackage[misc]{ifsym}
\newcommand{\corr}{(\Letter)}
\usepackage{amsmath, bm}
\usepackage{amssymb}

\usepackage{parskip}
\usepackage{mathrsfs}
\usepackage{mathabx}
\usepackage{dsfont}
\usepackage{mathtools}
\usepackage{stmaryrd}
\usepackage{xfrac}
\usepackage{derivative}

\usepackage{tabularx, multirow}
\usepackage{subcaption}
\usepackage{hyperref}
\usepackage{xspace}

\usepackage{duckuments}
\usepackage{pgfplots}
\usepackage{scalefnt}
\pgfplotsset{compat=1.18}

\usepackage[export]{adjustbox}
\usepackage[rawfloats=true]{floatrow}
\usepackage{subcaption}
\usepackage[outline]{contour} 

\usepackage{tikz}

\usepackage{listofitems}
\usetikzlibrary{fpu}
\usetikzlibrary{calc} 
\usepackage[skins]{tcolorbox} 
\usetikzlibrary{shadows.blur}
\usetikzlibrary{fit}
\usepackage{scalefnt}
\usepackage{mwe}
\usetikzlibrary{shapes.geometric}
\usetikzlibrary{patterns, positioning,fit}

\usepackage[square, numbers]{natbib}
\setcitestyle{aysep={,},citesep={,}}

\input{main-setup.tex}

\begin{document}
\title{CAM-Based Methods Can See through Walls}

\author{Magamed Taimeskhanov\inst{1} \corr \and
Ronan Sicre\inst{2} \and
Damien Garreau\inst{3}}
\institute{
Université Côte d'Azur, Laboratoire J.A. Dieudonn\'e, CNRS, Nice, France \email{magamed.taimeskhanov@etu.univ-cotedazur.fr}
\and
Centrale M\'editerran\'ee, Aix-Marseille Univ., CNRS, LIS, Marseille, France \email{ronan.sicre@lis-lab.fr}
\and
Julius-Maximilians Universit\"at W\"urzburg, Germany \email{damien.garreau@uni-wuerzburg.de}
}

\authorrunning{M. Taimeskhanov et al.}

\tocauthor{Magamed Taimeskhanov, Ronan Sicre, Damien Garreau}
\toctitle{CAM-Based Methods Can See through Walls}

\maketitle              
\begin{abstract}
CAM-based methods are widely-used post-hoc interpretability method that produce a saliency map to explain the decision of an image classification model. 
The saliency map highlights the important areas of the image relevant to the prediction. 
In this paper, we show that most of these methods can incorrectly attribute an important score to parts of the image that the model cannot see. 
We show that this phenomenon occurs both theoretically and experimentally.
On the theory side, we analyze the behavior of GradCAM on a simple masked CNN model at initialization. 
Experimentally, we train a VGG-like model constrained to not use the lower part of the image and  nevertheless observe positive scores in the unseen part of the image. 
This behavior is evaluated quantitatively on two new datasets. 
We believe that this is problematic, potentially leading to mis-interpretation of the model's behavior. 
\end{abstract}

\section{Introduction}
\label{sec:introduction}

The recent advances of machine learning pervade all applications, including the most critical.  
However, deep learning models intrinsically possess many parameters, have complicated architectures, and rely on many non-linear operations, preventing the users to get a good grasp of the rationale behind particular decisions. 
These models are often called ``black boxes'' for these reasons \citep{benitez_et_al_1997}.  
In this respect, there is a growing need for interpretability of the models that are used, which gave birth to the field of eXplainable AI (XAI).
When the model to explain is already trained, our main topic of interest, this is often called post-hoc interpretability~\citep{lipton18,zhang2021survey,linardatos_et_al_2021}. 

In the specific case of image classification, the explanations provided to the user often take the form of a saliency map superimposed to the original image, for instance simply looking at the gradient with respect to the input of the network \citep{simonyan_et_al_2013}. 
The message is simple: the areas highlighted by the saliency maps are used by the network for the prediction. 
When the first layers of the network are convolutional layers \citep{fukushima_1980}, one can take advantage of this and look at the activations of the filters corresponding to the class prediction that we are trying to explain. 
Indeed, these first layers act like a bank of filters on the input image, and the degree to which they are activated gives us information on the behavior of the network. 
Thus the first layers possess a certain degree of interpretability, even though it can be challenging to aggregate the information coming from different filters.
In any case, the next layers generally consist in a fully-connected neural network, thus suffering from the same caveats as other models. 
In addition, this second part of the network is equally important for the prediction, but is not taken into account in the explanations we provide if we simply look at activation values.

To solve this problem, a natural idea is to weight each activation map depending on how the second part of the network uses it. 
In the case of a single additional layer, this is called \emph{class activation maps} \citep[CAM,][]{zhou_et_al_2016}. 
The methodology was quickly generalized by \cite{selvaraju_et_al_2017}, using the \emph{average gradient} values of the subsequent layers instead, giving rise to GradCAM, arguably one of the most popular posthoc interpretability method for CNNs. 
Many extensions are proposed in the following years, we list them in Appendix~\ref{app:defs} and refer to \cite{zhang_et_al_2023} for a recent survey. 
Without being too technical, for all these methods, the explanations provided consist in a weighted average of the activation maps.

A close inspection of each of these methods reveals that the coefficient associated to each individual map is global, in the sense that the same coefficient is applied to the whole map. 
The main message of this paper is that this can be problematic, since different parts of the activation map may be used differently by the subsequent layers. 
Worse, \textbf{some parts may even be unused by the subsequent network and still highlighted in the final explanation} (see Figure~\ref{fig:gradcam-fail-panel}). 
Thus we believe that, while giving apparently more-than-satisfying results in practice, CAM-based methods should be used with caution, keeping in mind that some parts of the image may be highlighted whereas they are not even seen by the network. 

%%%%%%%%%%%%%%%%%%%%%%%%%%%%%%%%%%%%%%%%%%%%%%%%%%

\begin{figure}[t]
  \floatsetup{heightadjust=all, valign=c}
  \begin{floatrow}
    \ffigbox[\FBwidth]{
      \begin{adjustbox}{valign=c}
        \includegraphics[width=0.34\textwidth]{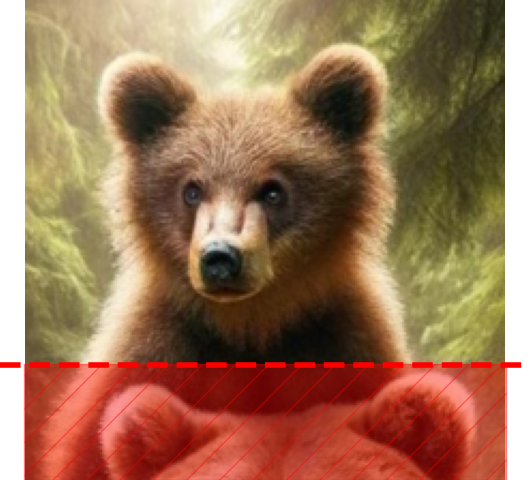}
      \end{adjustbox}
    }{
   
      \label{fig:gradcam-fail-img}
    }
 
    \ffigbox[\FBwidth]{
      \begin{adjustbox}{valign=c}
        \includegraphics[width=0.14\textwidth]{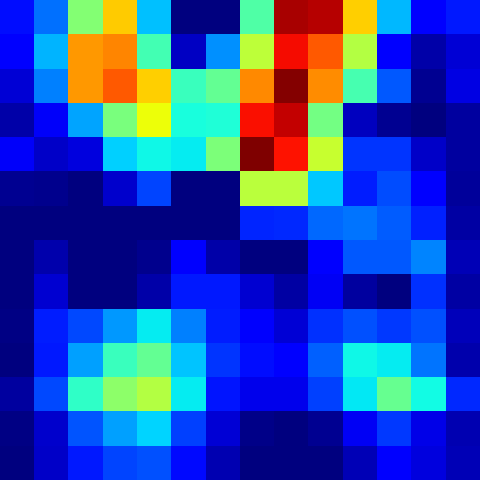}
      \end{adjustbox}
    }{
     
      \label{fig:gradcam-fail}
    }
    \hspace*{2.4mm} 
    \ffigbox[\FBwidth]{
      \begin{adjustbox}{valign=c}
        \includegraphics[width=0.34\textwidth]{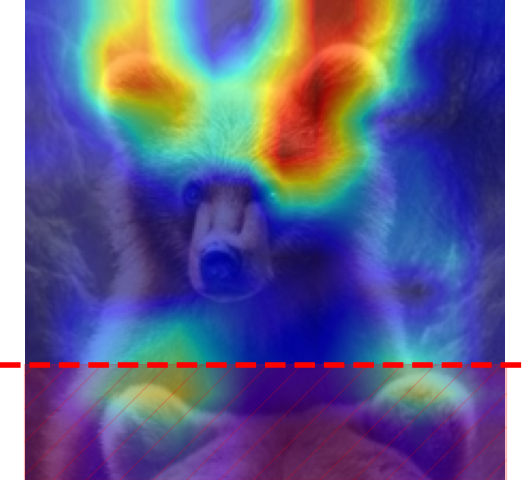}
      \end{adjustbox}
    }{
      
      \label{fig:gradcam-upscale-fail}
    }
  \end{floatrow}
\caption{\label{fig:gradcam-fail-panel}Example of GradCAM failure on a VGG-like model trained on the ImageNet dataset (masked $\abrev{VGG}$, see Figure~\ref{fig:vgg}). 
\emph{Left:} original image; 
\emph{Middle:} GradCAM explanation before up-sampling; \emph{Right:} original image with GradCAM explanation overlayed as a heatmap. The network does not have access to the red part of the image, \textbf{but GradCAM does highlight some pixels in this area.}
}
\end{figure}

%%%%%%%%%%%%%%%%%%%%%%%%%%%%%%%%%%%%%%%%%%%%%%%%%%%%%%%%%%%%%%%%%%%%%%%%%

\subsection{Related work}

This paper is inspired by a line of recent works concerned with the reliability of saliency maps claiming that solely relying on the visual explanation provided by a saliency map can be misleading \citep{kindermans2019reliability, ghorbani2019interpretation}. 
\cite{ghorbani2019interpretation} introduces a method for altering the input data with imperceptible perturbations which do not change the predicted label, yet generating different saliency maps.
On the other hand, \cite{kindermans2019reliability} shows that numerous saliency methods generate incorrect scores for the input features when the model prediction is invariant to translation of the input data by a constant. 
It is important to note that neither of these studies specifically challenges the reliability of CAM-based methods.

This perspective on saliency maps is supported by the work of \cite{adebayo2018}, which introduces a randomization-based sanity check indicating that some existing saliency methods are independent of both the model and the data. 
We note that GradCAM passes the sanity checks proposed by \cite{adebayo2018}. 
\cite{draelos_carin_2021}, proposing HiResCAM, are less positive regarding GradCAM pointing out, as we do, that the use of a global coefficient can produce positive explanations where there should not be. 
Compared to our work, they provide few theoretical explanations and perform experiments on model which are not using parts of the input image. 
Posthoc interpretability methods in the image realm (not specific to CNN architectures) have been investigated by other works such as \cite{garreau_mardaoui_2021} which looked into LIME for images \citep{ribeiro_et_al_2016}.

Taking another angle, \cite{heo2019} directly attacks the reliability of GradCAM saliency maps by adversarial model manipulation, \emph{i.e.}, fine-tuning a model with the purpose of making GradCAM saliency maps unreliable. 
This is achieved by using a specific loss function tailored to this effect. 
Our approach is different, as we simply force a strong form of sparsity in the model's parameters, not targeting a specific interpretability method.

%%%%%%%%%%%%%%%%%%%%%%%%%%%%%%%%%%%%%%%%%%%%%%%%%%%%%%%%%%%%%%%%%%%%%%%%%%

\subsection{Organization of the paper}

We start by looking at GradCAM in Section~\ref{sec:maths-description}. 
For a given simple CNN architecture described in Section~\ref{sec:setting}, we derive closed-form expressions for its explanations in Section~\ref{sec:computation}. 
Leveraging these expressions, we prove in Section~\ref{sec:theory} that GradCAM explanations are positive at initialization, even though a large part of the weights are set to zero. 

In Section~\ref{sec:experiments}, we demonstrate experimentally that this phenomenon remains true after training. 
To this extent, we proceed in two steps. 
First, we train to a reasonable accuracy a VGG-like model on ImageNet~\citep{dengimagenet2009} \textbf{which does not see the lower part of input images}, described in Section~\ref{sec:model}.
Then, we create two datasets (Section~\ref{sec:datasets}) consisting in superposition of images of the same class. 
We show experimentally in Section~\ref{sec:results} that \textbf{CAM-based methods applied to this model wrongly highlights a large portion of the lower part of the images}, misleading the user by showing that the lower part is used for the prediction whereas, by construction it is not. 
The code for training our model as well as the datasets are provided as supplementary material. Additionally, the code for all experiments is available
online.\footnote{\url{https://github.com/MagamedT/cam-can-see-through-walls}}
We conclude in Section~\ref{sec:conclusion}. 

%%%%%%%%%%%%%%%%%%%%%%%%%%%%%%%%%%%%%%%%%%%%%%%%%%%%%%%%%%%%%%%%%%%%%

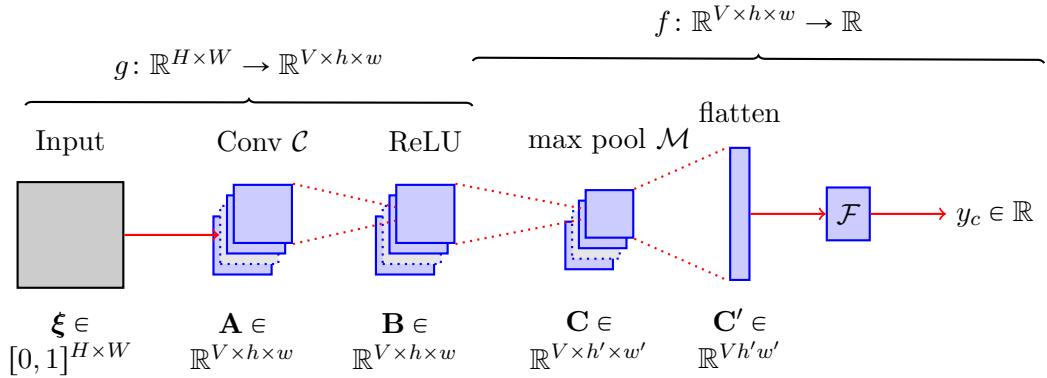
\begin{figure}[t]
\centering
\hspace*{-\dimexpr\oddsidemargin+1in\relax}\makebox[\paperwidth]{%
\input{figures/tikz/model_mnist1}} \hspace*{-\paperwidth}    
\caption{\label{fig:model}The model used for the derivation of feature importance scores, $\Model$. The number of filters in the convolutional layer~$\mathcal{C}$ is $V \in \Posint$. The size of the max pooling filters $\Pshp \in \Posint$ is implicitly defined such that $(\Mshp, \Mshpp) = \frac{1}{\Pshp} (\Ashp,\Ashpp)$ in $\Posint$. 
The fully-connected neural network $\fnn{\cdot}$ takes $\Cbf'$ as input and processes it through $L$ layers with ReLU activation functions to produce a raw score $y_c$, without converting this score into a ``probability.'' 
}
\end{figure}

%%%%%%%%%%%%%%%%%%%%%%%%%%%%%%%%%%%%%%%%%%%%%%%%%%%%%%%%%%%%%%%%%%%%

\section{Mathematical description}
\label{sec:maths-description}

The model used for the theoretical analysis done in Section~\ref{sec:theory} is described in Section~\ref{sec:setting}, the derivation of GradCAM coefficients in Section~\ref{sec:computation}. 

%%%%%%%%%%%%%%%%%%%%%%%%%%%%%%%%%%%%%%%%%%%%%%%%%%%%%%%%%%%%%%%%%%%%%%

\subsection{A simple CNN}
\label{sec:setting}

Let us describe mathematically the model we consider, denoted by $\Model$ and depicted in  Figure~\ref{fig:model}.
On a high-level, $\Model$ is a $(L+1)$-layers network, consisting in a single convolution / max pooling layer, followed by a $L$-layers fully-connected neural network with ReLU activations. 
Thus the case $L=1$ corresponds to a single convolutional / max pooling layer followed by a linear transformation.

More precisely, we consider a grayscale image $\Img \in [0,1]^{\Ishp \times \Ishpp}$ as input. For instance, if we consider the MNIST dataset~\cite{lecun_et_al_1998}, $(\Ishp,\Ishpp)=(28,28)$.
We note that our analysis can be easily extended to RBG images. 
The convolutional layer $\convo{\cdot}$ consists of $V$ filters $\Fbf = (\Fbf_1, \ldots, \Fbf_V)$, represented as a collection of $V$ matrices of shape $\Fshp \times \Fshp$.

Formally, the output of the convolution step, $\Abf \defeq \convo{\Img} \in\Reals^{V \times \Ashp \times \Ashpp}$, is given by:
\begin{equation}
\label{eq:def-activation-single-conv}
\forall v \in [V] , \, \forall (i,j)\in [\Ashp]\times [\Ashpp], \qquad 
\Abf_{i,j}^{(v)} = \sum_{p,q=1}^\Fshp \Img_{i+p-1,j+q-1} \Fbf_{p,q}^{(v)}
\, ,
\end{equation}
where $(h,w) = (H - \Fshp +1, W - \Fshp +1)$.

In practice, the filter weights are initialized randomly, typically i.i.d. uniform or Gaussian with proper scaling. 
There are two main trends on how to scale the variance, either Glorot~\citep{glorot_bengio_2010} (also called Xavier), or He~\citep{he_et_al_2015}. 
The later with uniform distribution is default for the CNN layer used in PyTorch. 
However, we assume from now on i.i.d. Gaussian $\gaussian{0}{\tau^2}$ initialization in our analysis for mathematical convenience.

After the convolution step, we apply a ReLU non-linearity, denoted by $\Activ\defeq \max(0,\cdot)$. 
We define the \emph{rectified activation maps} $\Bbf \defeq \activ{\Abf} \in \Reals^{V \times \Ashp \times \Ashpp}$, where $\Activ$ is applied coordinate-wise. 
Next, we consider a down-sampling layer, here a ($k' \times k'$) max pooling $\maxp{\cdot}$. 
One can see that the output of the max pooling, $\Cbf \defeq \maxp{\Bbf} \in\Reals^{V \times \Mshp \times \Mshpp}$, is given by:
\begin{equation}
\label{eq:def-max pool}
\forall v \in [V] , \, \forall (i',j') \in [\Mshp]\times [\Mshpp], \, 
\Cbf_{i',j'}^{(v)} = \max\left(\Bbf_{\Pshp (i'-1)+1:\Pshp i',\Pshp (j'-1)+1:\Pshp j'}^{(v)}\right) 
\, ,
\end{equation}
where $(h',w') = \frac{1}{k'}(h, w)$.
Note that we assume $k'$ to divide $h$ and $w$ for simplicity.

Finally, let us describe recursively the fully-connected part of $\Model$, denoted by $\fnn{\cdot}$: 
\begin{equation}
\begin{aligned}
\Fnn \colon &\Reals^{V \Mshp \Mshpp} \to \Reals \\
            &\Cbf' \xmapsto{} \preactiv{\Cbf'}{L}  
\end{aligned}
\quad \text{with}
\quad
\begin{cases}
\preactiv{\xbf}{0} = \xbf, \\
\preactiv{\xbf}{\ell} = \Wbf^{(\ell)} \activ{\preactiv{\xbf}{\ell-1}} & \text{for } \ell \in [L],
\end{cases}
\end{equation}

where $\Wbf^{(\ell)} \in \Reals^{d_\ell \times d_{\ell-1}}$ is a weight matrix connecting layer $(\ell-1)$ and $\ell$ with $\ell \in [L]$ and $d_\ell$ the size of layer $\ell$. 
Note that we set $d_0 = V \Mshp \Mshpp$, and $d_L = 1$, since we see the output of our model as the un-normalized logit associated to a given class of a prediction problem. 
We also, denote by $\abf^{(\ell)} \defeq \preactiv{\Cbf'}{\ell}$ the non-rectified activation of layer $\ell$ and $\rbf^{(\ell)} \defeq \activ{\abf^{(\ell)}}$ its rectified counterpart. 

\paragraph{Summary.}
The model we consider can be described concisely as $\fnn{\maxp{\activ{\convo{\cdot}}}}$.
As explained in introduction, given the nature of CAM-based explanations, \textbf{it is convenient to split $\Model$ in two functions $f$ and $g$} for the computations of the next Section~\ref{sec:computation}. 
More precisely, we write
\begin{equation}
\begin{aligned}
    \Model \colon \, &[0,1]^{H \times W} \to \Reals \\
    &\Img \xmapsto{} f \circ g(\Img) 
\end{aligned}
\quad \text{with}
\quad
\begin{cases}
    g(\Img) \defeq \activ{\convo{\Img}} \\
    f(\Cbf) \defeq \fnn{\maxp{\Cbf}} \, .
\end{cases}
\end{equation}
Recall that we refer to Figure~\ref{fig:model} for an illustration. 

%%%%%%%%%%%%%%%%%%%%%%%%%%%%%%%%%%%%%%%%%%%%%%%%%%%%%%%%%%%%%%%%

\subsection{Closed-form expression}
\label{sec:computation}

The original idea of CAM \cite{zhou_et_al_2016} was limited to computing the saliency map as a linear combination of the feature maps in the last convolutional layer when $L=1$. 
Later, GradCAM~\citep{selvaraju_et_al_2017} removed the architecture constraints by computing the average gradient of each feature map with respect to $\Bbf$. 
In our notation, we have:

\begin{definition}[GradCAM]
\label{def:gradcam}
For an input $\Img$ and model $\Model$, the GradCAM feature scores are given by
\begin{align*}
\abrev{GC} \defeq \activ{\sum_{v=1}^{V} \alpha_v \Bbf^{(v)}} \in \Reals_+^{h \times w}
\, ,
\end{align*}
where each $\alpha_v \defeq \gap{\nabla_{\Bbf^{(v)}} f(\Bbf)}\in\Reals$.
Here, $\text{GAP}$ denotes the \emph{global average pooling}, that is, the average of all values, and $\Activ$ the ReLU as before. 
\end{definition}

Definition~\ref{def:gradcam} is of course to be taken coordinate-wise. 
We note that, in practice, $\abrev{GC}$ is up-sampled and normalized to produce a saliency map with the \emph{same shape} as the input image. 
To be more precise, what we define as $\abrev{GC}$ is the middle panel of Figure~\ref{fig:gradcam-fail-panel}, whereas the final user will nearly always visualize the right panel.
The most important thing to notice in Definition~\ref{def:gradcam} is that $\alpha$ is a \textbf{global} coefficient. 

We now show why this can be an issue. 
Looking at Definition~\ref{def:gradcam}, whenever the underlying model is not too complicated, one can actually hope to derive a closed-form expression for the feature importance scores of $\abrev{GC}$ as a function of the model's parameters. 
This is achieved by: 

\begin{proposition}[$\alpha$ coefficients for GradCAM, $V=1$]
\label{prop:l-hidden-layer}
Recall that the $\abf$ vectors denote the non-rectified activation and $\Wbf$ the weights of the linear part of $\Model$. 
Then, for input~$\Img$, the $\abrev{GC}$ coefficient $\alpha$ is given by
\begin{align*} 
\alpha = \frac{1}{\Ashp  \Ashpp} \sum_{\substack{i,j=1}}^{h',w'} \sum_{i_1, \ldots, i_{L-1} = 1}^{d_1, \ldots, d_{L-1}} \indic{\abf_{i_1}^{(1)}, \ldots, \abf_{i_{L-1}}^{(L-1)} >0} \prod_{p=1}^{L} (\Wbf_{i_p, i_{p-1}}^{(p)})^\top \, ,
\end{align*}
where we set $i_0 \defeq (i,j)$ and $i_L = 1$. 
\end{proposition}

From Proposition~\ref{prop:l-hidden-layer}, we immediately deduce a closed-form expression for GradCAM explanations. 
We note that Proposition~\ref{prop:l-hidden-layer} can be readily extended to an arbitrary number of filters $V>1$, in which case the $\abf$ and $\Wbf$ should be interpreted as corresponding to the relevant $v\in [V]$. 

The proof of Proposition~\ref{prop:l-hidden-layer} can be found in Appendix~\ref{sec:proof-coefs}. 
In Appendix~\ref{app:defs}, we describe mathematically several other CAM-based methods in the setting of $\Model$: XGradCAM~\citep{fuXcam}, GradCAM++~\citep{ChattopadhayCam}, HiResCAM~\citep{draelos_carin_2021}, ScoreCAM~\citep{wangScorecam} and AblationCAM~\citep{desaiAblationcam}. 
A close inspection of these definitions reveals that they also use global weighting coefficients  applied to the corresponding activation maps, with the notable exception of HiResCAM.

%%%%%%%%%%%%%%%%%%%%%%%%%%%%%%%%%%%%%%%%%%%%%%%%%%%%%%%%%%%%%%%%%%%%%%%%%%%%%%%%%%%%%%

\begin{figure}[t]
    \centering

    \begin{subfigure}[t]{0.24\textwidth} 
        \includegraphics[width=\textwidth]{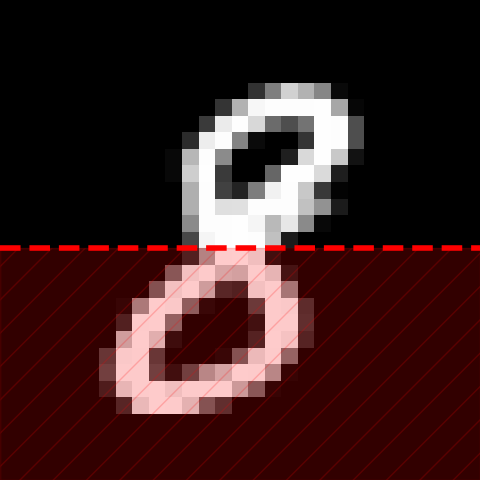} 
        \label{fig:mnist}
    \end{subfigure}
    \hspace{10pt}
    \begin{subfigure}[t]{0.24\textwidth} 
        \includegraphics[width=\textwidth]{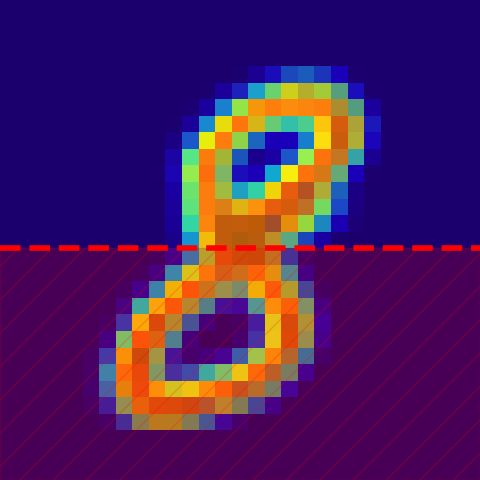}    
        \label{fig:cam_mnist}
    \end{subfigure}
\caption{\label{fig:experiments-theory}
Illustration of Theorem~\ref{th:expected-gradcam} on an MNIST~\cite{lecun_et_al_1998} digit (\emph{left panel}). 
We set to zero the lower part of $\Wbf$ for $\abrev{CNN}$, initialize the filter values and remaining weights to i.i.d. $\gaussian{0}{1}$, and run GradCAM to get a saliency map (\emph{right panel}).  
Even though our network does not see the red part of the image, \textbf{GradCAM does highlight some pixels in this area}, as predicted by Theorem~\ref{th:expected-gradcam}.}
\end{figure}

%%%%%%%%%%%%%%%%%%%%%%%%%%%%%%%%%%%%%%%%%%%%%%%%%%%%%%%%%%%%%%%%%%%%%%%%%%%%%%%

\subsection{Theoretical analysis}
\label{sec:theory}

Leveraging the results of Section~\ref{sec:computation}, we are able to describe precisely the behavior of GradCAM at initialization for $\Model$, specifically when the classifier part of our model comprises a single layer ($L=1$).
This analysis is justified by existing works~\citep{lee_et_al_2019,du_et_al_2019,allen_li_song_2019_b,allen_li_song_2019_a,zou_et_al_2020} showing that, in certain regimes, neural networks stay ``near initialization'' during training. 
As announced, we conduct this analysis when the network does not have access to the lower part of the image. 
Our main result is:

\begin{theorem}[Expected GradCAM scores, $L=1$, masked $\abrev{CNN}$] 
\label{th:expected-gradcam}
Let $\Img \in [0,1]^{H \times W}$ be an input image.  
Let $\mbf \defeq \Img_{i:i+k-1, j:j+k-1}$ be the patch of $\Img$ corresponding to index $(i,j)\in [h] \times [w]$. 
Assume that $h'$ is even, and $\Wbf_{:,-\frac{h'}{2}:,:} = 0$. 
Assume that the filter values and the non-zero weights are initialized i.i.d. $\gaussian{0}{\tau^2}$. 
Then, if the number of filters $V$ is greater than $20$, we have
the following expected lower bound on the GradCAM explanation for pixel $(i,j)$:
\begin{equation}
\label{eq:gradcam-lower-bound}
\expec{\abrev{GC}_{i,j}} = 
\expec{\activ{\sum_{v=1}^{V} \alpha_v \Bbf_{i,j}^{(v)}}} \geq \frac{V-20}{\sqrt{V}} \sqrt{\frac{h' w'}{16 \pi}} \frac{\tau^2}{h w} \norm{\mbf}_2
\, ,
\end{equation}
where the expectation in the previous inequality is taken with respect to initialization of the filters and the remaining weights of the linear layer. 
\end{theorem}

Setting $\Wbf_{:,-\frac{h'}{2}:,:}$ to $0$ disables the weights within $\Wbf$ that are connected to the lower half part of the activation map $\Cbf_{:,-\frac{h'}{2}:,:}$, effectively preventing $\Model$ from accessing the lower half of $\Cbf$. 
In turn, $\Model$ does not see the lower half of $\xi$, up to side effects. 
The main consequence of Theorem~\ref{th:expected-gradcam} is that, when the number of filters associated to the class to explain is large enough, $\abrev{GC}_{i,j}$ is positive in expectation if some pixels are activated in the receptive field associated to $(i,j)$. 
Thus \textbf{GradCAM highlights all parts of the image where there is some ``activity,'' even though this information is not used by the network in the end.} 
We illustrate Theorem~\ref{th:expected-gradcam} in Figure~\ref{fig:experiments-theory}.
The main limitation of this analysis is its focus on the behavior at initialization: we investigate in the following whether this behavior also happens after training.  
Another limitation is the restriction to a single linear layer, but we note that taking $L=1$ in the fully connected part of $\Model$ is a dominant approach since ResNet~\citep{he_et_al_2016}.

The proof of Theorem~\ref{th:expected-gradcam} can be found in Appendix~\ref{sec:proof-main}. 
The key ingredient of the proof is obtaining a probabilistic control of $\sum_q\alpha_q\Bbf^{(q)}$. 
We note that a similar analysis is possible for other expressions of $\alpha$, thus other CAM-based methods. 

%%%%%%%%%%%%%%%%%%%%%%%%%%%%%%%%%%%%%%%%%%%%%%%%%%%%%%%%%%%%%%%%%%%%%%%%%%%%%%%

\section{Experiments}
\label{sec:experiments}

We know ask the following question: are the consequences of Theorem~\ref{th:expected-gradcam} true after training, and for a more realistic model? 
To this extent, we train a CNN-based model which by construction cannot access some specified part of the input which we call the \emph{dead zone} (see Figure~\ref{fig:vgg}, details in Section~\ref{sec:model}).
Clearly, since the dead zone does not influence the output, it should not contain positive model explanations. 
To test whether this is true, we create two datasets (Section~\ref{sec:datasets}). 
Each item of the first one is composed of two images from ImageNet with the same label in both the seen and the unseen part of the image. 
The second dataset is built using generative models on the same categories with two objects in each image located in the seen and unseen part as well. 
We then check whether CAM-based methods wrongly highlight areas in the dead zone in Section~\ref{sec:results}.

%%%%%%%%%%%%%%%%%%%%%%%%%%%%%%%%%%%%%%%%%%%%%%%%%%%%%%%%%%%%%%%%%%

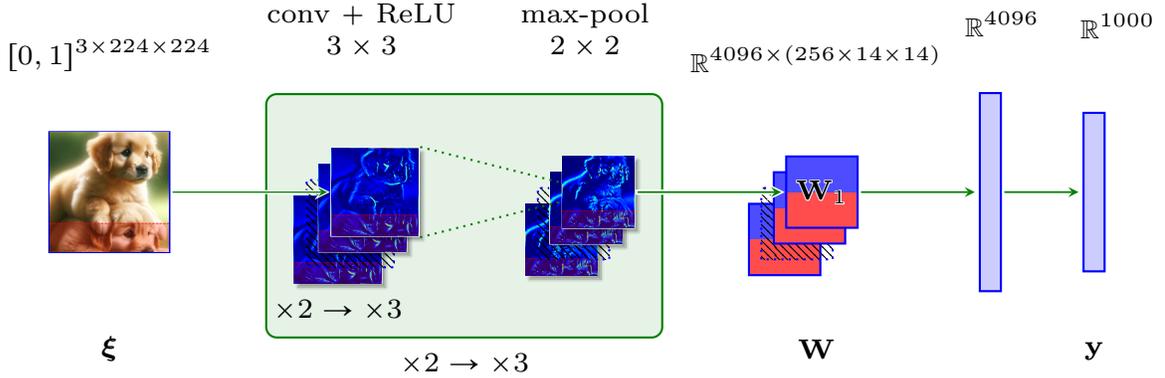
\begin{figure}[t!]
\centering
\hspace*{-\dimexpr\oddsidemargin+1in\relax}\makebox[\paperwidth]{%
\scalefont{0.6}\input{figures/tikz/vgg}} \hspace*{-\paperwidth}    
\caption{\label{fig:vgg}Our masked VGG16-based model trained on ImageNet with $87.0 \%$ top-5 accuracy. The \textcolor{red}{down weights $\Wbf_{:,\, :,\, -9:, \, :}$} are set to $0$ and not updated during training. Only the \textcolor{blue}{up weights $\Wbf_{:,\, :,\, :5, \, :}$} and the other parameters undergo training. This setting implies that every \textcolor{red}{red part} in the channels does not impact the prediction scores, meaning that they are not used. Symbol $\times 2 \rightarrow \times 3$ means the model first uses the green block twice, with each time having $2$ consecutive convolutions. Then, it uses the green block three times, with each time having~$3$ consecutive convolutions. There is no max pooling after the last convolution.}
\end{figure}

%%%%%%%%%%%%%%%%%%%%%%%%%%%%%%%%%%%%%%%%%%%%%%%%%%%%%%%%%%%%%%%%%%%%%%%%%%%%%%%%%%%%%%%

\subsection{Model}
\label{sec:model}

\paragraph{Model definition. }
The CNN used in our experiments is a modification of a classical VGG16 architecture~\citep{SimonyanZisserman2015} which we call $\abrev{VGG}$. 
Whereas the original VGG16 model is composed of $5$ convolutional blocks including either $2$ or $3$ convolutional layers with ReLU and max pooling, followed by $3$ dense layers.
In $\abrev{VGG}$ we remove the last max pooling (in the fifth convolutional block) and we further apply a mask on selected neurons of the first dense layer so the layer can not see the lower part of the activation maps, see Figure~\ref{fig:vgg} for more details.

\paragraph{Masking.}
We forbid the network from seeing the dead zone in a very simple way: in the first dense layer $\Wbf$, which has size $4096 \times (256 \times 14 \times 14)$, we permanently set to $0$ a band of height $9$ corresponding to the lower weights.
Formally, this means setting $\Wbf_{:,\, :,\, -9:, \, :} = 0$, which is denoted in red above $\Wbf$ in Figure~\ref{fig:vgg}. 
Effectively, we are building a wall that stops all information flowing from the last convolutional layer to the remainder of the network. 
Since the weights $\Wbf$ are directly connected to the final activation map $\Bbf \in \Reals_+^{256 \times 14 \times 14}$, this masking effectively zeroes out the lower sections in each channel denoted by $\Bbf_{:,\, -9:, \, :}$. 
We can trace back the zeroed activations in $\Bbf$ to the preceding activation map $\Cbf$, pinpointing the exact patches in $\Cbf$ that correspond (after convolution) to the features observed in the zeroed activation of $\Bbf$. 
Because of the side effects in the computation of convolutions, this area of $\Cbf$ is slightly smaller: some pixel activation will still play a role in the model's prediction. 
Repeating this process until we reach the original image yields a dead zone of height $54$ pixels, highlighted in red above $\Img$ in Figure~\ref{fig:vgg}, which covers $24 \%$ of the image area.
As we mentioned earlier, the other main difference with VGG16 is the removing of the final max pooling layer. 
This leads to a larger activation layer, allowing us to set weights to zero without hindering too much the network's ability, see Figure~\ref{fig:plot}. 
We note that $\abrev{VGG}$ bears a strong resemblance to $\Model$. 
The main difference is that the convolutional layer of $\Model$ is replaced by several convolutional blocks in $\abrev{VGG}$, see Figure~\ref{fig:vgg}.

\paragraph{Training. }
We train $\abrev{VGG}$ on Imagenet-1k~\citep{dengimagenet2009} using classical data augmentation recipe, \emph{i.e.}, random flip and random crop. 
As optimization algorithm, we use stochastic gradient descent with momentum, weight decay, and a learning rate scheduler.
To observe the slight accuracy drop induced by masking a significant part of images, we train a baseline model without masking.
We report the train loss and the validation accuracy across training in Figure~\ref{fig:plot}.

%%%%%%%%%%%%%%%%%%%%%%%%%%%%%%%%%%%%%%%%%%%%%%%%%%%%%%%%%%%%%

\begin{figure}[t]
    \begin{center}
        \includegraphics[scale = 0.9]{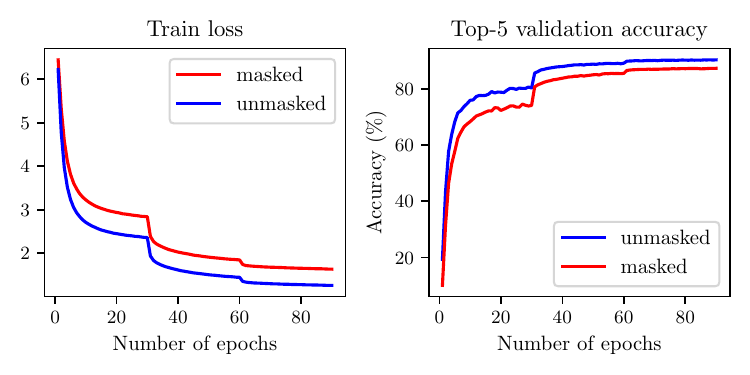}
    \end{center}
    \caption{\label{fig:plot}Plots of the training loss and the top-5 validation accuracy of our unmasked and masked VGG16-like models on Imagenet-1k (2012)~\citep{dengimagenet2009}. The unmasked $\abrev{VGG}$ (baseline) and our masked $\abrev{VGG}$ yield $87.0 \%$, resp. $90.4 \%$ top-5 accuracy and $66.5 \%$ resp. $71.5 \%$ top-1 accuracy on the validation set.}
\end{figure}

%%%%%%%%%%%%%%%%%%%%%%%%%%%%%%%%%%%%%%%%%%%%%%%%%%%%%%%%%%%%%%%

\paragraph{Comparison to SOTA. }
We also compare the validation top-1 and top-5 accuracy of the VGG16 model found in the PyTorch repository. 
Our $\abrev{VGG}$ without max pooling and no masking offers the same performance: $71.5 \%$ top-1 and $90.4 \%$ top-5 accuracy on the validation set.
As we mention in Figure~\ref{fig:plot}, our model $\abrev{VGG}$ with masking has lower performance, which is expected as a fourth of the input image, $\Img_{:, \, 171:224, \, :}$, is unseen by the model.
We obtain $66.5 \%$, resp. $71.5 \%$, top-1 and  $87.0 \%$, resp. $90.4 \%$, top-5 accuracy on the validation set for our masked $\abrev{VGG}$, resp. unmasked $\abrev{VGG}$.
Nevertheless, we see that $\abrev{VGG}$ is a \textbf{realistic network able to predict ImageNet classes with reasonable accuracy}. 
We believe that the drop in accuracy is only minor because ImageNet images are centered, and there is enough information in the upper part of the image to achieve near-perfect prediction.

%%%%%%%%%%%%%%%%%%%%%%%%%%%%%%%%%%%%%%%%%%%%%%%%%%%%%%%%%%%%%%%%%%%%%%%%%%%%%%%%%%%%%%%%%%

\subsection{Proposed datasets}
\label{sec:datasets}

\paragraph{Objective. }
To assess how much CAM-based saliency maps emphasize irrelevant areas of an image, we introduce two new datasets in which we control the positions of the image elements using two techniques: cutmix~\citep{yuncutmix2019} and generative model. 
More precisely, we produce two datasets, called \textsc{STACK-MIX} and \textsc{STACK-GEN}. 
Where each image contains two objects, one in the bottom part of the image which is the dead zone for $\abrev{VGG}$, and the second subject at the top of the image. Therefore, the subject at the center of the image will be mainly responsible for the top-1 predicted score by our masked $\abrev{VGG}$.

%%%%%%%%%%%%%%%%%%%%%%%%%%%%%%%%%%%%%%%%%%%%%%%%%%%%%%%%%%%%%%%%%%%%%%%%%%%%%%%%%%%%%%%%%%%%%%%%

\begin{figure}[t!]
    \centering
    \begin{tabular}{l@{\hskip 0.5cm}llll}
\rotatebox[origin=c]{90}{\textsc{STACK-MIX}} & \includegraphics[width=.2\linewidth,valign=m]{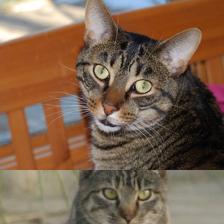} & \includegraphics[width=.2\linewidth,valign=m]{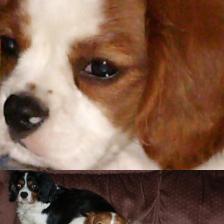} & \includegraphics[width=.2\linewidth,valign=m]{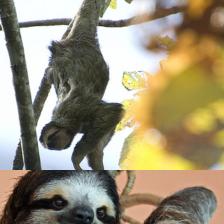}& \includegraphics[width=.2\linewidth,valign=m]{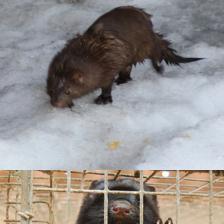}\\[1cm]
\rotatebox[origin=c]{90}{\textsc{STACK-GEN}} & \includegraphics[width=.2\linewidth,valign=m]{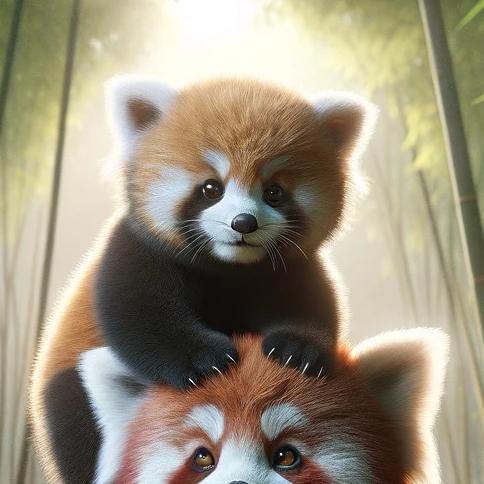} & \includegraphics[width=.2\linewidth,valign=m]{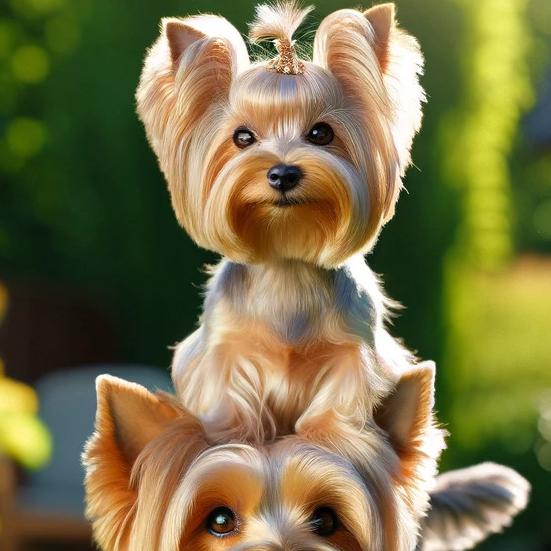} & \includegraphics[width=.2\linewidth,valign=m]{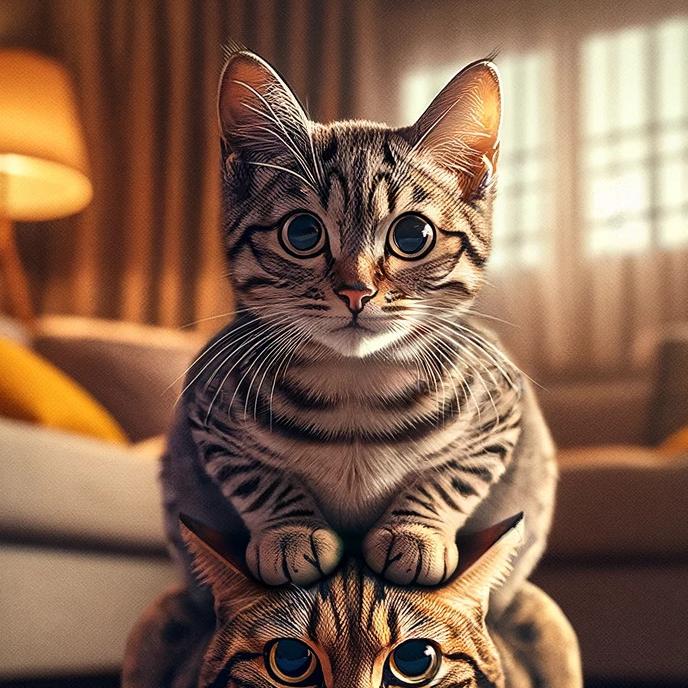}& \includegraphics[width=.2\linewidth,valign=m]{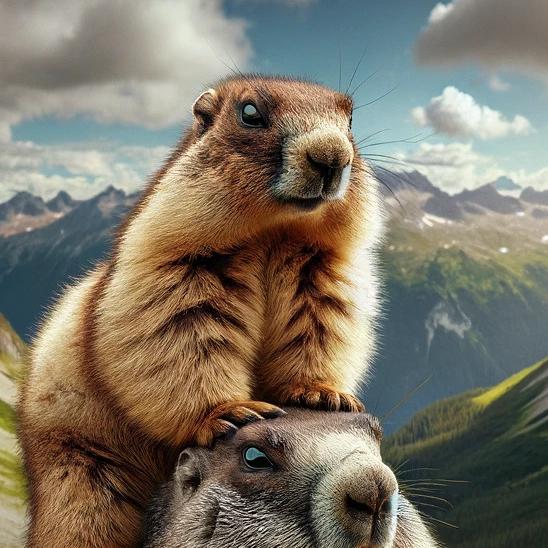}\\[1cm]
    \end{tabular}

\caption{\label{fig:sampled_imgs}Sampled images from both of our datasets, \emph{i.e.}, \textsc{STACK-MIX} and \textsc{STACK-GEN}.}
\end{figure}

%%%%%%%%%%%%%%%%%%%%%%%%%%%%%%%%%%%%%%%%%%%%%%%%%%%%%%%%%%%%%%%%%%%%%%%%%%%%%%%%%%%%%%%%%%%%%%%%

\paragraph{STACK-MIX. }
We first generate labels for our datasets by ramdomly sampling $100$ classes from the $398$ first labels of Imagenet, which corresponds to animals. 
The first dataset, called \textsc{STACK-MIX}, consists of $100$ images featuring one image from each of the $100$ classes. 
Each example $\Img$ is created by mixing, in a  cutmix~\citep{yuncutmix2019} fashion, two images $(\Img_1, \Img_2)$ with same label and sampled randomly in the \textbf{validation set} of Imagenet as follows:
\begin{equation}
    \Img \defeq 
\begin{pmatrix}
(\Img_1)_{:, \, :170, \, :} \\
(\Img_2)_{:, \, 171:224, \, :}
\end{pmatrix} 
\in [0,1]^{3 \times 224 \times 224}
\, ,
\end{equation}
meaning that we create a composite image $\Img$ by superposing an upper vertical slice, taken from the top region of $\Img_1$ with size $3 \times 170 \times 224$, with a lower vertical slice, taken from the bottom region of $\Img_2$ with size $3 \times 54 \times 224$.
Finally, the quality of the generated images is verified through manual inspection.
This dataset lacks realism due to the distinct separation between the two subjects. We address this issue with the help of generative models.

\paragraph{STACK-GEN. }
The second dataset, called \textsc{STACK-GEN}, consists of $100$ images featuring one image from each of the same $100$ classes. 
It was generated using ChatGPT + DALL·E 3~\cite{brown2020language, ramesh2021zero} by sampling prompts of the following form: ``A photo of \{animal name\} stacked on top of \{same animal name\}''. 
The word ``stacked'' determines the positions of the subjects in the generated image, which proceeds as follows: first, ChatGPT refines the original prompt to enhance its suitability for DALL·E 3, then the image is generated. 
We then preprocess the generated images by selectively editing them to minimize the background and centering the focus on the two animals. This editing involves cropping the images to a 1:1 ratio, ensuring one animal is predominantly within the dead zone as defined by our $\abrev{VGG}$, while the other is positioned in the upper part of the new image. 
Figure~\ref{fig:sampled_imgs} shows examples of the created images. 
Note that both datasets are provided in the supplementary material and online.\footnote{\url{https://github.com/MagamedT/cam-can-see-through-walls}}

%%%%%%%%%%%%%%%%%%%%%%%%%%%%%%%%%%%%%%%%%%%%%%%%%%%%%%%%%%%%%%%%%%%%%%%%%%%%%%%%

\subsection{Results}
\label{sec:results}

\begin{figure}[h]
    \centering
    \begin{tabular}{l@{\hskip 0.5cm}llll}
\rotatebox[origin=c]{90}{Input image} & \includegraphics[width=.2\linewidth,valign=m]{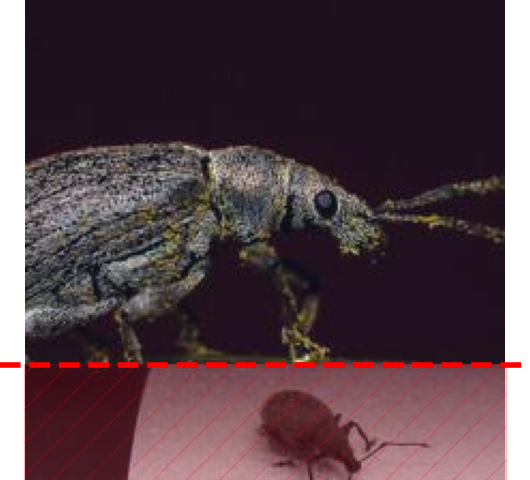}& \includegraphics[width=.2\linewidth,valign=m]{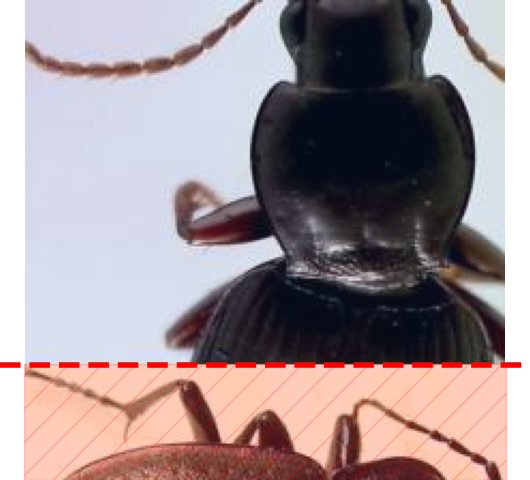} & \includegraphics[width=.2\linewidth,valign=m]{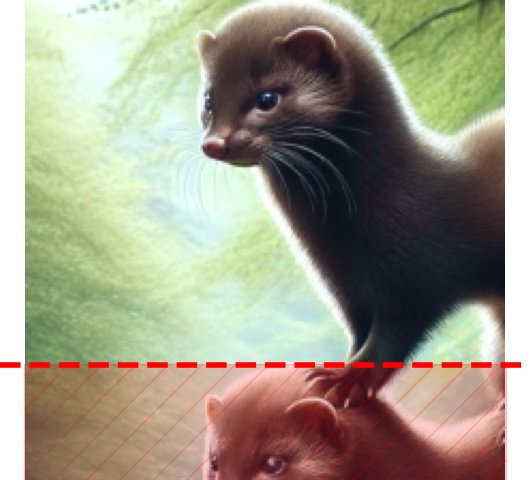} & \includegraphics[width=.2\linewidth,valign=m]{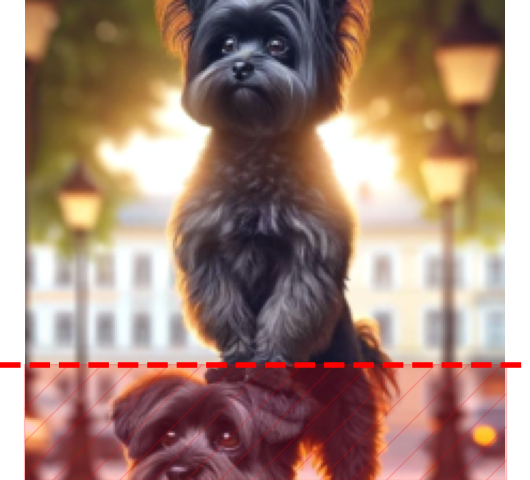} \\[1cm]
\rotatebox[origin=c]{90}{GradCAM} & \includegraphics[width=.2\linewidth,valign=m]{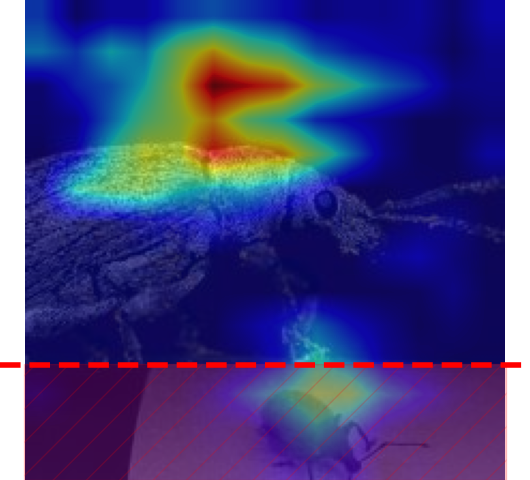}& \includegraphics[width=.2\linewidth,valign=m]{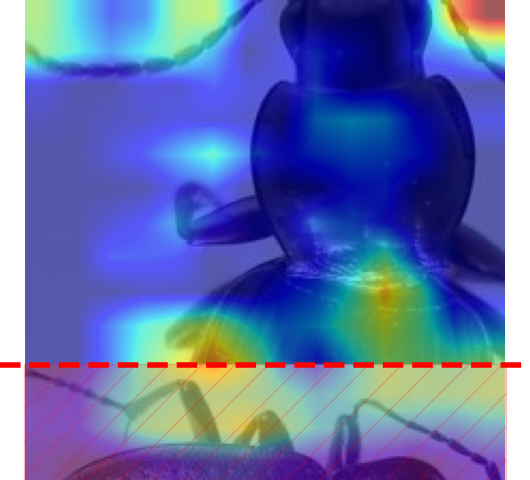}  & \includegraphics[width=.2\linewidth,valign=m]{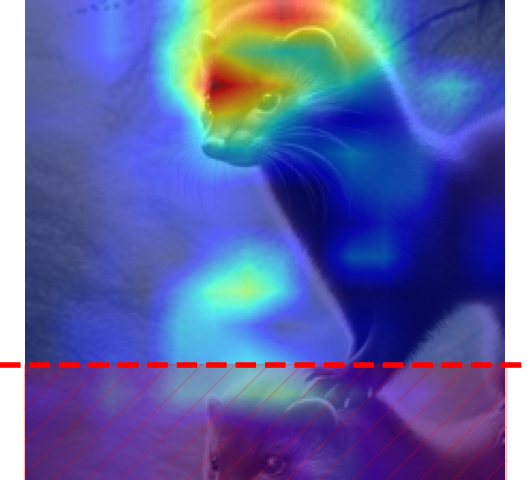} & \includegraphics[width=.2\linewidth,valign=m]{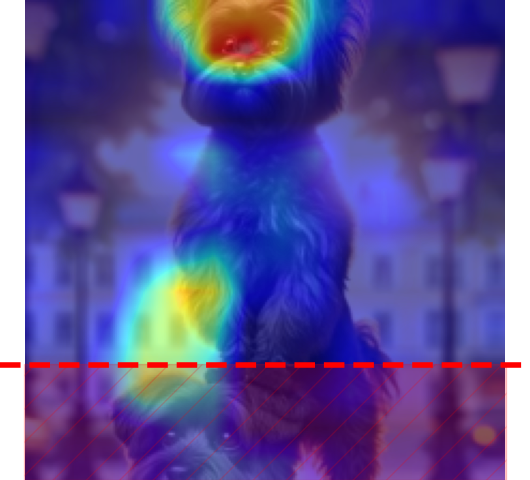} \\[1cm]
\rotatebox[origin=c]{90}{GradCAM++} & \includegraphics[width=.2\linewidth,valign=m]{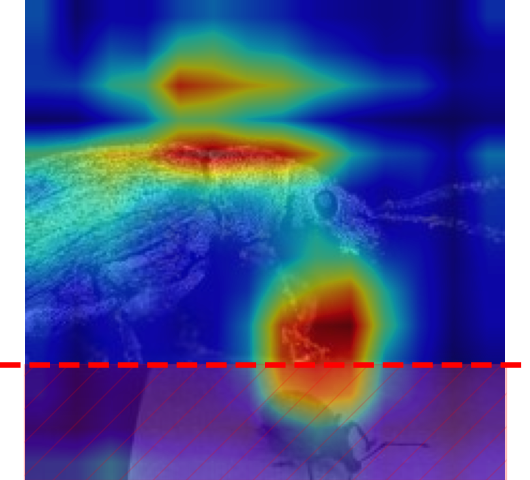}& \includegraphics[width=.2\linewidth,valign=m]{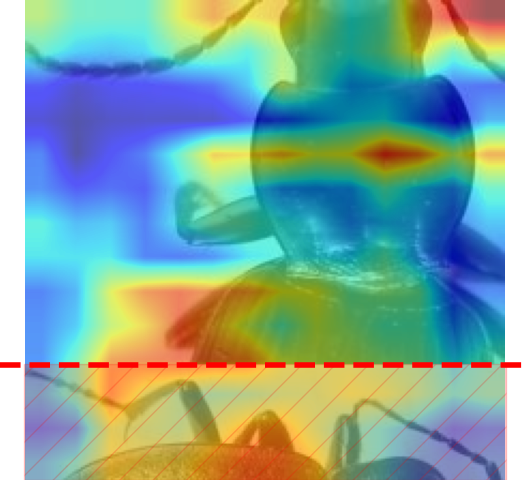} & \includegraphics[width=.2\linewidth,valign=m]{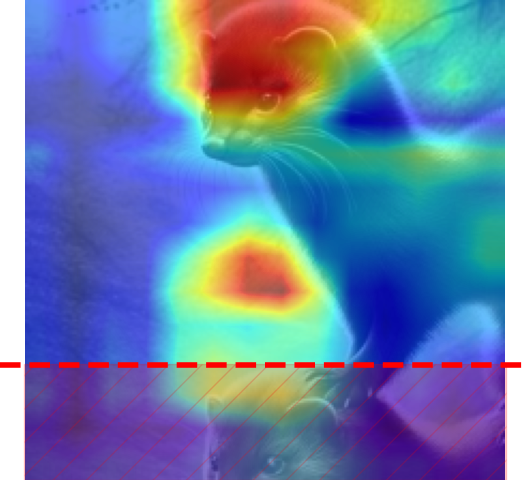} & \includegraphics[width=.2\linewidth,valign=m]{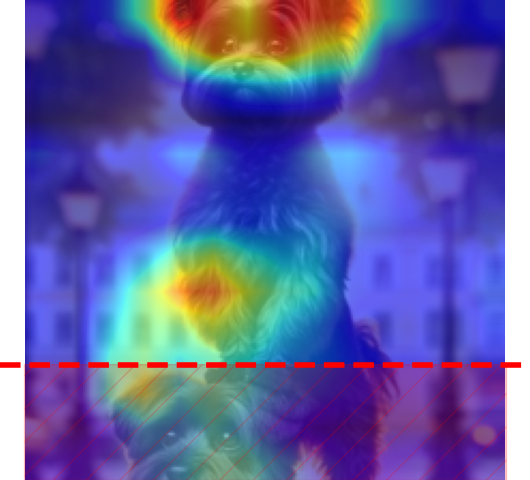} \\[1cm]
\rotatebox[origin=c]{90}{XGradCAM} & \includegraphics[width=.2\linewidth,valign=m]{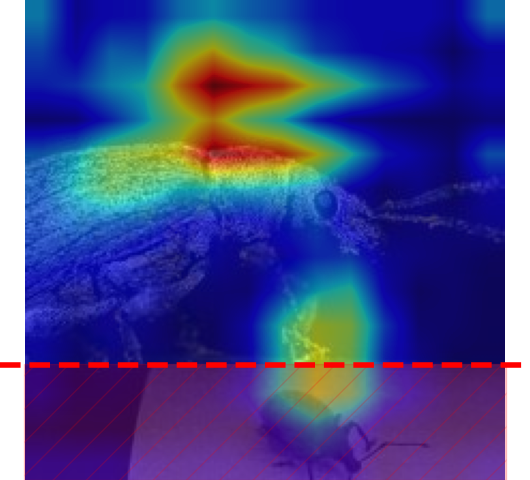}& \includegraphics[width=.2\linewidth,valign=m]{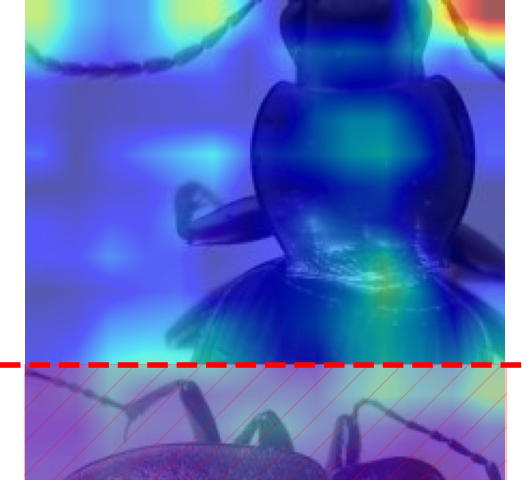} & \includegraphics[width=.2\linewidth,valign=m]{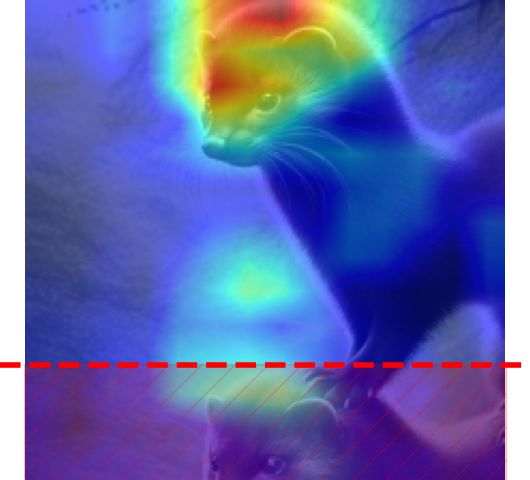} & \includegraphics[width=.2\linewidth,valign=m]{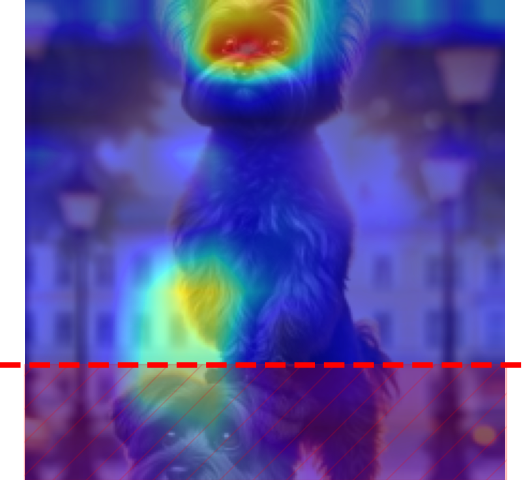} \\[1cm]
\rotatebox[origin=c]{90}{ScoreCAM}  & \includegraphics[width=.2\linewidth,valign=m]{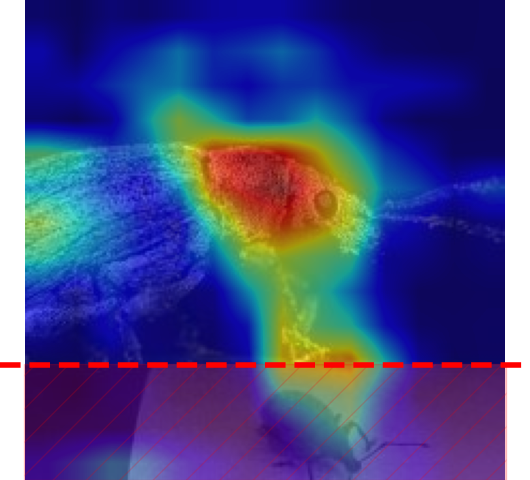}& \includegraphics[width=.2\linewidth,valign=m]{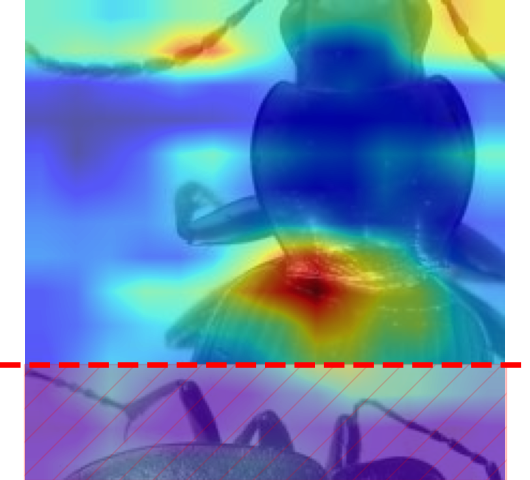} & \includegraphics[width=.2\linewidth,valign=m]{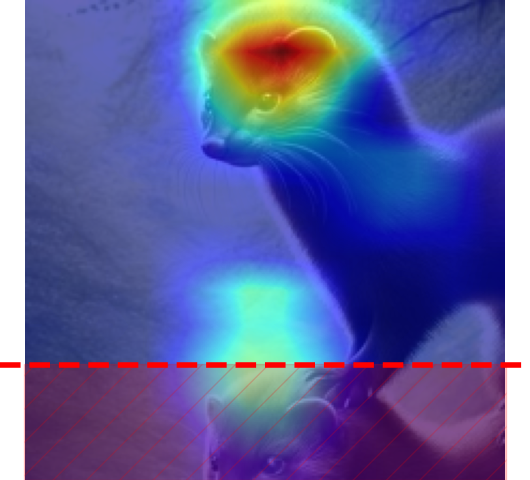} & \includegraphics[width=.2\linewidth,valign=m]{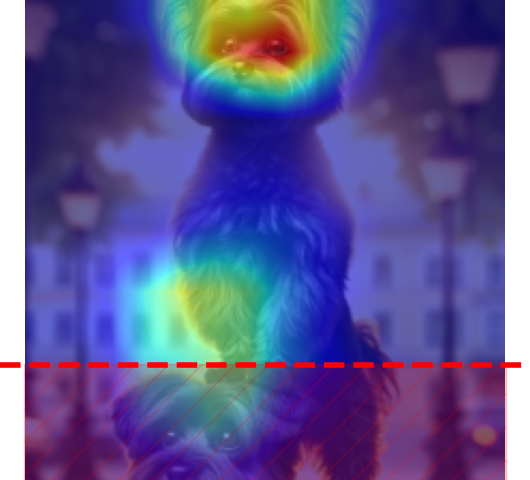}\\[1cm]
\rotatebox[origin=c]{90}{Opti-CAM}  & \includegraphics[width=.2\linewidth,valign=m]{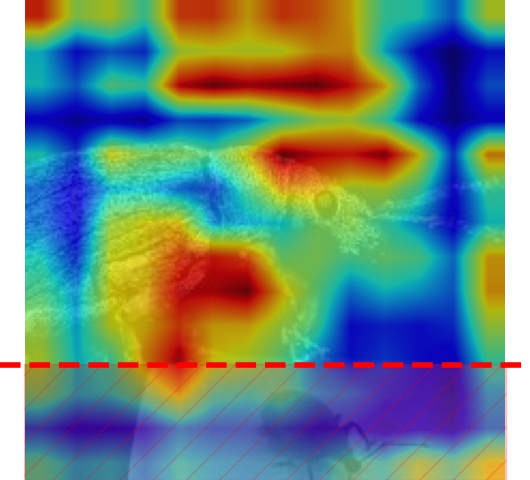}& \includegraphics[width=.2\linewidth,valign=m]{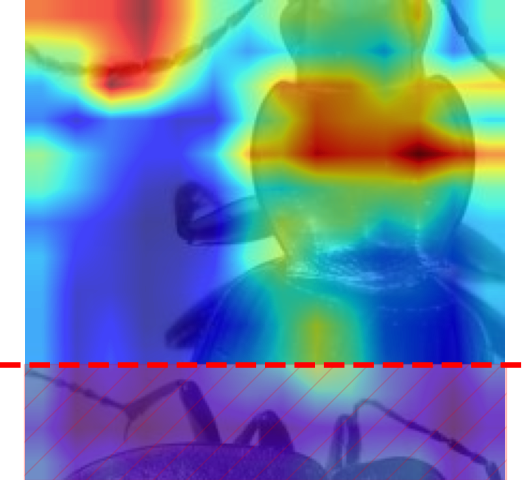} & \includegraphics[width=.2\linewidth,valign=m]{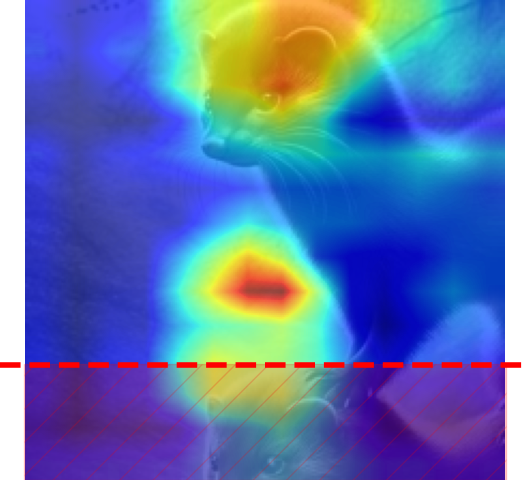} & \includegraphics[width=.2\linewidth,valign=m]{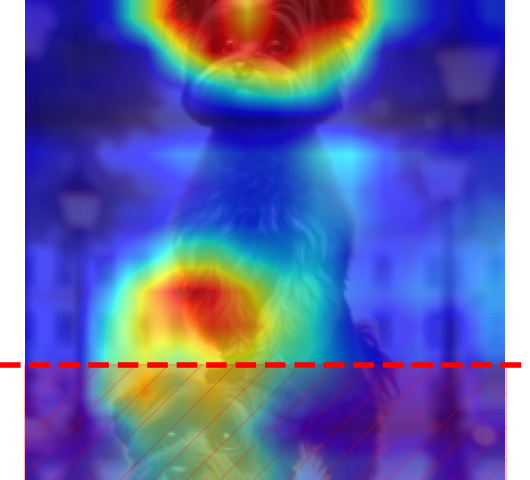}\\[1cm]
\rotatebox[origin=c]{90}{AblationCAM} & \includegraphics[width=.2\linewidth,valign=m]{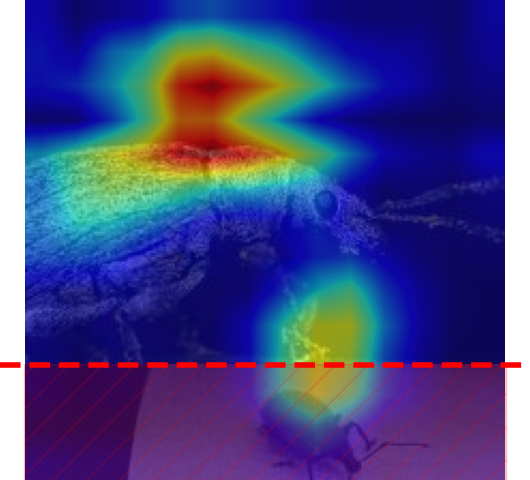}& \includegraphics[width=.2\linewidth,valign=m]{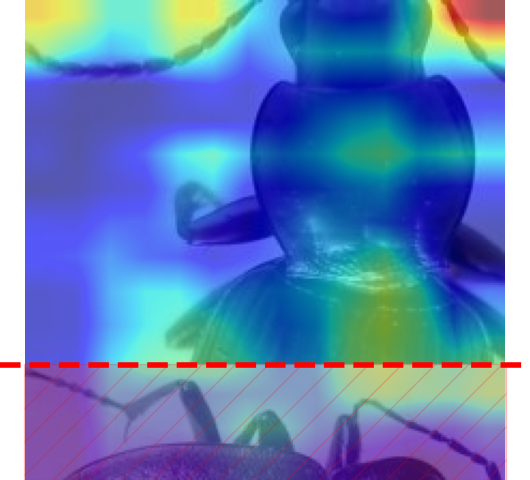} & \includegraphics[width=.2\linewidth,valign=m]{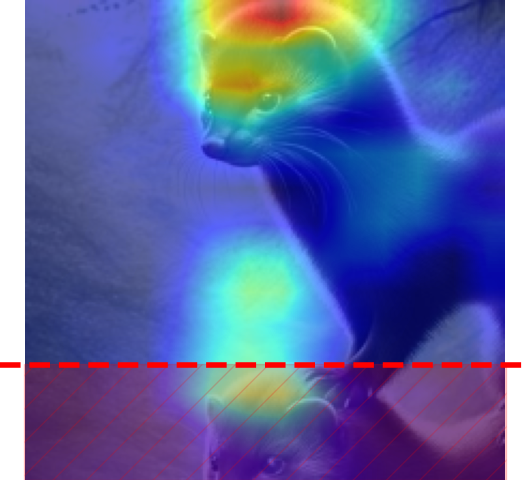} & \includegraphics[width=.2\linewidth,valign=m]{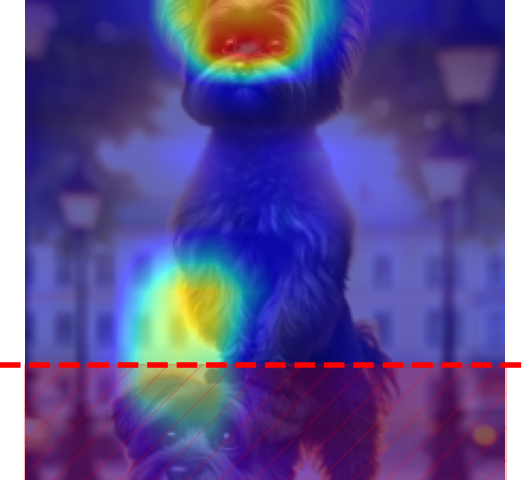} \\[1cm]
\rotatebox[origin=c]{90}{EigenCAM} & \includegraphics[width=.2\linewidth,valign=m]{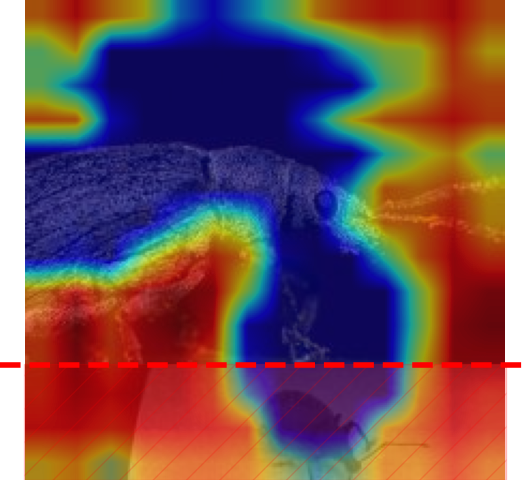}& \includegraphics[width=.2\linewidth,valign=m]{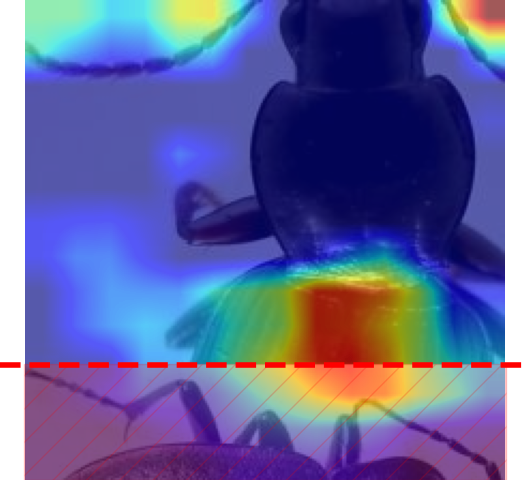} & \includegraphics[width=.2\linewidth,valign=m]{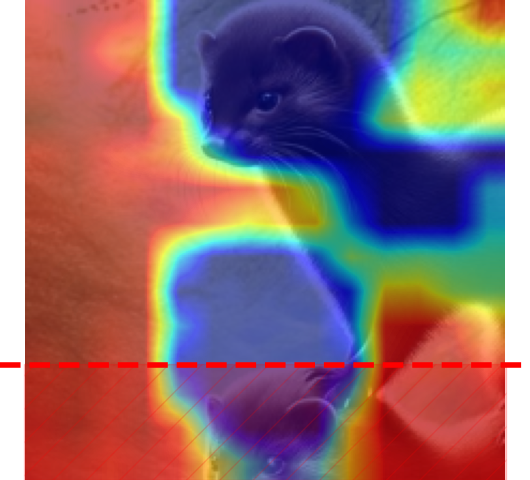} & \includegraphics[width=.2\linewidth,valign=m]{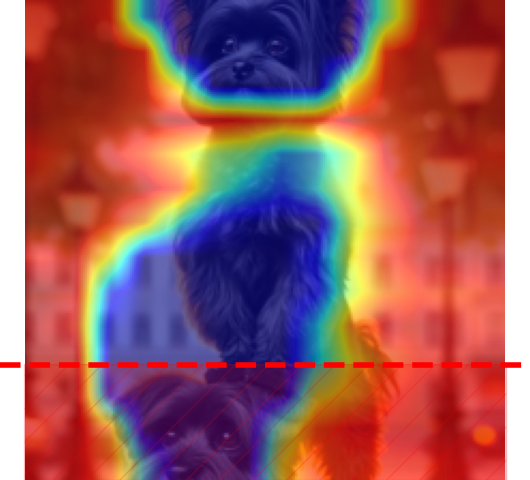} \\[1cm]
\rotatebox[origin=c]{90}{HiResCAM} & \includegraphics[width=.2\linewidth,valign=m]{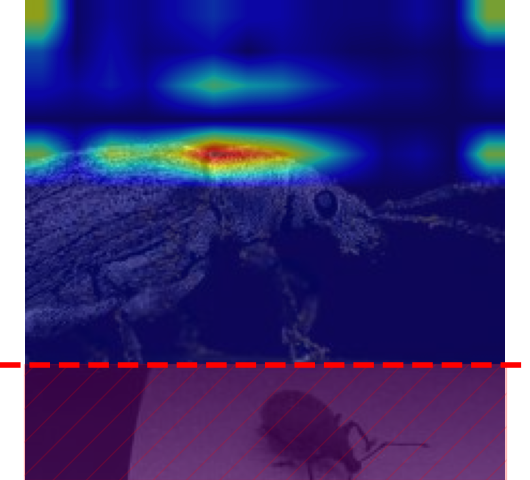}& \includegraphics[width=.2\linewidth,valign=m]{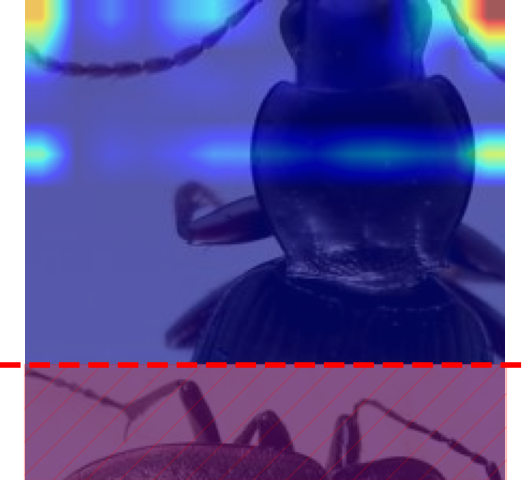}  & \includegraphics[width=.2\linewidth,valign=m]{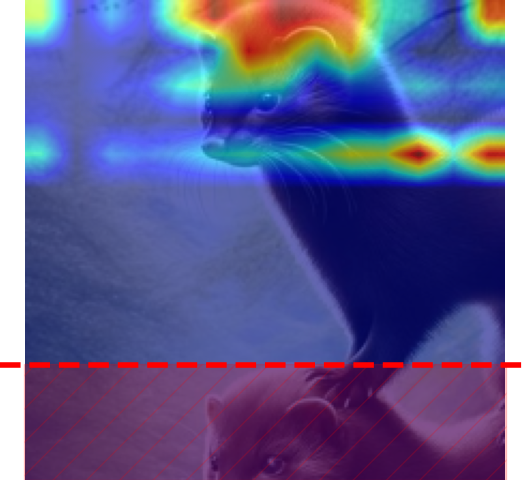} & \includegraphics[width=.2\linewidth,valign=m]{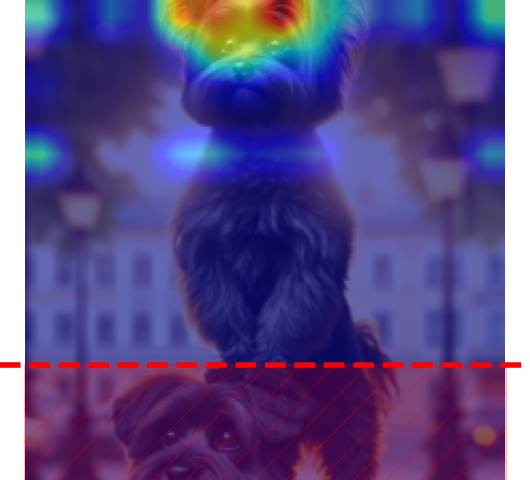} 
\end{tabular}
\caption{Saliency maps given by the considered CAM-based methods for $\abrev{VGG}$. With the notable exception of HiResCAM, all method highlight parts of images from $\textsc{STACK-GEN}$ and $\textsc{STACK-MIX}$ which are unseen by the network. The lower part in red is unseen by the model.}
\label{fig:qualitative-results}
\end{figure}

\begin{figure}[h]
    \centering
    \begin{tabular}{l@{\hskip 0.5cm}llll}
\rotatebox[origin=c]{90}{Input image} & \includegraphics[width=.2\linewidth,valign=m]{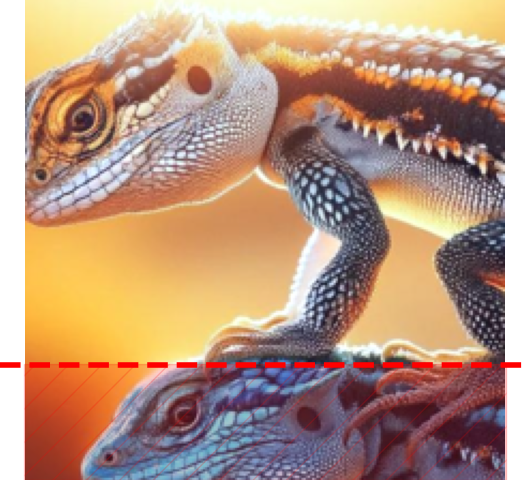} & \includegraphics[width=.2\linewidth,valign=m]{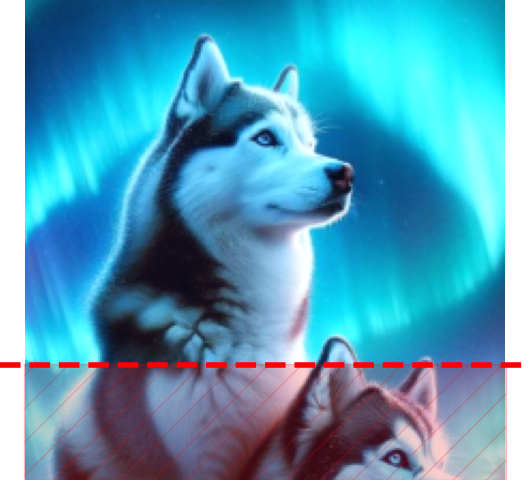} & \includegraphics[width=.2\linewidth,valign=m]{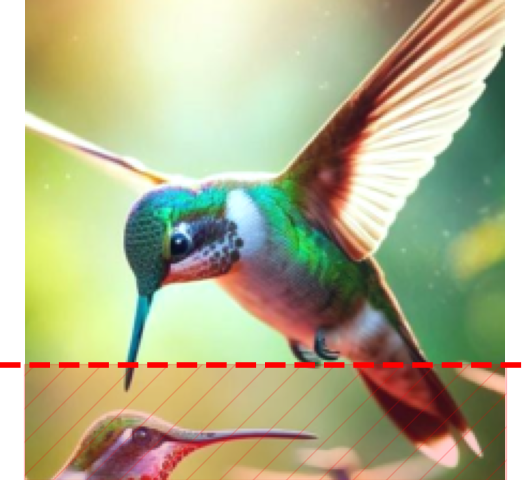}& \includegraphics[width=.2\linewidth,valign=m]{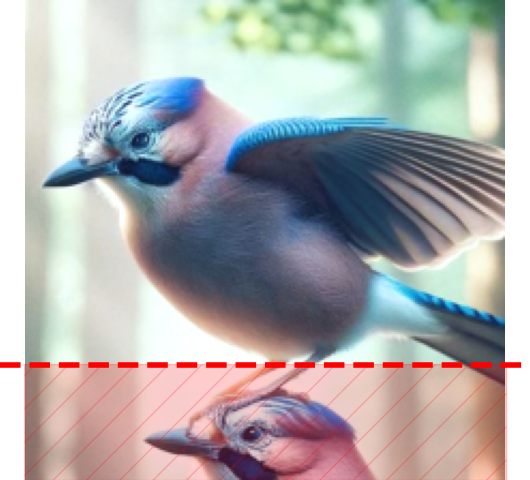}\\[1cm]
\rotatebox[origin=c]{90}{GradCAM} & \includegraphics[width=.2\linewidth,valign=m]{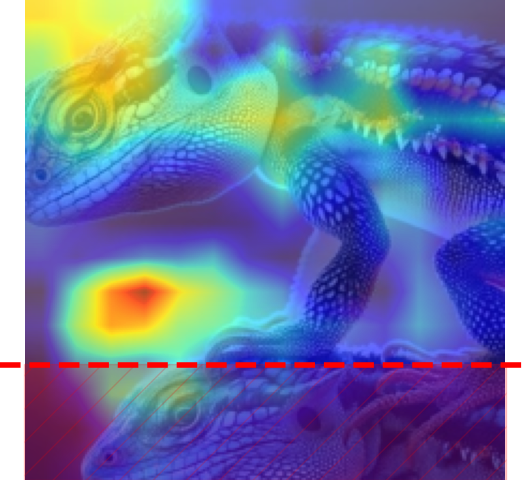} & \includegraphics[width=.2\linewidth,valign=m]{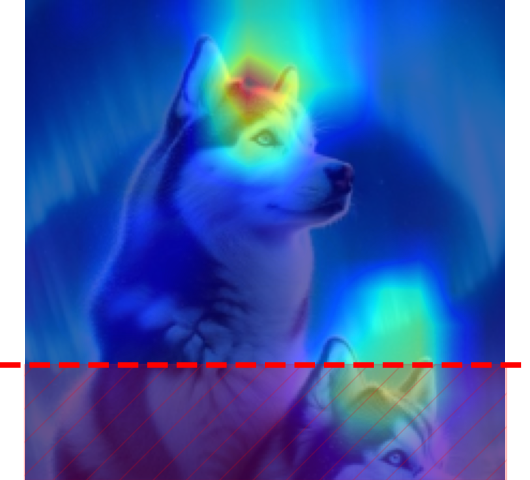} & \includegraphics[width=.2\linewidth,valign=m]{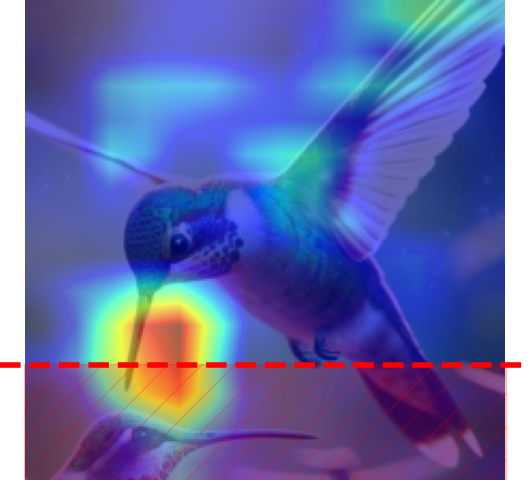}& \includegraphics[width=.2\linewidth,valign=m]{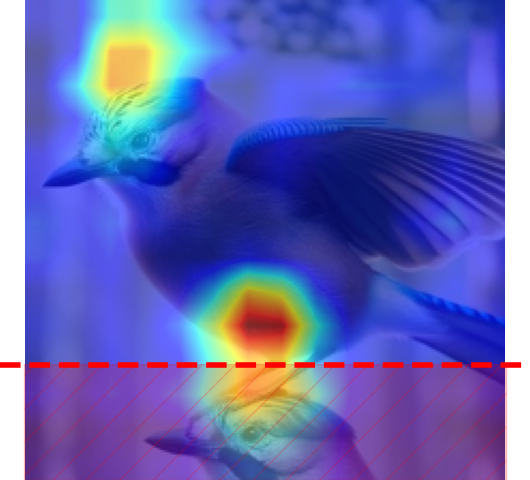}\\[1cm]
\rotatebox[origin=c]{90}{GradCAM++} & \includegraphics[width=.2\linewidth,valign=m]{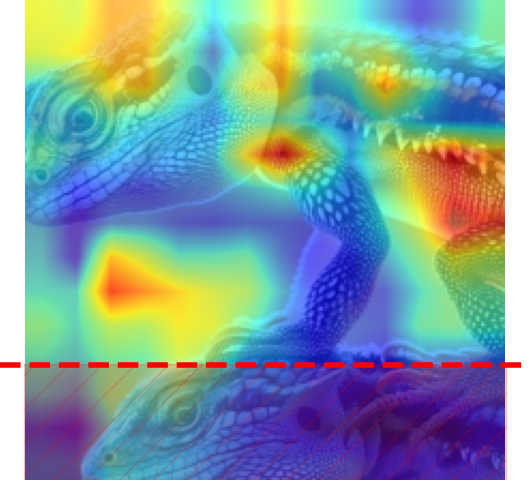} & \includegraphics[width=.2\linewidth,valign=m]{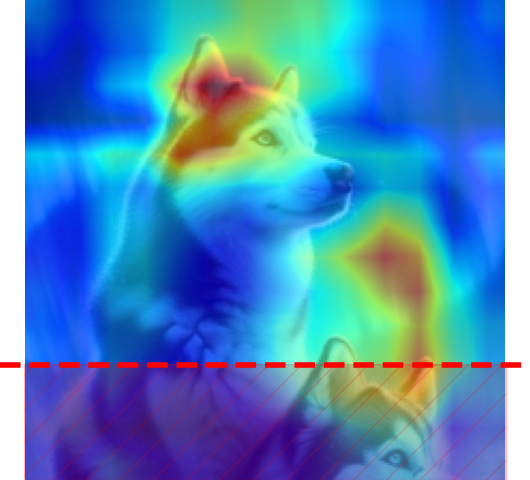} & \includegraphics[width=.2\linewidth,valign=m]{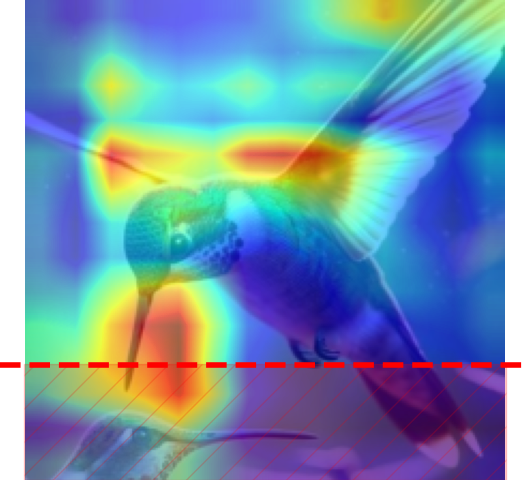}& \includegraphics[width=.2\linewidth,valign=m]{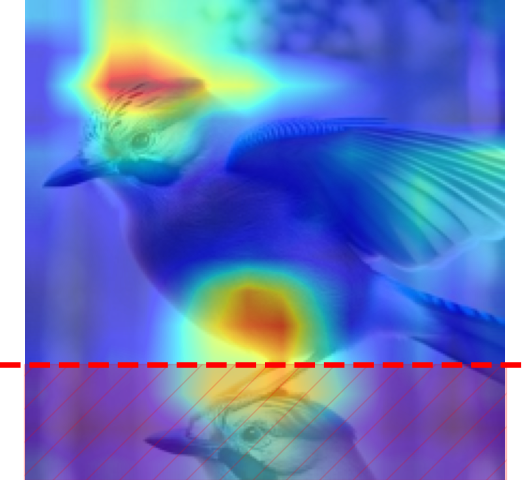}\\[1cm]
\rotatebox[origin=c]{90}{XGradCAM} & \includegraphics[width=.2\linewidth,valign=m]{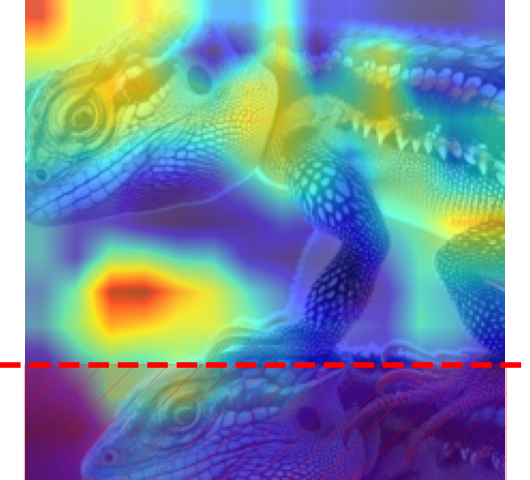} & \includegraphics[width=.2\linewidth,valign=m]{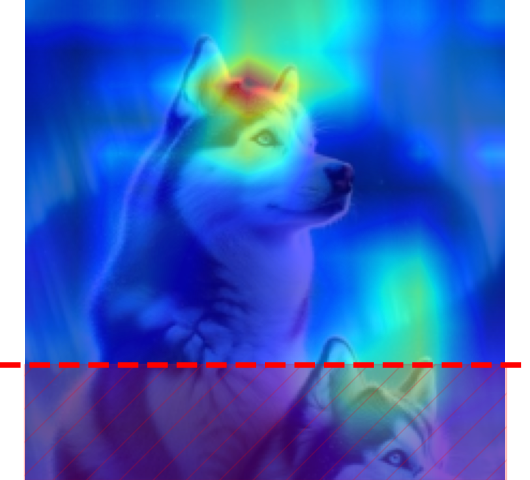} & \includegraphics[width=.2\linewidth,valign=m]{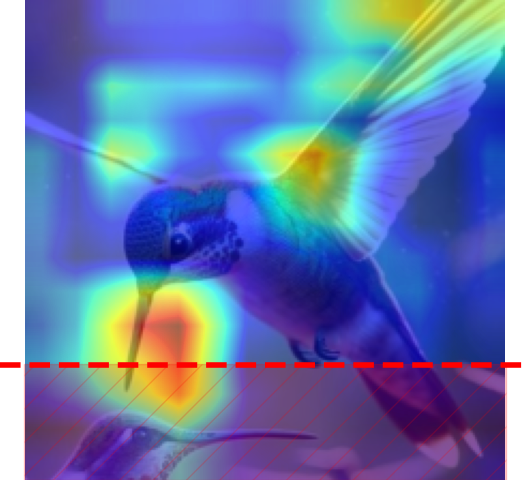}& \includegraphics[width=.2\linewidth,valign=m]{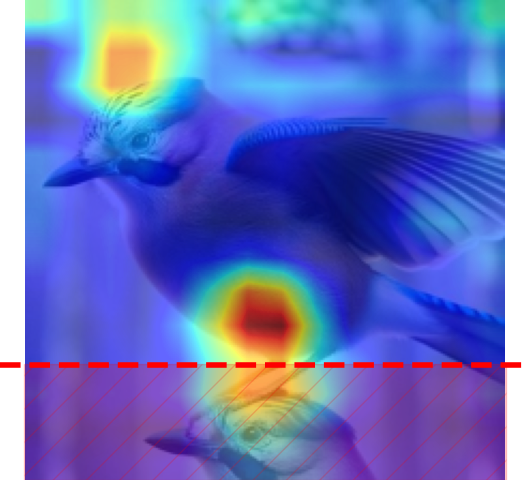}\\[1cm]
\rotatebox[origin=c]{90}{ScoreCAM} & \includegraphics[width=.2\linewidth,valign=m]{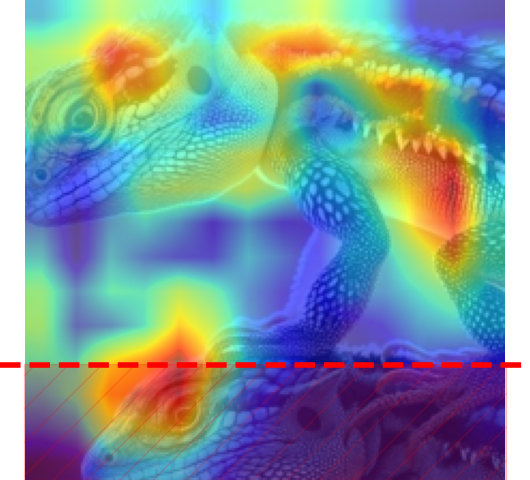} & \includegraphics[width=.2\linewidth,valign=m]{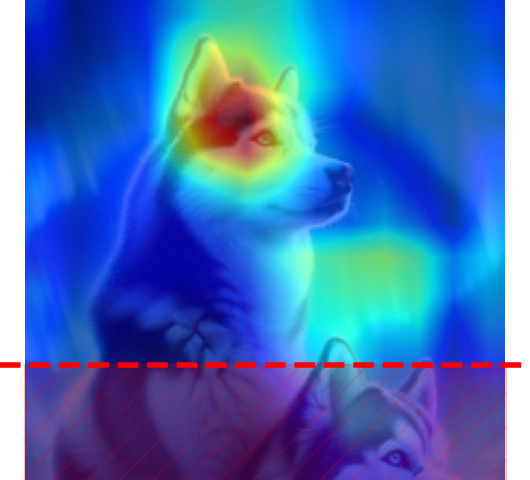} & \includegraphics[width=.2\linewidth,valign=m]{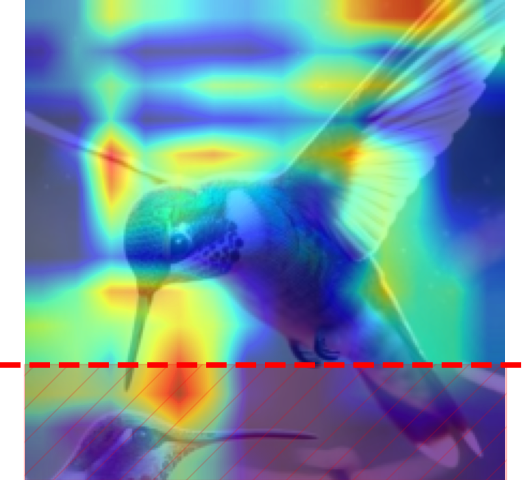}& \includegraphics[width=.2\linewidth,valign=m]{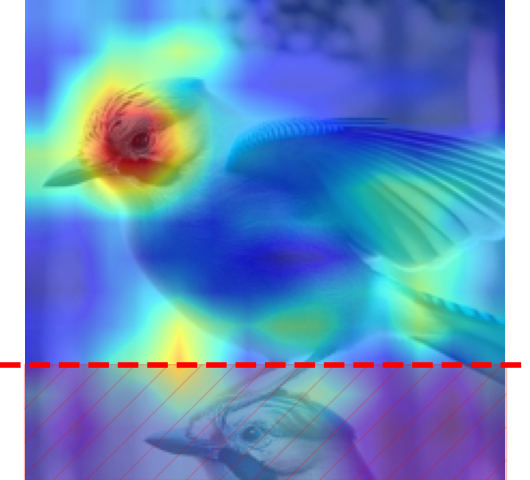}\\[1cm]
\rotatebox[origin=c]{90}{Opti-CAM} & \includegraphics[width=.2\linewidth,valign=m]{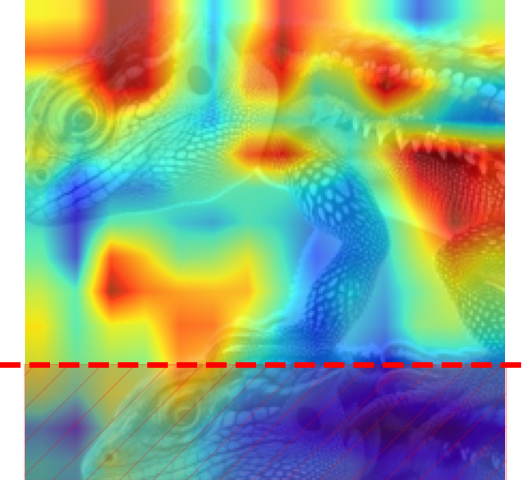} & \includegraphics[width=.2\linewidth,valign=m]{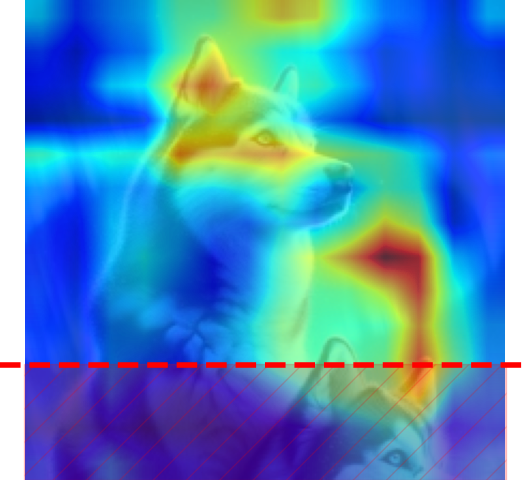} & \includegraphics[width=.2\linewidth,valign=m]{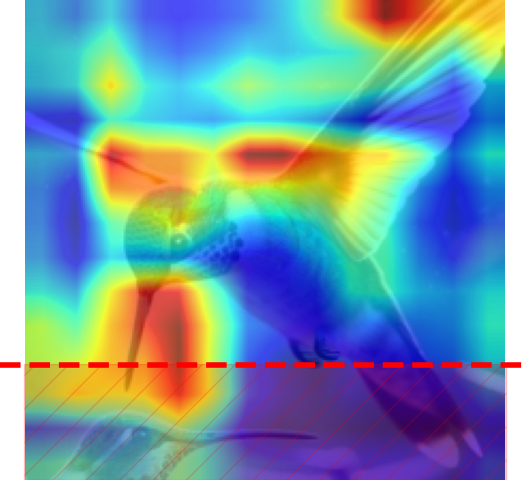}& \includegraphics[width=.2\linewidth,valign=m]{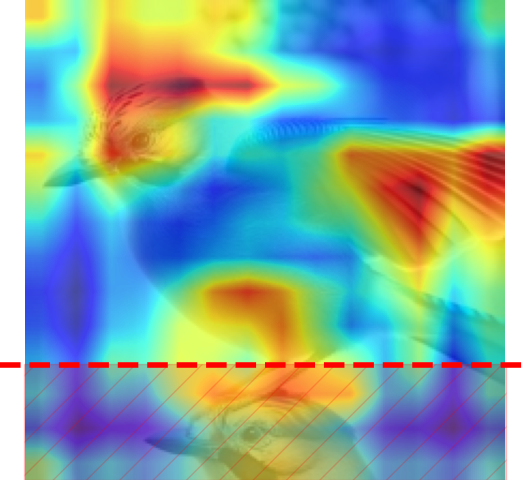}\\[1cm]
\rotatebox[origin=c]{90}{AblationCAM} & \includegraphics[width=.2\linewidth,valign=m]{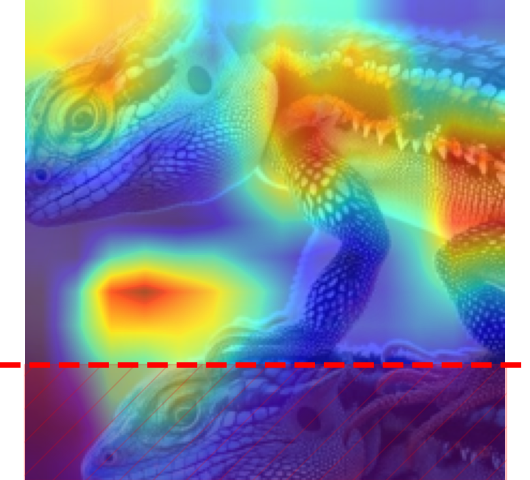} & \includegraphics[width=.2\linewidth,valign=m]{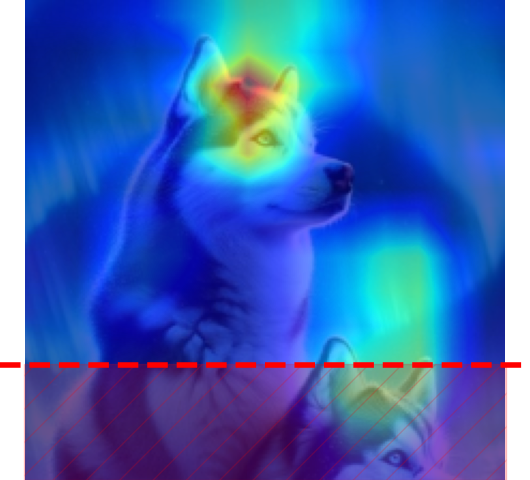} & \includegraphics[width=.2\linewidth,valign=m]{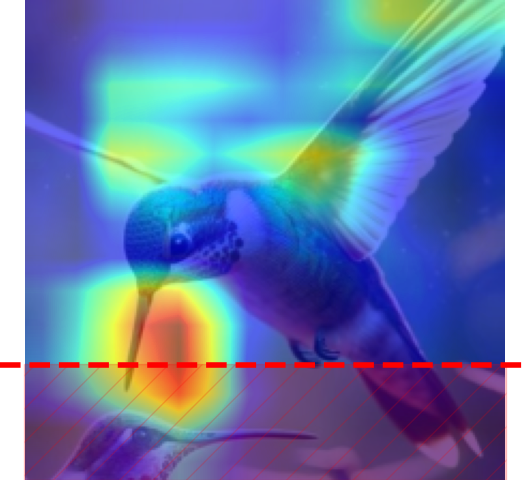}& \includegraphics[width=.2\linewidth,valign=m]{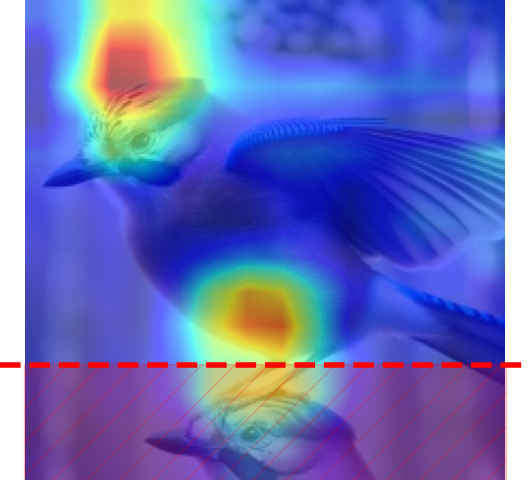}\\[1cm]
\rotatebox[origin=c]{90}{EigenCAM} & \includegraphics[width=.2\linewidth,valign=m]{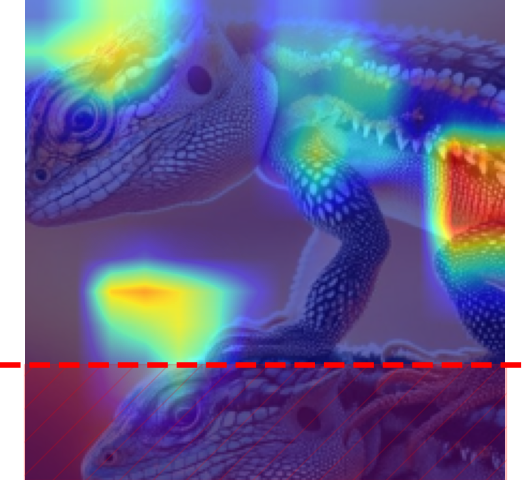} & \includegraphics[width=.2\linewidth,valign=m]{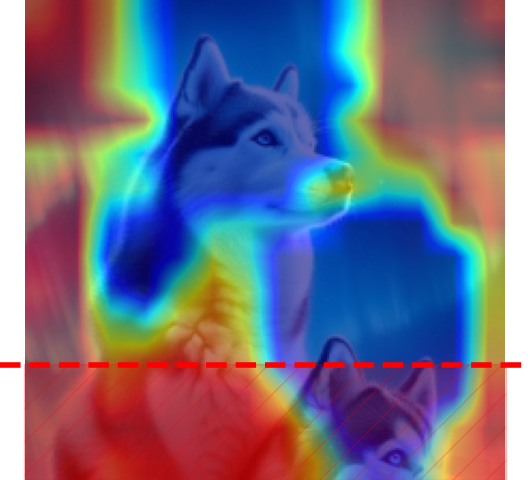} & \includegraphics[width=.2\linewidth,valign=m]{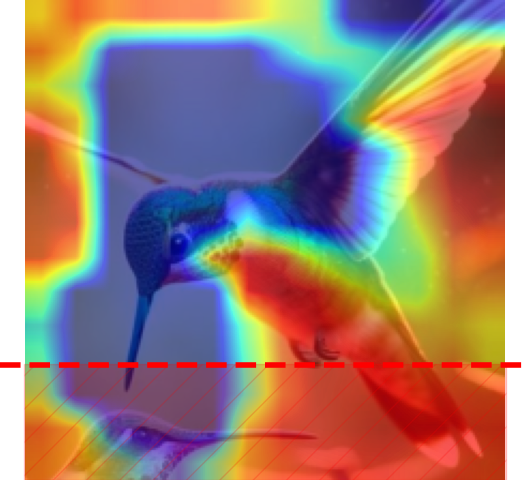}& \includegraphics[width=.2\linewidth,valign=m]{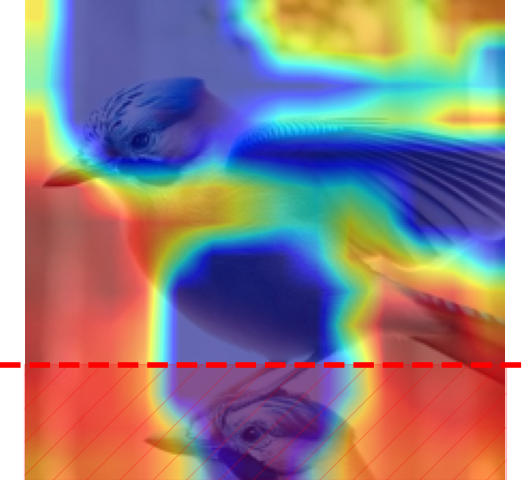}\\[1cm]
\rotatebox[origin=c]{90}{HiResCAM} & \includegraphics[width=.2\linewidth,valign=m]{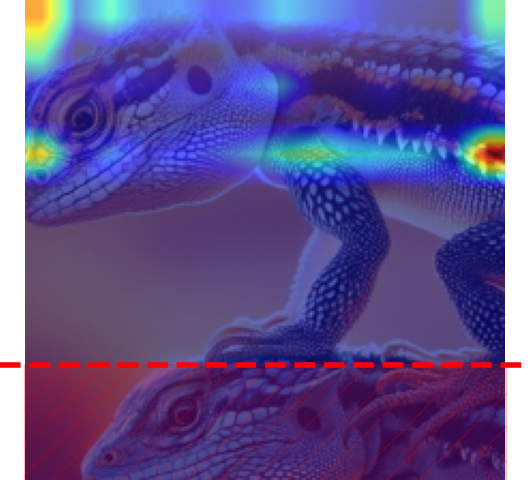} & \includegraphics[width=.2\linewidth,valign=m]{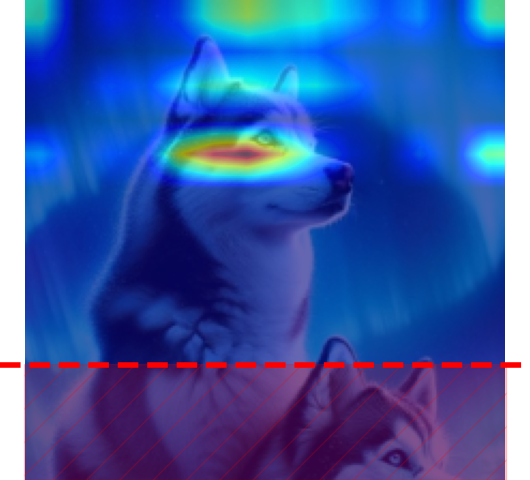} & \includegraphics[width=.2\linewidth,valign=m]{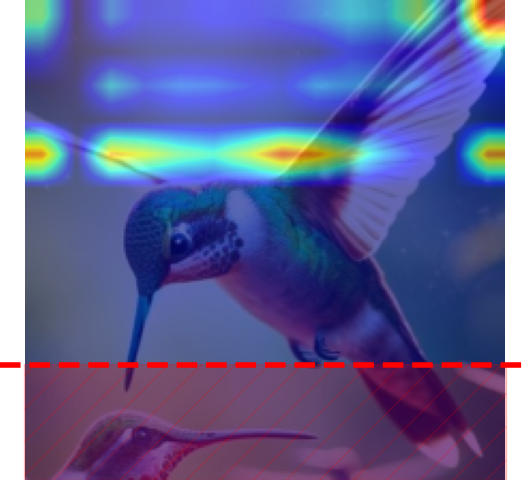}& \includegraphics[width=.2\linewidth,valign=m]{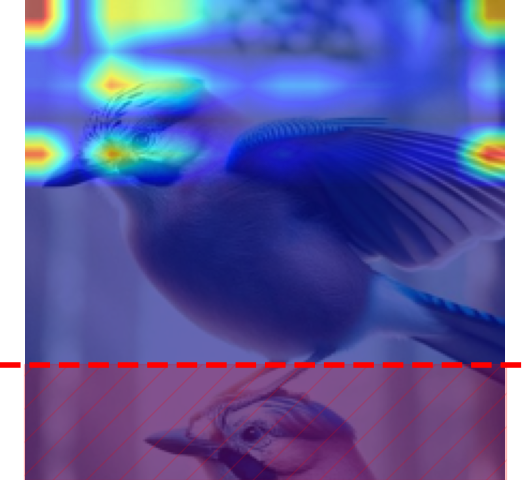}
    \end{tabular}
    \caption{Saliency maps given by the considered CAM-based methods for $\abrev{VGG}$. With the notable exception of HiresCAM, they all highlight parts of images from $\textsc{STACK-GEN}$ which are unseen by the network (this is denoted by the red, rectangular shape in the lower part of the image).}
\label{fig:qualitative-results-1}
\end{figure}
%%%%%%%%%%%%%%%%%%%%%%%%%%%%%%%%%%%%%%%%%%%%%%%%%%%%%%%%%%%%%%%%%%%%%

% Table 1
\input{figures/table_cam}

For our $\abrev{VGG}$, we generate saliency maps from various CAM-based methods on our two datasets, $\textsc{STACK-MIX}$ and $\textsc{STACK-GEN}$, using the predicted category for each example.
We used publicly available implementations whenever possible.
Regarding Opti-CAM, since our model differs from the one described by \cite{zhang_et_al_2023}, we have adjusted the learning rate and number of epochs of the optimization step to achieve a low average drop, as described in the original paper.
For each method, we measure how much of the CAM-based saliency maps emphasize the unseen part, \emph{i.e.}, the dead zone. 
We use the metric $\mu(\cdot)$ defined for a upscaled saliency map $\Cam \in \Reals^{224 \times 224}_+$ as follows:
\begin{equation}
    \mu(\Cam) \defeq \frac{\norm{\Cam_{171:224, \, :}}_2}{\norm{\Cam}_2} \, ,
\end{equation}
where $\norm{\cdot}_2$ is the $\ell^2$-norm and the lower part of the image $\Img_{:, \, 171:224, \, :}$ is unseen by our $\abrev{VGG}$.
We note that for a saliency map $\Cam$, the lower $\mu(\Cam)$, the better.

The results can be found in Table~\ref{tab:cam-evaluation}, Figure~\ref{fig:qualitative-results} and \ref{fig:qualitative-results-1}. 
We observe that every CAM-based methods, except HiResCAM, highlights unseen parts of an image to some extent.
Moreover, the observation are consistent over both datasets.

We believe that HiResCAM avoids this problem because of how its weighting coefficients are computed. 
Following the notation of Section~\ref{sec:setting}, these can be written $\alpha^{(4)}_v \defeq \nabla_{\Bbf^{(v)}} f(\Bbf) \in \Reals^{14 \times 14}$,  
and are applied \emph{globally} to $\Bbf^{(v)}$ (see Definition~4 in the Appendix for more details). 
Because of the masking used in our model, the lower part of $\alpha^{(4)}_v$ is zeroed out, and therefore HiResCAM does not show activity in the lower part of the image. 

We notice that, while at first glance HiResCAM appears to perform well in our setting, it has another issue: 
since $( \alpha^{(4)}_v )_{-9:, \, :} = 0$, the upscaled HiResCAM's saliency map $\Cam \in \Reals^{224 \times 224}_+$ will be zero out in a larger area than the deadzone. 
Namely, $\Cam_{89:224, \, :} = 0$ which represents $61 \%$ of the input image area compared to the $24 \%$ of the deadzone. 
This issue can be observed in Figure~\ref{fig:qualitative-results}.

%%%%%%%%%%%%%%%%%%%%%%%%%%%%%%%%%%%%%%%%%%%%%%%%%%%%%%%%%%%%%%%%%%%%

\section{Conclusion}
\label{sec:conclusion}

In this paper, we looked into several CAM-based methods, with a particular focus on GradCAM. 
We showed that they can highlight parts of the input image that are provably not used by the network. 
This was also showed theoretically, looking at the behavior of GradCAM for a simple, masked CNN at initialization: the saliency map is positive in expectation, even in areas which are unseen by the network. 
Experimentally, this phenomenon appears to remain true, even on a realistic network trained to a good accuracy on ImageNet. 

As future work, we would like to extend the theory to a ResNet-like architecture and other CAM-based methods, such as LayerCAM~\citep{jiang2021layercam}. 
We also would like to multiply the number of images in our two new datasets, with the hope that this framework can become a standard check for saliency maps explanations.

\section*{Acknowledgements}

This work was funded in part by the French Agence Nationale de la Recherche (grant number ANR-19-CE23-0009-01 and ANR-21-CE23-0005-01). 
Most of this work was realized while DG was employed at Universit\'e C\^ote d'Azur. 
We thank Jenny Benois-Pineau for her valuable insights.

%%%%%%%%%%%%%%%%%%%%%%%%%%%%%%%%%%%%%%%%%%%%%%%%%%%%%%%%%%%%%%%%%%%%%%%%%%%%%%%%

\bibliographystyle{unsrtnat}
\bibliography{biblio}

%%%%%%%%%%%%%%%%%%%%%%%%%%%%%%%%%%%%%%%%%%%%%%%%%%%%%%%%%%%%%%%%%%%%%%%%%%%%%%%%

%\newpage 
\appendix

\input{appendix}

\end{document}

%% file: main-setup.tex
\DeclareMathOperator{\Fnn}{\mathcal{F}}

\DeclareMathOperator{\Expec}{\mathbb{E}}% symbol for expectation
\DeclareMathOperator{\Gaussian}{\mathcal{N}}% symbol for Gaussian random variable
\DeclareMathOperator*{\Maximize}{Maximize}% maximization problem
\DeclareMathOperator{\Proba}{\mathbb{P}}% probability
\DeclareMathOperator{\Trace}{Tr}% trace 
\DeclareMathOperator{\Var}{Var}% symbole for the variance
\DeclareMathOperator{\Gap}{GAP}% global average pool
\DeclareMathOperator{\Convo}{\mathcal{C}}% convolutional layer
\DeclareMathOperator{\Maxp}{\mathcal{M}}% maxpooling layer
\DeclareMathOperator{\Activ}{\sigma}% activation func
\DeclareMathOperator{\Model}{[\mathbf{CNN}]}% simple CNN model
\newcommand{\gap}[1]{\Gap\left(#1\right)}

\newcommand{\fnn}[1]{\Fnn\left(#1\right)}

\newcommand{\condexpec}[2]{\mathbb{E}\left[#1\middle|#2\right]}% conditional expectation
\newcommand{\defeq}{\vcentcolon =}% define equals
\newcommand{\expec}[1]{\Expec\left[#1\right]}% expected value
\newcommand{\gaussian}[2]{\Gaussian\left(#1,#2\right)}% Gaussian distribution
\newcommand{\Indic}{\mathds{1}}% indicator
\newcommand{\indic}[1]{\Indic_{#1}}% indicator function
\newcommand{\norm}[1]{\left\lVert#1\right\rVert}% norm of a vector
\newcommand{\Posint}{\mathbb{N}^{\star}}% positive integers 
\newcommand{\proba}[1]{\Proba\left (#1\right )}% probability of an event
\newcommand{\Reals}{\mathbb{R}}% real numbers
\newcommand{\trace}[1]{\Trace\left(#1\right)}% trace of a matrix
\newcommand{\var}[1]{\Var\left(#1\right)}% variance of a random variable
\newcommand{\convo}[1]{\Convo\left(#1\right)}% convolution layer
\newcommand{\maxp}[1]{\Maxp\left(#1\right)}% maxpooling layer
\newcommand{\activ}[1]{\Activ\left(#1\right)}% activation function
\newcommand{\model}[1]{\Model\left(#1\right)}% simple cnn model
\newcommand{\preactiv}[2]{h_{#2}\left(#1\right)}% FNN preactivation
\newcommand{\abrev}[1]{[\mathbf{#1}]}% NAME Abreviations

\newcommand{\Img}{\bm{\xi}} % motif or image xi
\newcommand{\Idt}{\mathrm{Id}} % Identity.
\newcommand{\abf}{\mathbf{a}}% bold font a
\newcommand{\rbf}{\mathbf{r}}% bold font r
\newcommand{\xbf}{\mathbf{x}}% bold font x
\newcommand{\Abf}{\mathbf{A}}% activation
\newcommand{\Bbf}{\mathbf{B}}% rectified activation
\newcommand{\Cbf}{\mathbf{C}}% max pooled activation (input of the fully connected feedforward part of the network)
\newcommand{\Dbf}{\mathbf{D}}
\newcommand{\Fbf}{\mathbf{F}}% filter
\newcommand{\Wbf}{\mathbf{W}}% weights

\newcommand{\mbf}{\mathbf{m}}
\newcommand{\ybf}{\mathbf{y}}
\newcommand{\Cam}{\mathbf{\Lambda}} % CAM heatmap

\newcommand{\Ishp}{H} 
\newcommand{\Ishpp}{W}
\newcommand{\Ashp}{h}
\newcommand{\Ashpp}{w}
\newcommand{\Mshp}{h'}
\newcommand{\Mshpp}{w'}
\newcommand{\Pshp}{k'}
\newcommand{\Fshp}{k}

\makeatletter
\DeclareRobustCommand\onedot{\futurelet\@let@token\@onedot}
\def\@onedot{\ifx\@let@token.\else.\null\fi\xspace}

\def\wrt{w.r.t\onedot}

\makeatother

%% file: figures/tikz/model_mnist1.tex
\tikzstyle{node}=[thick,draw=blue,fill=blue!20,circle,minimum size=22]
\tikzstyle{filter}=[thick, draw=blue, fill=blue!20, rectangle, minimum size=22pt, inner sep=0pt]
\begin{tikzpicture}[x=1.4cm,y=1.4cm]
    % define global variable
    \def\gap{1.2} % gap between layers

    % convolutional layer(s)
    \def\Nfilter{4, 4} % number of filters per convolutional layers
    
      \foreach \N [count=\lay,remember={\N as \Nprev (initially 0);}]
                   in \Nfilter{ % loop over layers
        \foreach \i [evaluate={\y=(\i/10); \x=\lay*\gap + \i*0.05; \prev=int(\lay-1);}]
                     in {1,...,\N}{ % loop over nodes
                     
        \pgfmathsetmacro\halfN{int(\N/2)}
        \pgfmathtruncatemacro{\iInt}{\i} % Convert \i to an integer for comparison
        
        \ifnum \iInt = \halfN % draw a dotted filter
          \node[filter, dotted, fill=gray!20] (F\lay-\i) at (\x,\y) {};
         \else
            \node[filter] (F\lay-\i) at (\x,\y) {};
        \fi
        
           \ifnum \Nprev>0 % connect to previous layer
           \ifnum \i = \N
             \draw[thick, dotted, red] (F\prev-\Nprev.north east) -- (F\lay-\N);
             \draw[thick, dotted, red] (F\prev-\Nprev.south east) -- (F\lay-\N);
           \fi
           \fi
        }
      }
      
    \node[above = 0.3 cm of F1-4] {Conv $\mathcal{C}$};
    \node[below = 0.3 cm of F1-1] {
    \begin{tabular}{c}
    $\Abf \in$ \\
    $\Reals^{V \times h \times w}$
    \end{tabular}};
    \node[above = 0.3 cm of F2-4](texte-relu) {ReLU};
    \node[below = 0.3 cm of F2-1] {\begin{tabular}{c}
    $\Bbf \in$ \\
    $\Reals^{V \times h \times w}$
    \end{tabular}};

    % maxpooling layer(s)
    \tikzstyle{filter}=[thick, draw=blue, fill=blue!20, rectangle, minimum size=18pt, inner sep=0pt] % redefine size of channels
    
    \def\Nmax{4} % number of maxpooled channels
    \def\layer{3}
    \def\PrevSize{4}
          \foreach \N [count=\lay,remember={\N as \Nprev (initially 0);}]
                   in \Nmax{ % loop over layers
        \foreach \i [evaluate={\y=(\i/10); \x=\lay + \i*0.05 + (\layer-1 + 0.3)*\gap ; \prev=int(\lay-1);}]
                     in {1,...,\N}{ % loop over nodes
                     
        \pgfmathsetmacro\halfN{int(\N/2)}
        \pgfmathtruncatemacro{\iInt}{\i} % Convert \i to an integer for comparison
        \ifnum \iInt = \halfN % draw dotted channel(s)
          \node[filter, dotted, fill=gray!20] (M\lay-\i) at (\x,\y) {};
         \else
            \node[filter] (M\lay-\i) at (\x,\y) {};
        \fi
        
        % draw connections
        \ifnum \i = \N
            \pgfmathsetmacro\PrevLayer{int(\layer-1)}
            \draw[thick, dotted, red] (F\PrevLayer-\PrevSize.north east) -- (M\lay-4);
            \draw[thick, dotted, red] (F\PrevLayer-\PrevSize.south east) -- (M\lay-4);
        \fi
        }
      }
    
    \node[above = 0.34 cm of M1-4](text-maxpool) {max pool $\mathcal{M}$};
    \node[below = 0.34 cm of M1-1] {\begin{tabular}{c}
    $\Cbf  \in$ \\
    $\Reals^{V \times h' \times w'}$
    \end{tabular}};
    
    % draw flatten layer
    \node[thick, draw=blue, fill=blue!20, rectangle, minimum width=1pt, minimum height=50pt, right = \gap-0.5 of M1-4] (flat1) {};
    \draw[thick, dotted, red] (M1-4.north east) -- (flat1.north west);
    \draw[thick, dotted, red] (M1-4.south east) -- (flat1.south west);
    
    \node[above = 0.19 cm of flat1] {flatten};
    \node[below = 0.19 cm of flat1] {\begin{tabular}{c}
    $\Cbf' \in$ \\
    $\Reals^{Vh'w'}$
    \end{tabular}};

    % draw fully connected layers
    \node[thick, draw=blue, fill=blue!20, rectangle, minimum width=1pt, minimum height=20pt, right = of flat1] (fnn1) {$\Fnn$};
    \draw[thick, red, ->] (flat1) -- (fnn1);
    
    \node[right = of fnn1] (score) {$y_c \in \Reals$};
    \draw[thick, red, ->] (fnn1) -- (score);

    % draw input 
    \node[thick, draw=black, fill=black!20, rectangle, minimum width=40pt, minimum height=40pt, left = \gap-0.5 of F1-2] (img) {};
    \draw[thick, red, ->] (img) -- (F1-2);
    
    \node[below = 0.1 cm of img] {\begin{tabular}{c}
    $\Img \in$ \\
    $[0,1]^{H \times W}$
    \end{tabular}};
    \node[above = 0.2 cm of img](texte-img) {Input};

    % draw annotations of layers

    \draw[decorate,decoration={brace, raise=0.2cm}, thick] (texte-img.north west) -- (texte-relu.north east) node[midway, above, yshift=0.4cm] {$g \colon \Reals^{H \times W} \to \Reals^{V \times h \times w}$};
    \draw[decorate,decoration={brace, raise = 0.8cm}, thick, above=5cm] (texte-relu.north east) -- ([yshift=0.9cm]score.north east) node[midway, yshift=1.cm] {$f \colon \Reals^{V \times h \times w} \to \Reals$};

\end{tikzpicture}

%% file: figures/tikz/vgg.tex
\pgfdeclarelayer{bg}    % declare background layer
\pgfsetlayers{bg,main}  % set the order of the layers (main is the standard layer)

\definecolor{ao(english)}{rgb}{0.0, 0.5, 0.0}
\tikzstyle{node}=[thick,draw=blue,fill=blue!20,circle,minimum size=22]
\tikzstyle{filter}=[thick, draw=blue, fill=blue!20, rectangle, minimum size=22pt, inner sep=0pt]

% glowing arrows https://tex.stackexchange.com/questions/80171/faded-or-blurred-lines/80207#80207
\tikzset{
  laser beam action/.style={
    line width=\pgflinewidth+.2pt,draw opacity=.1,draw=#1,
  },
  laser beam recurs/.code 2 args={%
    \pgfmathtruncatemacro{\level}{#1-1}%
    \ifthenelse{\equal{\level}{0}}%
    {\tikzset{preaction={laser beam action=#2}}}%
    {\tikzset{preaction={laser beam action=#2,laser beam recurs={\level}{#2}}}}
  },
  laser beam/.style={preaction={laser beam recurs={10}{#1}},draw opacity=1,draw=#1},
}

% https://tex.stackexchange.com/questions/219356/how-to-create-a-rectangle-filled-with-image-using-tikz
% define filter with img inside
\tikzset{
  path image/.style args={#1}{
    path picture={
      \node at (path picture bounding box.center) {
        \includegraphics[height=3cm]{#1}};
    }
  },
  path tikzimage/.style={
    path picture={
      \node at (path picture bounding box.center)
        [circle, fill=blue!50, scale=2, text=yellow]{Bravo};
    }
  }
}

% two colored arrows : https://tex.stackexchange.com/questions/72784/arrow-with-two-colors-with-tikz
\tikzset{
  double arrow/.style args={#1 colored by #2 and #3}{
    -stealth,line width=#1,#2, % first arrow
    postaction={draw,-stealth,#3,line width=(#1)/3,
                shorten <=(#1)/3,shorten >=2*(#1)/3}, % second arrow
  }
}

\begin{tikzpicture}[x=2.2cm,y=1.4cm,scale=1.5, every node/.style={transform shape}]
    
    % define global variable
    \def\gap{.7} % gap between layers

    % convolution
    \def\Nfilter{4} % number of filters per filters per layer
    
      \foreach \N [count=\lay,remember={\N as \Nprev (initially 0);}]
                   in \Nfilter{ % loop over layers
        \foreach \i [evaluate={\y=(\i/10); \x=\lay + \i*0.05; \prev=int(\lay-1);}]
                     in {1,...,\N}{ % loop over nodes
                     
        \pgfmathsetmacro\halfN{int(\N/2)}
        \pgfmathtruncatemacro{\iInt}{\i} % Convert \i to an integer for comparison
        \ifnum \iInt = \halfN % draw ... (i.e, etc...)
          \node[filter, dotted, fill=black, pattern=north west lines, pattern color=black] (F\lay-\i) at (\x,\y) {};
         \else
            \node[shade, blur shadow={shadow blur steps=5}, draw=white, thick, inner sep = 0 pt] (F\lay-\i) at (\x,\y) {\includegraphics[width=22pt, height=22pt]{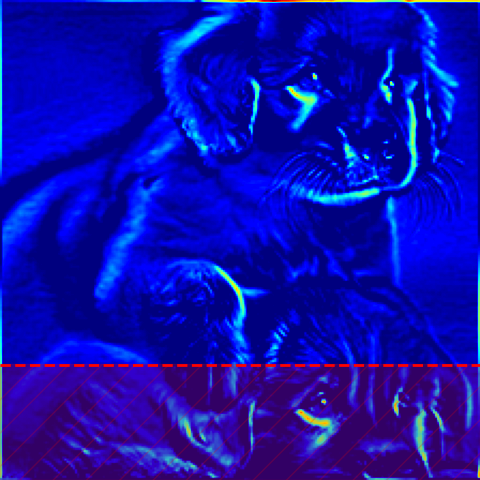}};
        \fi
        
          % \ifnum \Nprev>0 % connect to previous layer
          % \ifnum \i = \N
          %   \draw[thick, dotted, red] (F\prev-\Nprev.north east) -- (F\lay-\N);
          %   \draw[thick, dotted, red] (F\prev-\Nprev.south east) -- (F\lay-\N);
          % \fi
          % \fi
        }
      }

    % maxpooling
    \tikzstyle{filter}=[thick, draw=blue, fill=blue!20, rectangle, minimum size=18pt, inner sep=0pt] % redefine size of channels
    \def\Nmax{4} % number of maxpooled channels
    \def\layer{2}
    \def\PrevSize{4}
          \foreach \N [count=\lay,remember={\N as \Nprev (initially 0);}]
                   in \Nmax{ % loop over layers if you want to loop on same type of layers
        \foreach \i [evaluate={\y=(\i/10); \x=\lay + \i*0.05 + \gap ; \prev=int(\lay-1);}]
                     in {1,...,\N}{ % loop over nodes
                     
        \pgfmathsetmacro\halfN{int(\N/2)}
        \pgfmathtruncatemacro{\iInt}{\i} % Convert \i to an integer for comparison
        \ifnum \iInt = \halfN % draw dotted channel(s)
          \node[filter, dotted, fill=black, pattern=north west lines, pattern color=black] (M1-\i) at (\x,\y) {};
         \else
            \node[shade, blur shadow={shadow blur steps=5}, draw=white, thick, inner sep = 0 pt] (M1-\i) at (\x,\y) {\includegraphics[width=18pt, height=18pt]{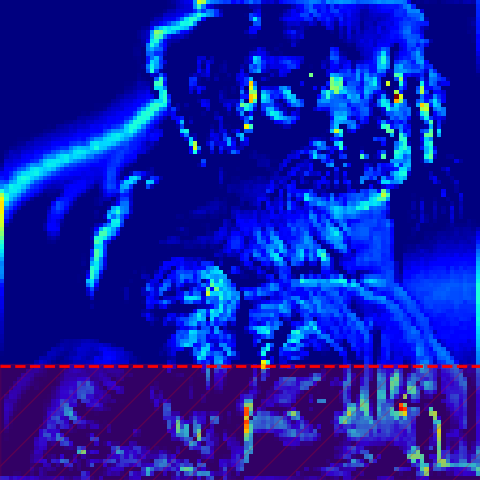}};
        \fi
        
        % draw connections
        \ifnum \i = \N
            \pgfmathsetmacro\PrevLayer{int(\layer-1)}
            \draw[thick, dotted, ao(english)] (F\PrevLayer-\PrevSize.north east) -- (M1-4);
            \draw[thick, dotted, ao(english)] (F\PrevLayer-\PrevSize.south east) -- (M1-4);
        \fi
        }
      }

    \tikzstyle{filter}=[thick, draw=blue, rectangle, minimum size=18pt, inner sep=0pt] % redefine size of channels
    \def\Nmax{4} % number of maxpooled channels
    \def\layer{3}
    \def\PrevSize{4}
          \foreach \N [count=\lay,remember={\N as \Nprev (initially 0);}]
                   in \Nmax{ % loop over layers
        \foreach \i [evaluate={\y=(\i/10); \x=\lay + \i*0.05 + 2*\gap ; \prev=int(\lay-1);}]
                     in {1,...,\N}{ % loop over nodes
                     
        \pgfmathsetmacro\halfN{int(\N/2)}
        \pgfmathtruncatemacro{\iInt}{\i} % Convert \i to an integer for comparison
        \ifnum \iInt = \halfN % draw dotted channel(s)
            \node[filter, dotted, fill=black, pattern=north west lines, pattern color=black] (M2-\i) at (\x,\y) {};
         \else
            \coordinate (pos) at (\x,\y);
            % Fill a vertical rectangle corresponding to the left half of the square
            \fill[red!70] ([xshift=9pt]pos) rectangle ([xshift=0pt,yshift=-9pt]pos);
            \fill[red!70] ([xshift=-9pt]pos) rectangle ([xshift=0pt,yshift=-9pt]pos);
            
            \fill[blue!70] ([xshift=+9pt]pos) rectangle ([xshift=0pt,yshift=9pt]pos);
            \fill[blue!70] ([xshift=-9pt]pos) rectangle ([xshift=0pt,yshift=9pt]pos);            
            
            \node[filter] (M2-\i) at (pos) {};
        \fi
        
        }
      }

    % draw bouding boxes
    \begin{pgfonlayer}{bg}    % select the background layer
        \node[draw, rounded corners, color=ao(english), fill = ao(english)!10, thick, inner sep = 15pt, inner xsep=10pt, fit=(F1-1) (F1-2) (F1-2) (F1-3) (F1-4) (M1-1) (M1-2) (M1-3) (M1-4)] (box) {};
    \end{pgfonlayer}

    % \draw[thick, ao(english), ->, ] (M1-4) -- (M2-4);
    \draw[double arrow=2pt colored by white and ao(english)] (M1-4) -- (M2-4);

    % draw linear layer
    \node[thick, draw=blue, fill=blue!20, rectangle, minimum width=1pt, minimum height=50pt, right = 20pt of M2-4] (linear) {};
    % \draw[thick, ao(english), ->] (M2-4) -- (out);
    \draw[double arrow=2.pt colored by white and ao(english)] (M2-4) -- (linear);

    % draw output layer
    \node[thick, draw=blue, fill=blue!20, rectangle, minimum width=1pt, minimum height=40pt, right = 20pt of linear] (out) {};
    % \draw[thick, ao(english), ->] (M2-4) -- (out);
    \draw[double arrow=2.pt colored by white and ao(english)] (linear) -- (out);

    % draw input 
    \node[draw, color=blue, very thick, inner sep=0pt, left = 25pt of F1-4] (img) {\includegraphics[width=30pt, height=30pt]{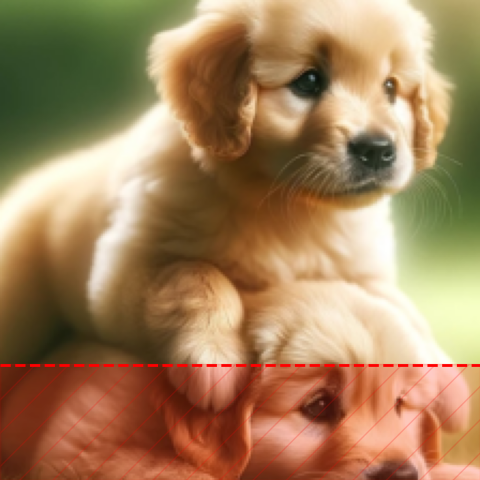}};
    % \path[laser beam=white] (img) -- (F1-4);
    % \draw[thick, green, ->] (img) -- (F1-4);

    \draw[double arrow=1.5pt colored by white and ao(english)] (img) -- (F1-4);
    
    \node[above = 12 pt of img] {
    $[0,1]^{3 \times 224 \times 224}$
    };

    \node[below = 18 pt of img] {
    $\Img$
    };

    % draw annotations of layers
    \node[below =1 pt of F1-1] {$\times 2 \rightarrow \times 3$};
        \node[below =1 pt of box] {$\times 2 \rightarrow \times 3$};
    
    % \draw[|-latex, red] ([yshift=-40pt]M2-4.south) -- ([yshift=-38pt]F1-4.south)
    % node[midway, below] {impact of $\mathbf{W}$};

    \node at (M2-4) {\contour{white}{$\mathbf{W}_{1}$}};

    \node[above = 30pt of F1-1, xshift = 6pt] {    
    \begin{tabular}{c}
    conv $+$ ReLU \\
    $3\times3$
    \end{tabular}
    };
    
    \node[above = 32 pt of M1-1, xshift = 6pt] {
    \begin{tabular}{c}
    max-pool \\
    $2\times2$
    \end{tabular}
    };
    
    \node[above = 30 pt of M2-1, xshift = 8pt] {
    $\mathbb{R}^{4096 \times (256 \times 14 \times 14)}$};

    \node[below = 13 pt of M2-1, xshift = 8pt] {
    $\mathbf{W}$
    };

    \node[above = 11 pt of linear, xshift = 3pt] {$\mathbb{R}^{4096}$};
    
    \node[above = 7 pt of out, xshift = -6pt] { % 15.6 pt , 6pt
    $\mathbb{R}^{1000}$
    };

    \node[below = 14.5 pt of out] {
    $\mathbf{y}$};

    % \node[below = 1 pt of M1-1, xshift = 6pt] {$(\star)$};

%     \draw[decorate,decoration={brace,raise=0.1cm}]
% (texte.south east) -- ($(texte.south west)+(1,0)$) ;

\end{tikzpicture}

%% file: figures/table_cam.tex
\begin{table}[t] 
\caption{\label{tab:cam-evaluation}Activity in the unseen part of the image, measured by $\mu(\cdot) \times 100$ for several CAM-based methods on both proposed datasets (only images in the validation set are considered).
}

\def\arraystretch{1.2}
\centering
\begin{tabular}{@{}l@{\hspace{0.7cm}}r@{\hspace{0.7cm}}r@{}} 
\toprule

methods  & \textsc{STACK-MIX} $\downarrow$ & \textsc{STACK-GEN} $\downarrow$ \\

\midrule

GradCAM~\citep{selvaraju_et_al_2017} & $22.7 \pm 13.4$ & $21.6 \pm 11.6$  \\

GradCAM++~\citep{ChattopadhayCam} & $28.8 \pm 8.1$ & $28.5 \pm 7.9$  \\

XGradCAM~\citep{fuXcam} & $23.8 \pm 9.0$ &  $22.8 \pm 9.0$ \\

ScoreCAM~\citep{wangScorecam} & $19.9 \pm 10.3$ &  $18.5 \pm 10.6$ \\

Opti-CAM~\citep{zhang_et_al_2023} & $32.7 \pm 7.9$ &  $32.0 \pm 7.8$ \\

AblationCAM~\citep{desaiAblationcam} & $21.0 \pm 9.9$ &  $20.8 \pm 9.6$ \\

EigenCAM~\citep{muhammadEigencam} & $51.7 \pm 19.7$ &  $55.8 \pm 21.6$ \\

HiResCAM~\citep{draelos_carin_2021} & $0.0 \pm 0.0$ &  $0.0 \pm 0.0$ \\

\bottomrule

\end{tabular}
\end{table}

%% file: appendix.tex
%%%%%%%%%%%%%%%%%%%%%%%%%%%%%%%%%%%%%%%%%%%%%%%%%%%%%%%%%%%%%%%%%%%%%%%%%%%%%%%

\newpage
\section{Other definitions}
\label{app:defs}

The work of \cite{fuXcam} proposed an alternative to $\abrev{GC}$ by re-scaling the gradient in the weighting coefficient $\alpha$. 
The main objective of this alternative is to obtain a method that satisfies two axioms (sensitivity and conservation properties) defined by \cite{fuXcam}, which, according to their findings, enhances visualization performance \wrt specific metrics.

\begin{definition}[XGradCAM]
\label{def:xgradcam}
In our notation, for an input $\Img$ and model $\Model$, the XGradCAM feature scores are given by
\begin{align*}
\abrev{XC} \defeq \activ{\sum_{v = 1}^{V} \alpha_v^{(2)} \Bbf^{(v)}} \in \Reals_+^{h \times w}
\, ,
\end{align*}
where each $\alpha_v^{(2)} \defeq \gap{\frac{\Bbf^{(v)}}{\norm{\Bbf^{(v)}}_1} \odot \nabla_{\Bbf^{(v)}} f(\Bbf)}$ is the \emph{global average pooling} of the gradient of $f$ at $\Bbf^{(v)}$ rescaled by the normalized activation $\frac{\Bbf^{(v)}}{\norm{\Bbf^{(v)}}_1}$.
\end{definition}

GradCAM++~\citep{ChattopadhayCam} then modified GradCAM by computing rectified gradient together with second and third order derivative information. 
This alternative appears to produce better saliency maps in cases where the input image contains multiple subjects. 
\begin{definition}[GradCAM++]
\label{def:gradcamplusplus}
In our notation, for an input $\Img$ and model $\Model$, the GradCAM++ feature scores are given by
\begin{align*}
\abrev{C+} \defeq \activ{\sum_{v = 1}^{V} \alpha_v^{(3)} \Bbf^{(v)}} \in \Reals_+^{h \times w}
\, ,
\end{align*}
where each $\alpha_v^{(3)} \defeq \gap{ \frac{\partial_{\Bbf^{(v)}}^2 f(\Bbf)}{2 \partial_{\Bbf^{(v)}}^2 f(\Bbf) + \norm{\Bbf^{(v)}}_1 \partial_{\Bbf^{(v)}}^3 f(\Bbf)} \odot \activ{\nabla_{\Bbf^{(v)}} f(\Bbf)}}$ is the \emph{global average pooling} of the gradient of $f$ at $\Bbf^{(v)}$ rescaled by second and third order derivatives defined as 
\[
\partial_{\Bbf^{(v)}}^n f(\Bbf) \defeq \left( \partial_{\Bbf^{(v)}_{i,j}}^n f(\Bbf) \right)_{i,j \in [h] \times [w]}
\, .
\]
\end{definition}

More recently, HiResCAM~\citep{draelos_carin_2021} proposed to replace the averaging of the gradient over the map by the \emph{element-wise} multiplication between gradient and activation. 
In our notation: 
%This formulation of the weighting coefficients thus prevent HiResCAM from highlighting unseen parts of the image.

\begin{definition}[HiResCAM]
\label{def:hirescam}
For an input $\Img$ and model $\Model$, the HiResCAM feature scores are given by
\begin{align*}
\abrev{HC} \defeq \activ{\sum_{v = 1}^{V} \alpha^{(4)}_v \odot \Bbf^{(v)}} \in \Reals_+^{h \times w}
\, ,
\end{align*}
where $\alpha^{(4)}_v \defeq \nabla_{\Bbf^{(v)}} f(\Bbf) \in \Reals^{h \times w}$ and $\odot$ is the element-wise matrix product. 
\end{definition}

The work of \cite{wangScorecam} got rid of the dependency on gradients of the weighting coefficients $\alpha$. 
% Indeed, ScoreCAM~\citep{wangScorecam} first computes the activation maps. Then use each map to mask the input image before passing through the network, to measure the change in prediction which will define the weight of each activations.

\begin{definition}[ScoreCAM]
\label{def:scorecam}
In our notation, for an input $\Img$ and model $\Model$, the ScoreCAM feature scores are given by
\begin{align*}
\abrev{SC} \defeq \activ{\sum_{v = 1}^{V} \alpha_v^{(5)} \Bbf^{(v)}} \in \Reals_+^{h \times w} 
\, .
\end{align*}
In the previous display, $\alpha^{(5)}=\mathrm{softmax}(\beta^{(5)})$, where $\beta_v^{(5)} \defeq \model{\Img \odot H^{(v)}} - \model{\Img_b}$, $H^{(v)} \defeq s\left(\mathrm{Up}\left(\Bbf^{(v)}\right)\right)$ the upsampled and normalized activation map $\Bbf^{(v)}$, $s(\cdot)$ a normalization function that maps matrix values into $[0,1]$, and $\mathrm{Up}(\cdot)$ an upsampling function which resize a matrix to the size of $\Img$. The baseline input image $\Img_b$ can be taken as $\Img_b \defeq \Img$.
\end{definition}

Similar to ScoreCAM, the work of \cite{zhang_et_al_2023} introduces Opti-CAM, a method that uses optimization to compute the weighting coefficient $\alpha$.

\begin{definition}[Opti-CAM]
\label{def:opticam}
Define 
\[
H_\beta\defeq s\left(\mathrm{Up}\left(\sum_{v=1}^{V} \mathrm{softmax}(\beta)_v \Bbf^{(v)}\right)\right)
\, ,
\]
where $s(\cdot)$ is a normalization function that maps matrix values into $[0,1]$, and $\mathrm{Up}(\cdot)$ an upsampling function which resizes a matrix to the size of $\Img$. 
Take $\alpha^{(6)}=\mathrm{softmax}(\beta^{\star})$, where $\beta^{\star}$ is solution to the optimization problem
\begin{align*}
%\beta^{\star} \in 
\Maximize_{\beta \in \Reals^V} 
%\underset{\beta \in \Reals^V}{\Argmax} \, g_c \left(
\model{\Img \odot H_\beta } 
%\right) 
\, .
\end{align*}
Then, for an input $\Img$, the Opti-CAM feature scores are given by
\begin{align*}
\abrev{OC} \defeq \activ{\sum_{v = 1}^{V} \alpha_v^{(6)} \Bbf^{(v)}} \in \Reals_+^{h \times w} 
\, .
\end{align*}
%and $g_c(\cdot)$ a function that operates on the logit vector $\ybf$, outputting a scalar. 
%Since our logit vector is scalar, we define $g_c(y) \defeq y$.
\end{definition}

% We detail the settings for the Opti-CAM entry of Table~\ref{tab:cam-evaluation}. 
% Since our model applies CAM on the penultimate convolutional layer, we adapt the learning rate and number of epochs so Opti-CAM offers a low average drop, as seen in \cite{zhang_et_al_2023}.
%and obtaining visually satisfying saliency maps in our setting.

It is evident from the definitions of ScoreCAM and OptiCAM that the ReLU function $\sigma(\cdot)$ is redundant when $\Bbf$ represents the rectified activation maps, since they are already positive.

Finally, AblationCAM~\citep{desaiAblationcam} removes each activation to observe the impact on the prediction. 
Ablations resulting in larger drops receive higher weights.
AblationCAM is similar to ScoreCAM: a sort of masking is performed to observe changes in prediction and no gradient is required. 
However, both methods require lots of forward passes through the network.

\begin{definition}[AblationCAM]
\label{def:ablationcam}
In our notation, for an input $\Img$ and model $\Model$, the AblationCAM feature scores are given by
\begin{align*}
\abrev{AC} \defeq \activ{\sum_{v = 1}^{V} \alpha_v^{(7)} \Bbf^{(v)}} \in \Reals_+^{h \times w}
\, ,
\end{align*}
where each $\alpha_v^{(7)} \defeq \frac{\ybf - \ybf_v}{\ybf}$, $\ybf \defeq \model{\Img}$ the predicted score by our model and $\ybf_v$ is the output of $\Model$ when the activation map $\Bbf^{(v)}$ is zero out.
\end{definition}

It should be clear that the method we use to compute the weighting coefficients $\alpha$ of $\abrev{GC}$, in Proposition~\ref{prop:l-hidden-layer}, yields the coefficients $\alpha^{(4)}$ for $\abrev{HC}$ and $\alpha^{(2)}$ for $\abrev{XC}$, since in proving Proposition~\ref{prop:l-hidden-layer} we have computed $\nabla_{\Bbf^{(v)}} f(\Bbf)$, which is the challenging part.

%%%%%%%%%%%%%%%%%%%%%%%%%%%%%%%%%%%%%%%%%%%%%%%%%%%%%%%%%%%%%%%%%%%%%%%%%%%%%%%%%%%

%\section{Additional experimental details}

% DGA: not so informative, this goes to additional experimental details
% \begin{remark}
%     The actual shape of $\Wbf$ in our model is $4096 \times 50176$ where $50176 = 256 \times 14 \times 14$. This matches the output shape $256 \times 14 \times 14$ of the final activation map $\Bbf$. That is why we index the lines of $\Wbf$ as if it was of the same shape as $\Bbf$. When in practice, $\Bbf$ is flattened into a single-dimensional array before it is fed into the first dense layer. For a visual depiction of this statement, one can see the first dense layer represented by $\Wbf$ in Figure~\ref{fig:vgg}
% \end{remark}

% DGA: speak more about training recipe : scheduler etc... 

%%%%%%%%%%%%%%%%%%%%%%%%%%%%%%%%%%%%%%%%%%%%%%%%%%%%%%%%%%%%%%%%%%%%%%%%%%%%%%%

\section{Technical results}
\label{sec:technical-results}

% DGA: this is not finished, should talk about B since we apply the lemmas to B
As a consequence of Eq.~\eqref{eq:def-activation-single-conv}, $\Abf$ (seens as a vector of size $\Ashp\times \Ashpp$) is a centered Gaussian vector, with covariance matrix given by
\begin{equation}
\label{eq:covariance-post-convolution}
\forall (i,j),(i',j')\in ([\Ashp]\times [\Ashpp])^2,\,\,   
\left(\Sigma^A\right)_{(i,j),(i',j')} = \trace{\xi_{i:i+\Fshp,j:j+\Fshp}^\top \xi_{i':i'+\Fshp,j':j'+\Fshp}} 
\, .
\end{equation}
Because of this simple remark, we can describe precisely the distribution of $\Abf$ (and, further, $\Bbf$), thanks to the following lemmas. 
We denote by $\phi$ the density of the standard Gaussian distribution and $\Phi$ its cumulative distribution function.

\begin{lemma}[Expectation of rectified Gaussian] 
\label{lemma:expec-relu}
    Let $X \sim \gaussian{\mu}{\tau^2}$, we get the following expectation for the rectified Gaussian $X^+ \defeq \activ{X}$:
    \begin{align*}
        \expec{X^+} = \mu \Phi\left(\frac{\mu}{\tau}\right) + \tau \phi\left(\frac{-\mu}{\tau}\right) 
\, .
    \end{align*}
\end{lemma}

% DGA: proof of this lemma
% MAG : plus long si le gaussian est pas centrée sur 0 sans admettre des trucs.
% https://math.stackexchange.com/questions/1963292/expectation-and-variance-of-gaussian-going-through-rectified-linear-or-sigmoid-f
\begin{proof}
See \cite{beauchamp2018numerical}. 
\qed 
\end{proof}

% MAG : add def of frobrenius norm ?
\begin{lemma}[Law of convolution] 
\label{lemma:conv-gaussian}
    Let $\mbf \in \Reals^{k \times k}$ and $\Fbf \sim \gaussian{0}{\tau^2 \Idt_k}$, then the convolution $\Fbf \star \mbf$ has distribution $\gaussian{0}{(\tau \norm{\mbf}_2)^2}$.
\end{lemma}

% DGA: proof of this lemma
\begin{proof}
First let us remind that:
\begin{align*}
    \Fbf \star \mbf &= \sum_{i,j = 1}^{k} \Fbf_{i,j} \mbf_{i,j} \, .
\end{align*}
This lemma is derived from rapid calculations that proceed as follows, utilizing the fact that the elements $\left(\Fbf_{i,j}\right)_{i,j}$ are independent and identically distributed (i.i.d.):
\begin{align*}
\expec{\Fbf \star \mbf} &= 0 \, ,\\
\var{\Fbf \star \mbf} &= \tau^2 \sum_{i,j = 1}^{k} \mbf_{i,j}^2 = \tau^2 \norm{\mbf}_2^2 \, .\\
\end{align*}
\qed 
\end{proof}

Straightforward computations yields: 

\begin{lemma}[Moments of squared rectified Gaussian] 
\label{lemma:moments-squaredrelu}
Let $X \sim \gaussian{0}{\tau^2}$ of density $f(\cdot)$, we get the following two moments for the squared rectified Gaussian $(X^+)^2$:
    \begin{align*}
        \expec{(X^+)^2} = \frac{\tau^2}{2} \quad \text{and } \var{(X^+)^2} = \frac{5}{4} \tau^4 
    \, .
    \end{align*}
\end{lemma}

\section{Proof of Proposition~\ref{prop:l-hidden-layer}}
\label{sec:proof-coefs}

\paragraph{Notation.}
Before starting the proof, let us recall some notation.
In Section~\ref{sec:setting}, we defined $\abf^{(u)} \defeq \preactiv{\Cbf'}{u}$ the non-rectified activation of layer $u$, $\rbf^{(u)} \defeq \activ{\abf^{(u)}}$ its rectified counterpart. 
We also defined $\Cbf$ the max pooled rectified activation $\mathcal{M} (\Bbf)$.
In this proof, we write $f_u$ the function that maps the input of the $u$-th layer $\rbf^{(u-1)}$ to its output $\rbf^{(u)}$ as follows $f_u(\rbf^{(u-1)}) \defeq \activ{\Wbf^{(u)}\rbf^{(u-1)}}$ with $u \in [L]$. 

Now let us turn to the computation of $\nabla_{\Bbf} f (\Bbf) \in \Reals^{(\Ashp \times \Ashpp) \times 1}$ with $\Bbf \in \Reals^{\Ashp \times \Ashpp}$. 
By the chain rule, we get $\nabla_{\Bbf} f (\Bbf) =~ \nabla_{\Bbf} \left(\Fnn \circ \mathcal{M} \right)(\Bbf) = \nabla_{\Bbf} \mathcal{M} \cdot  \nabla_{\mathcal{M} (\Bbf)} \Fnn$ where: 
\begin{align*}
\nabla_{\Cbf} \Fnn &= \nabla_{\Cbf} f_1 \cdot \nabla_{f_{1}(\Cbf)} \left(\Wbf^{(L)} f_{L-1} \circ \cdots \circ f_{2}\right) \\
&= \left(\nabla_{\Cbf} f_1\right) \cdot \left(\nabla_{\rbf^{(1)}} f_2 \right) \cdots \left(\nabla_{\rbf^{(L-2)}} f_{L-1}\right) \cdot \left(\Wbf^{(L)}\right)
\, .
\end{align*}
We compute $\nabla_{\rbf^{(i)}} f_{i+1} \in \Reals^{d_i \times d_{i+1}}$ as follows:
\begin{align*}
    \nabla_{\rbf^{(i)}} f_{i+1} = \nabla_{\rbf^{(i)}} \activ{\Wbf^{(i+1)} \rbf^{(i)} } = \left(\Wbf^{(i+1)}\right)^\top \nabla_{\abf^{(i+1)}} \activ{\abf^{(i+1)}}
\, .
\end{align*}
Now let us remark that, for all $x\in\Reals^d$, $\partial_i \activ{x}_j = \partial_i \activ{x_j} = \indic{x_j > 0} \indic{i=j}$, since $\Activ$ is applied element-wise.
Thus
\begin{align*}
    \nabla_{\abf^{(i+1)}} \activ{\abf^{(i+1)}} = 
    \begin{pmatrix}
    \indic{\abf^{(i+1)}_1> 0} & 0 & \cdots & 0\\
    0 & \indic{\abf^{(i+1)}_2> 0} & \cdots & 0 \\
    \vdots & \vdots & \ddots & \vdots\\
    0 & 0 & 0 & \indic{\abf^{(i+1)}_{d_{i+1}}> 0}
    \end{pmatrix}
\, .
\end{align*}
Finally, 
\begin{align*}
\nabla_{\rbf^{(i)}} f_{i+1} &= 
\begin{pmatrix}
\Wbf^{(i+1)^\top}_{1, \,:} & \cdots & \Wbf^{(i+1)^\top}_{d_{i+1}, \,:}
\end{pmatrix} 
\nabla_{\abf^{(i+1)}} \activ{\abf^{(i+1)}} \\
&= 
\begin{pmatrix}
\indic{\abf^{(i+1)}_1> 0} \Wbf^{(i+1)^\top}_{1, \,:} & \cdots & \indic{\abf^{(i+1)}_{d_{i+1}}> 0} \Wbf^{(i+1)^\top}_{d_{i+1}, \,:}
\end{pmatrix}
\, .
\end{align*}
Now to compute $\nabla_{\Cbf} \Fnn$ let us start from the end with $\nabla_{\rbf^{(L-2)}} f_{L-1} \Wbf^{(L)^\top}$:
\begin{align*}
    \nabla_{\rbf^{(L-2)}} f_{L-1} \Wbf^{(L)^\top} &= 
    \begin{pmatrix}
    \indic{\abf^{(L-1)}_1> 0} \Wbf^{(L-1)^\top}_{1, \,:} & \cdots & \indic{\abf^{(L-1)}_{d_{L-1}}> 0} \Wbf^{(L-1)^\top}_{d_{L-1}, \,:}
    \end{pmatrix} \Wbf^{(L)^\top} \\
    &=
    \begin{pmatrix}
    \sum_{i_{L-1}=1}^{d_{L-1}} \indic{\abf^{(L-1)}_{i_{L-1}}> 0} \Wbf^{(L-1)^\top}_{i_{L-1}, \,:} \Wbf^{(L)}_{1,i_{L-1}}
    \end{pmatrix}
\, .
\end{align*}

Similarly, 
\begin{align*}
\nabla_{\rbf^{(L-3)}} f_{L-2} &=
\begin{pmatrix}
\indic{\abf^{(L-2)}_1> 0} \Wbf^{(L-2)^\top}_{1, \,:} & \cdots & \indic{\abf^{(L-2)}_{d_{L-2}}> 0} \Wbf^{(L-2)^\top}_{d_{L-2}, \,:}
\end{pmatrix} \, ,
\end{align*}
so the next computation gives
\begin{align*}
&\nabla_{\rbf^{(L-3)}} f_{L-2} \nabla_{\rbf^{(L-2)}} f_{L-1} \Wbf^{(L)^\top} \\
&=\nabla_{\rbf^{(L-3)}} f_{L-2}
\begin{pmatrix}
\sum_{i_{L-1}=1}^{d_{L-1}} \indic{\abf^{(L-1)}_{i_{L-1}}> 0} \Wbf^{(L-1)^\top}_{i_{L-1}, \,:} \Wbf^{(L)}_{1,i_{L-1}}
\end{pmatrix} \\
&= 
\nabla_{\rbf^{(L-3)}} f_{L-2}
\begin{pmatrix}
\sum_{i_{L-1}=1}^{d_{L-1}} \indic{\abf^{(L-1)}_{i_{L-1}}> 0} \Wbf^{(L-1)^\top}_{i_{L-1}, \,1} \Wbf^{(L)}_{1,i_{L-1}} \\
\vdots \\
\sum_{i_{L-1}=1}^{d_{L-1}} \indic{\abf^{(L-1)}_{i_{L-1}}> 0} \Wbf^{(L-1)^\top}_{i_{L-1}, \, d_{L-2}} \Wbf^{(L)}_{1,i_{L-1}}
\end{pmatrix} \\
&=
\begin{pmatrix}
\sum_{i_{L-2},i_{L-1}=1}^{d_{L-2}, d_{L-1}} \indic{\abf^{(L-2)}_{i_{L-2}}, \, \abf^{(L-1)}_{i_{L-1}}> 0} \Wbf^{(L-2)^\top}_{i_{L-2}, \,:} \Wbf^{(L-1)^\top}_{i_{L-1}, \,i_{L-2}} \Wbf^{(L)}_{1,i_{L-1}}
\end{pmatrix} 
\, .
\end{align*}

Now we can conclude by straightforward induction:
\begin{align*}
&\nabla_{\Cbf} \Fnn 
\\ &=
\begin{pmatrix}
\sum_{i_1, \ldots, i_{L-1} = 1}^{d_1, \ldots, d_{L-1}} \indic{\abf^{(1)}_{i_1}, \ldots, \abf^{(L-1)}_{i_{L-1}}>0} (\Wbf_{i_1,:}^{(1)})^\top (\Wbf_{i_2, i_1}^{(2)})^\top \ldots (\Wbf_{i_{L-1},i_{L-2}}^{(L-1)})^\top \Wbf_{1,i_{L-1}}^{(L)}
\end{pmatrix}
\, .
%&\in \Reals^{(\Mshp \times \Mshpp) \times 1}
\end{align*}
% DGA: generalize the following to k x k, this is written for 2 x 2
% MAG : need to make clear the "reshaped rectified activation map".
Now to finish we need to compute $\nabla_{\Bbf} \mathcal{M}(\Bbf)$, to do so, we suppose that the convolutional layer $\mathcal{C}$ returns an reshaped rectified activation map $\Bbf$ such that we get the following gradient:
% \begin{align*}
% \begin{pmatrix}
% \indic{\max{\Bbf_{1,1}, \Bbf_{1,2}, \Bbf_{2,1}, \Bbf_{2,2}} = \Bbf_{1,1}} & 0 & \cdots & 0 \\
% \indic{\max{\Bbf_{1,1}, \Bbf_{1,2}, \Bbf_{2,1}, \Bbf_{2,2}} = \Bbf_{1,2}} & 0 & \cdots & 0 \\
% \indic{\max{\Bbf_{1,1}, \Bbf_{1,2}, \Bbf_{2,1}, \Bbf_{2,2}} = \Bbf_{2,1}} & 0 & \cdots & 0 \\
% \indic{\max{\Bbf_{1,1}, \Bbf_{1,2}, \Bbf_{2,1}, \Bbf_{2,2}} = \Bbf_{2,2}} & 0 & \cdots & 0 \\
% 0 & \indic{\max{\Bbf_{1,3}, \Bbf_{1,4}, \Bbf_{2,3}, \Bbf_{2,4}} = \Bbf_{1,3}} & \cdots & 0\\
% 0 & \indic{\max{\Bbf_{1,3}, \Bbf_{1,4}, \Bbf_{2,3}, \Bbf_{2,4}} = \Bbf_{1,4}} & \cdots & 0\\
% 0 & \indic{\max{\Bbf_{1,3}, \Bbf_{1,4}, \Bbf_{2,3}, \Bbf_{2,4}} = \Bbf_{2,3}} & \cdots & 0\\
% 0 & \indic{\max{\Bbf_{1,3}, \Bbf_{1,4}, \Bbf_{2,3}, \Bbf_{2,4}} = \Bbf_{2,4}} & \cdots & 0\\
% \vdots & \vdots & & \vdots \\
% 0 & 0 & \cdots & \indic{\max{\Bbf_{\Ashp-1,\Ashpp-1}, \Bbf_{\Ashp,\Ashpp-1}, \Bbf_{\Ashp-1,\Ashpp}, \Bbf_{\Ashp,\Ashpp}} = \Bbf_{\Ashp-1,\Ashpp-1}}\\
% 0 & 0 & \cdots & \indic{\max{\Bbf_{\Ashp-1,\Ashpp-1}, \Bbf_{\Ashp,\Ashpp-1}, \Bbf_{\Ashp-1,\Ashpp}, \Bbf_{\Ashp,\Ashpp}} = \Bbf_{\Ashp,\Ashpp-1}}\\
% 0 & 0 & \cdots & \indic{\max{\Bbf_{\Ashp-1,\Ashpp-1}, \Bbf_{\Ashp,\Ashpp-1}, \Bbf_{\Ashp-1,\Ashpp}, \Bbf_{\Ashp,\Ashpp}} = \Bbf_{\Ashp-1,\Ashpp}}\\
% 0 & 0 & \cdots & \indic{\max{\Bbf_{\Ashp-1,\Ashpp-1}, \Bbf_{\Ashp,\Ashpp-1}, \Bbf_{\Ashp-1,\Ashpp}, \Bbf_{\Ashp,\Ashpp}} = \Bbf_{\Ashp,\Ashpp}}\\
% \end{pmatrix}   
% \end{align*}
% We rewrite this block matrix as follows:
\begin{align*}
\begin{pmatrix}
\Dbf_1 & 0 & 0\\
0 & \Dbf_2 & 0 \\
\vdots &  & \vdots \\
0 & \cdots & \Dbf_{\Mshp \times \Mshpp}
\end{pmatrix} \in \Reals^{((k')^2 \times \Mshp \times \Mshpp) \times (\Mshp \times \Mshpp)} = \Reals^{(\Ashp \times \Ashpp) \times (\Mshp \times \Mshpp)}
\end{align*}
with $\Dbf_i \in \Reals^{(k')^2 \times 1}$ being the $i$-th block of indicator functions defined as follows:
\begin{align*}
    \Dbf_i \defeq 
    \begin{pmatrix}
     \indic{\max{\mbf_i} = \mbf_{i,(1,1)}} \\
      \indic{\max{\mbf_i} = \mbf_{i,(1,2)}} \\
    \vdots \\
    \indic{\max{\mbf_i} = \mbf_{i,(2,1)}} \\
    \vdots \\
    \indic{\max{\mbf_i} = \mbf_{i,(k',k')}}
    \end{pmatrix} \, ,
\end{align*}
where $\mbf_{i, (p,q)} \defeq \left( \mbf_i \right)_{p,q}$ and $\mbf_i \defeq \Bbf_{\Pshp (i'-1)+1:\Pshp i',\Pshp (j'-1)+1:\Pshp j'} \in \Reals_+^{k' \times k'}, \, (i', j') \defeq ((i-1) / k' + 1, \, (i-1 \mod k') + 1 )$ is the $i$-th patch of $\Bbf$ (the patch are ordered from left to right, starting from the top of $\Bbf$). Also note that $\cdot / \cdot$ is the integer division.

Finally, we can compute $\nabla_{\Bbf} f (\Bbf)$ as follows: 
\begin{align}
\begin{split}
    \nabla_{\Bbf} f (\Bbf) 
    &= \nabla_{\Bbf} \mathcal{M} \cdot  \nabla_{\mathcal{M} (\Bbf)} \Fnn \\
    &=
    \begin{pmatrix}
    \Dbf_1 \rho_1 \\
    \vdots \\
    \Dbf_{\Mshp \times \Mshpp} \rho_{h' \times w'}
    \end{pmatrix} \\
    &\in \Reals^{(\Ashp \times \Ashpp) \times 1}
\end{split}
\end{align}
where $\rho_b \in \Reals$, $b \in [h' \times w']$ is defined as follows:
\begin{align*}
    \rho_b \defeq \sum_{i_1, \ldots, i_{L-1} = 1}^{d_1, \ldots, d_{L-1}} \indic{\abf^{(1)}_{i_1}, \ldots, \abf^{(L-1)}_{i_{L-1}}>0} (\Wbf_{i_1, b}^{(1)})^\top (\Wbf_{i_2, i_1}^{(2)})^\top \ldots (\Wbf_{i_{L-1},i_{L-2}}^{(L-1)})^\top \Wbf_{1,i_{L-1}}^{(L)} \, .
\end{align*}

To compute $\alpha$, we simply take the average of the components of the previous display, $\nabla_{\Bbf} f (\Bbf)$, which yields
\begin{align*}
    \alpha &= \frac{1}{\Ashp \Ashpp} \sum_{k = 1}^{(k')^2} \sum_{j=1}^{\Mshp \times \Mshpp} \Dbf_{j,k} \rho_j \\
    &= \frac{1}{\Ashp \Ashpp} \sum_{j=1}^{\Mshp \times \Mshpp} \left(\sum_{k = 1}^{(k')^2} \Dbf_{j,k} \right) \rho_j \\
    &= \frac{1}{\Ashp \Ashpp} \sum_{j=1}^{\Mshp \times \Mshpp} \rho_j \, ,
\end{align*}
% DGA: beware, this also changes when going to general size
% MAG : sum pas egale à 1 en maths mais la convention choisi par Autograd, ça se somme à 1. (on prend la premiere realisation du max).
since $\sum_{k = 1}^{(k')^2} \Dbf_{j,k} = 1$ with fixed $j$ (it is also the case in practice when using automatic differentiation) and $\Dbf_{j,k} \defeq \left(\Dbf_{j}\right)_k$.
\qed 

%%%%%%%%%%%%%%%%%%%%%%%%%%%%%%%%%%%%%%%%%%%%%%%%%%%%%%%%%%%%%%%%%%%%%%%%%%

\section{Proof of Theorem~\ref{th:expected-gradcam}}
\label{sec:proof-main}

% assume that h mod 2 = 0
% \alpha = \alpha^{(1)}

Before jumping into the proof, let us remark that the left-hand side of Eq.~\eqref{eq:gradcam-lower-bound} is non-negative, thus there is nothing to prove if $\norm{\mbf}_2=0$. 
Thus we will assume that $\norm{\mbf}_2>0$ from now on. 

% First let us remind that the dense layer in our model is initialized as follows $\Wbf^{(n)} \stackrel{i.i.d.}{\sim} \gaussian{0}{\sigma^2 \Idt}$ ($\Wbf^{(n)}$ is connected to the $n$-th activation map) and in the convolutional layer $\Fbf^{(n)} \stackrel{i.i.d.}{\sim} \gaussian{0}{\sigma^2 \Idt}$ with $\Fbf^{(n)} \indep \Wbf^{(n)}$ and $n \in [K]$.
%
\paragraph{Separating the randomness. }
There are two sources of randomness: the weights of the filters $\Fbf = (\Fbf^{(1)},\ldots,\Fbf^{(V)})$ and the coefficients of the linear layer $\Wbf$. 
We start this proof by dissociating these two sources of randomness. 
More precisely, following Definition~\ref{def:gradcam}, and using the law of total expectation, we compute the expected GradCAM heatmap at coordinates $(i,j)$ as 
\begin{equation}
\label{eq:cond-expec}
\expec{\activ{\sum_{q=1}^{V} \alpha_q \Bbf_{i,j}^{(q)}}} = \expec{\condexpec{\activ{\sum_{q=1}^{V} \alpha_q \Bbf_{i,j}^{(q)}}}{\Fbf}}
\, .
\end{equation}
Using the computed GradCAM coefficient in Proposition~\ref{prop:l-hidden-layer} with $L=1$, we get $\alpha_q = \frac{1}{hw} \sum_{p=1}^{h' w'} \Wbf_p^{(q)}$. 
Since we assume that the weights of the linear layer are i.i.d. $\gaussian{0}{\tau^2}$, 
%%%%%%%%%%%%
for the upper part, and $0$ for the lower part,
%%%%%%%%%%%
$\alpha_q$ is a centered Gaussian with variance $(\frac{\tau}{h w})^2 \frac{h' w'}{2}$.
Now, we recall that $\Bbf$ does not depend on $\Wbf$. 
Therefore, conditionally to $\Fbf^{(q)}$, $\alpha_q \Bbf_{i,j}^{(q)}$ is a centered Gaussian with variance 
%%%%%%%%%%%%%%%%
%$\left(\Bbf_{i,j}^{(q)}\right)^2 (\frac{\tau}{h w})^2 h' w'$. 
\[
%\frac{1}{(hw)^2}\sum_{p=1}^{h'w'/2}
\left(\Bbf_{i,j}^{(q)}\right)^2 \left(\frac{\tau}{h w}\right)^2 \frac{h' w'}{2}
\, .
\]
%%%%%%%%%%%%%%%%
We deduce that 
\begin{equation}
\label{eq:sum-distribution}
\sum_{q=1}^{V} \alpha_q \Bbf_{i,j}^{(q)} | \Fbf \sim \gaussian{0}{\sum_{q=1}^{V} \left(\Bbf_{i,j}^{(q)}\right)^2 \left(\frac{\tau}{h w}\right)^2 \frac{h' w'}{2}}
\, .
\end{equation}

\paragraph{Computing expectations. }
Since, conditionally to $\Fbf$, $\sum_q \alpha_q \Bbf^{(q)}$ is a a centered Gaussian with variance given by Eq.~\eqref{eq:sum-distribution}, we can use Lemma~\ref{lemma:expec-relu} to compute 
% DGA: for later, I would vote removing any reference to \phi and only give exact values
\begin{align*}
\condexpec{\activ{\sum_{q=1}^{V} \alpha_q \Bbf_{i,j}^{(q)}}}{\Fbf} &= \sqrt{\left(\frac{\tau}{h w}\right)^2 \frac{h' w'}{2} \sum_{q=1}^{V} \left(\Bbf_{i,j}^{(q)}\right)^2} \phi(0) \\
&= \frac{\tau}{2 h w \sqrt{\pi} } \sqrt{h' w'\sum_{q=1}^{V} \left(\Bbf_{i,j}^{(q)}\right)^2} 
\, .
\end{align*}
Coming back to Eq.~\eqref{eq:cond-expec}, we deduce that 
\begin{equation}
\label{eq:key-computation}
\expec{\activ{\sum_{q=1}^{V} \alpha_q \Bbf_{i,j}^{(q)}}} = \frac{\tau}{2 h w \sqrt{\pi}} \sqrt{h' w'} \expec{\sqrt{\sum_{q=1}^{V} \left(\Bbf_{i,j}^{(q)}\right)^2}} 
\, .
\end{equation}

\paragraph{Lower bound. }
Up to the best of our knowledge, there is no closed-form expression for Eq.~\eqref{eq:key-computation}, and we now proceed to find a lower bound to this expression. 
Let $Y \defeq  \sum_{q=1}^{V} \left(\Bbf_{i,j}^{(q)}\right)^2$.
By Lemma~\ref{lemma:conv-gaussian} and~\ref{lemma:moments-squaredrelu}, $\expec{Y} = V \frac{(\tau \norm{\mbf}_2)^2}{2}$ and $\var{Y} = \frac{5}{4} V (\tau \norm{\mbf}_2)^4$ with $\mbf \defeq \Img_{i:i+k-1, j:j+k-1}$ a patch at index $(i,j)$ in the image $\Img$.
By Chebyshev's inequality, for any $t>0$,
\begin{equation}
\label{eq:chebyshev}
\proba{Y \geq \expec{Y} - t} \geq 1 - \frac{\var{Y}}{t^2}
\, .
\end{equation}
Let us set $t = \frac{\expec{Y}}{2}>0$ in the previous display, in this way 
\[
\frac{\var{Y}}{t^2}= \frac{5}{4}V\tau^4\norm{m}^4 \cdot \left(\frac{V^2 \tau^4\norm{m}^4}{2^2\cdot 2^2}\right)^{-1} = \frac{20}{V}
\, .
\]
Since $\norm{\mbf}_2 > 0$, 
\[
\proba{Y \geq \frac{\expec{Y}}{2}} \geq 1 - \frac{20}{V}
\, .
\]
We now conclude the proof writing
\begin{align*}
    &\expec{\activ{\sum_{q=1}^{V} \alpha_q \Bbf_{i,j}^{(q)}}} \\
%    &= \frac{\tau}{\sqrt{2\pi} h w } \sqrt{h' w'} \expec{\sqrt{\sum_{n=1}^{V} \left(\Bbf_{i,j}^{(n)}\right)^2}} \\
    &= \frac{\tau}{2 h w \sqrt{\pi}} \sqrt{h' w'} \left( \condexpec{\sqrt{Y}}{Y \geq \frac{\expec{Y}}{2}} \proba{Y \geq \frac{\expec{Y}}{2}} \right. \\
    &\qquad \left. + \condexpec{\sqrt{Y}}{Y < \frac{\expec{Y}}{2}} \proba{Y < \frac{\expec{Y}}{2}}  \right) \\
    &\geq \frac{\tau}{2 h w \sqrt{\pi}} \sqrt{h' w'} \condexpec{\sqrt{Y}}{Y \geq \frac{\expec{Y}}{2}} \proba{Y \geq \frac{\expec{Y}}{2}} \qquad \text{(since $Y>0$)} \\
    &\geq \frac{\tau}{2 h w \sqrt{\pi}} \sqrt{h' w'} \expec{\sqrt{\frac{\expec{Y}}{2}}} \left(1 -\frac{20}{V}\right) \\
    &= \frac{\sqrt{V h' w'}}{4 \sqrt{\pi}} \frac{\tau^2}{h w} \left(1 -\frac{20}{V}\right) \norm{m}_2 
\,  .
\end{align*}
\qed
%%%%%%%%%%%%%%%%%%%%%%%%%%%%%%%%%%%%%%%%%%%%%%%%%%%%%%%%%%%%%%%%%%%%%%%%%%%%%%%%%%
\section{Additional Experiments}
In this section, we present an additional set of experiments on STACK-MIX (Figure~\ref{fig:qualitative-results-2}).

\begin{figure}[h]
    \centering
    \begin{tabular}{l@{\hskip 0.5cm}llll}
\rotatebox[origin=c]{90}{Input image} & \includegraphics[width=.2\linewidth,valign=m]{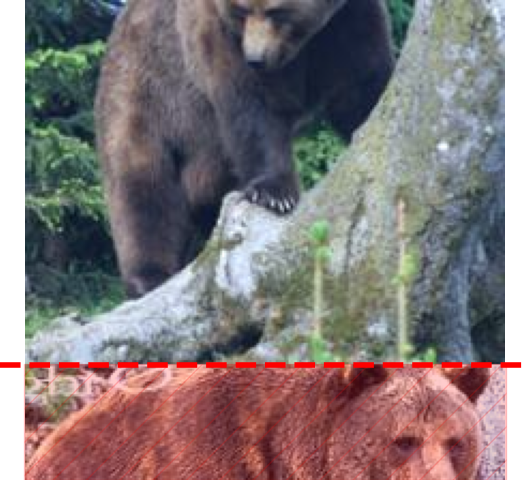} & \includegraphics[width=.2\linewidth,valign=m]{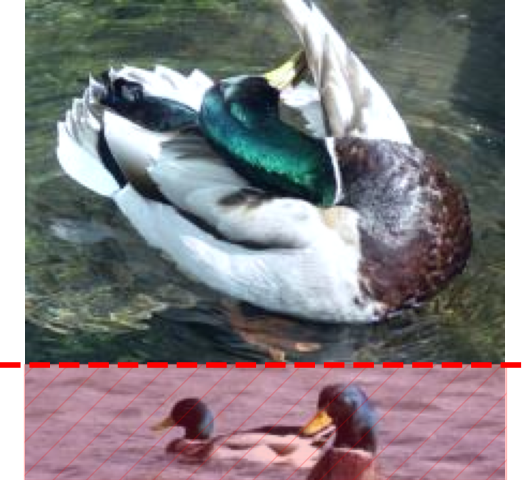} & \includegraphics[width=.2\linewidth,valign=m]{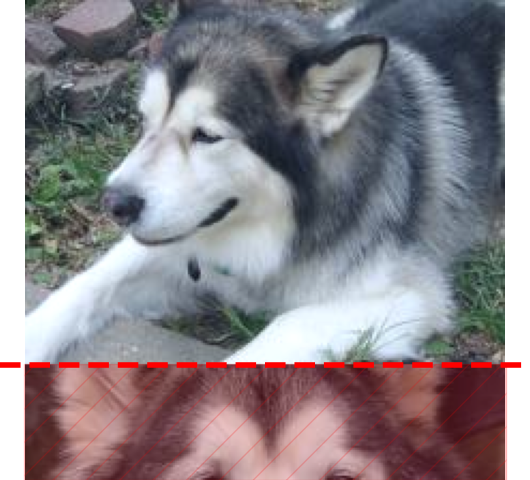}& \includegraphics[width=.2\linewidth,valign=m]{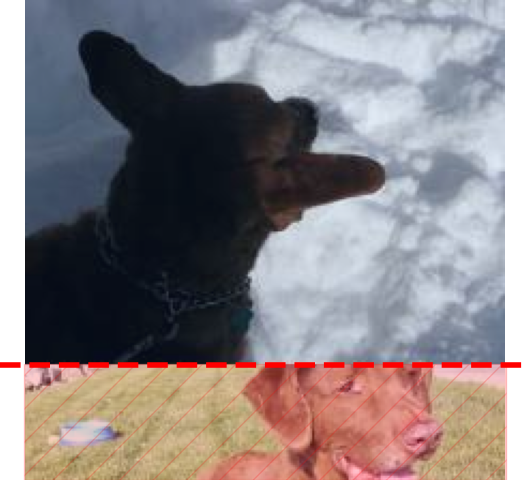}\\[1cm]
\rotatebox[origin=c]{90}{GradCAM} & \includegraphics[width=.2\linewidth,valign=m]{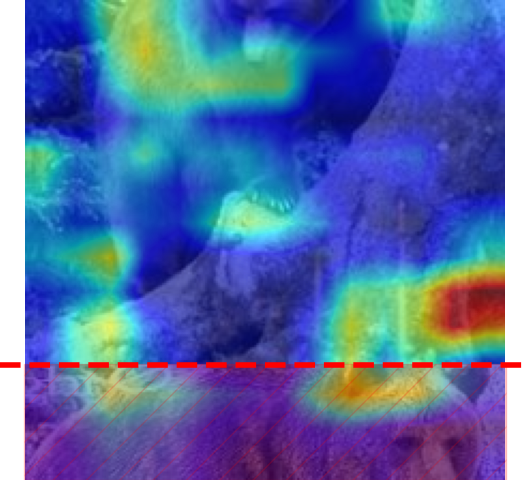} & \includegraphics[width=.2\linewidth,valign=m]{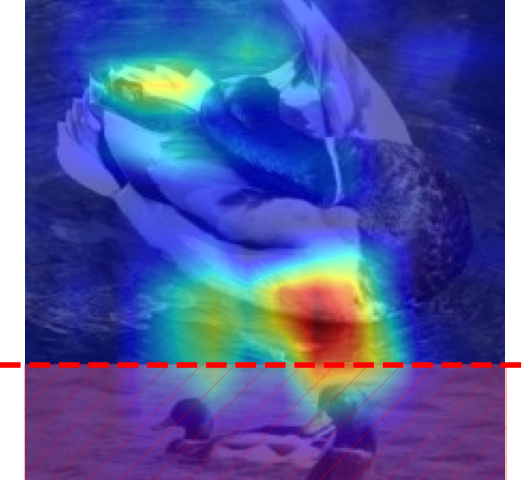} & \includegraphics[width=.2\linewidth,valign=m]{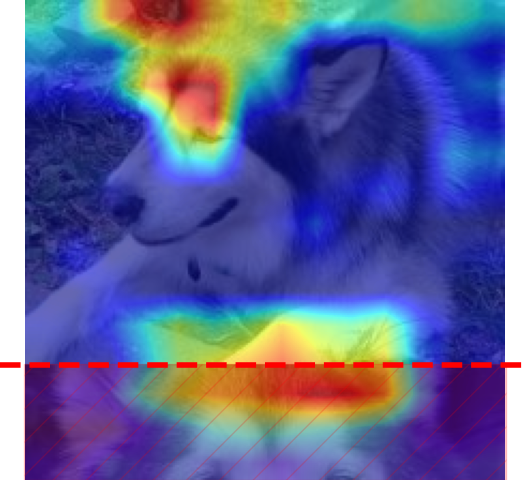}& \includegraphics[width=.2\linewidth,valign=m]{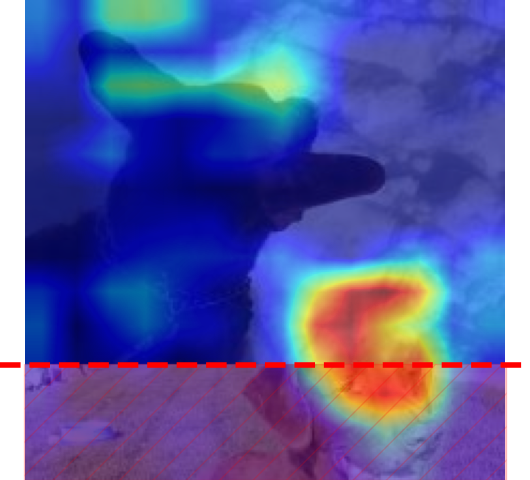}\\[1cm]
\rotatebox[origin=c]{90}{GradCAM++} & \includegraphics[width=.2\linewidth,valign=m]{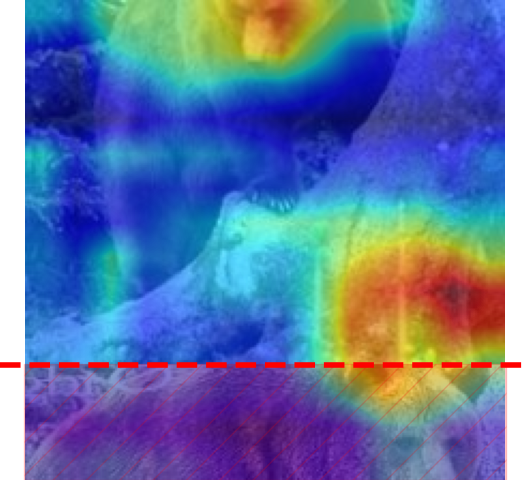} & \includegraphics[width=.2\linewidth,valign=m]{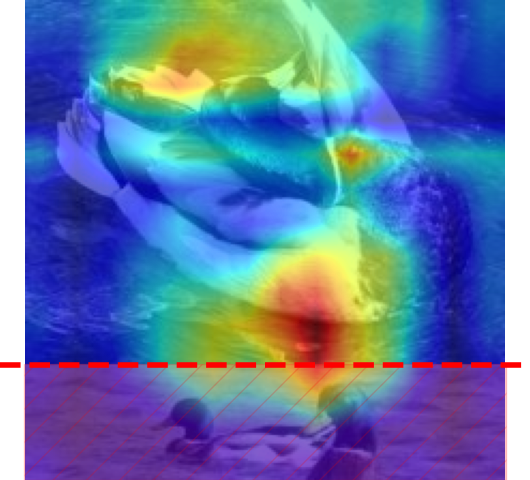} & \includegraphics[width=.2\linewidth,valign=m]{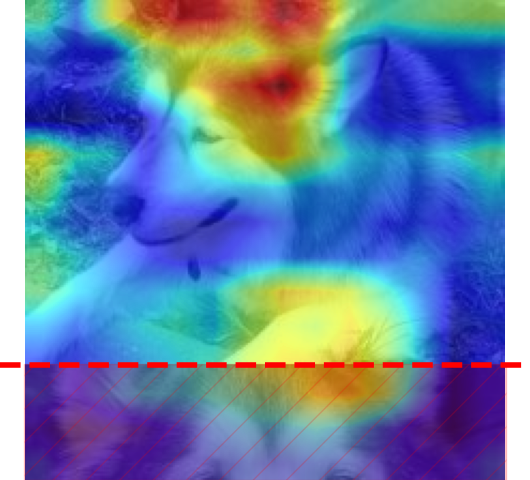}& \includegraphics[width=.2\linewidth,valign=m]{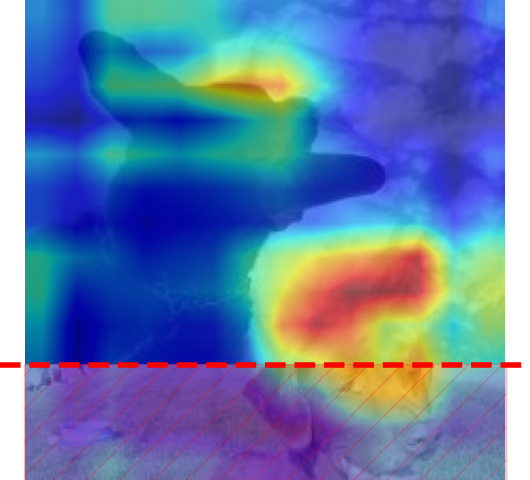}\\[1cm]
\rotatebox[origin=c]{90}{XGradCAM} & \includegraphics[width=.2\linewidth,valign=m]{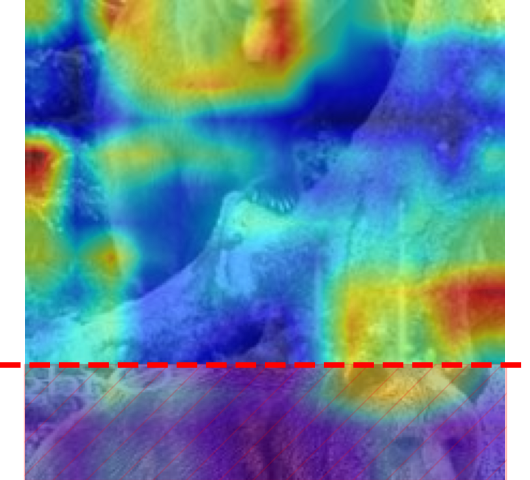} & \includegraphics[width=.2\linewidth,valign=m]{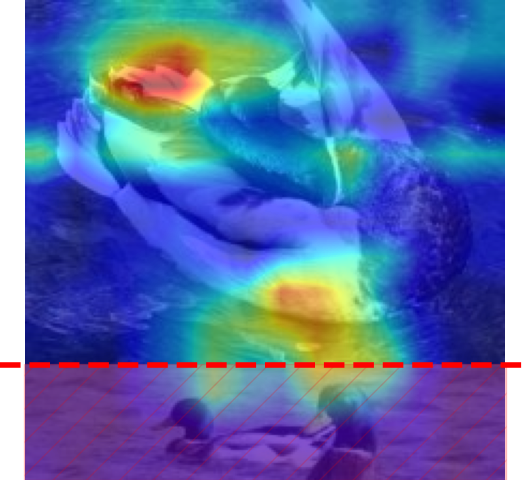} & \includegraphics[width=.2\linewidth,valign=m]{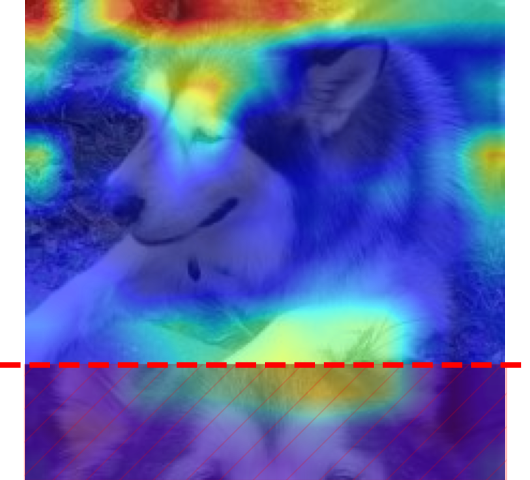}& \includegraphics[width=.2\linewidth,valign=m]{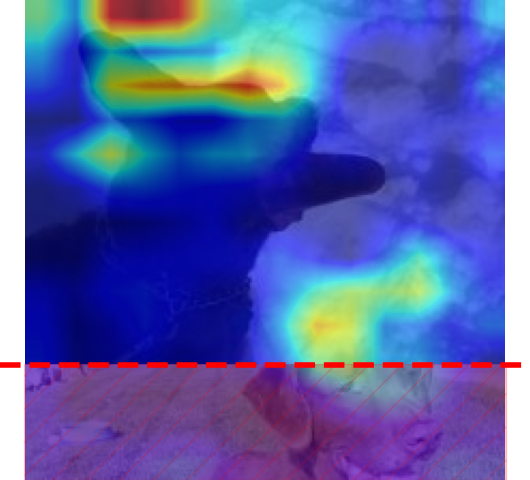}\\[1cm]
\rotatebox[origin=c]{90}{ScoreCAM} & \includegraphics[width=.2\linewidth,valign=m]{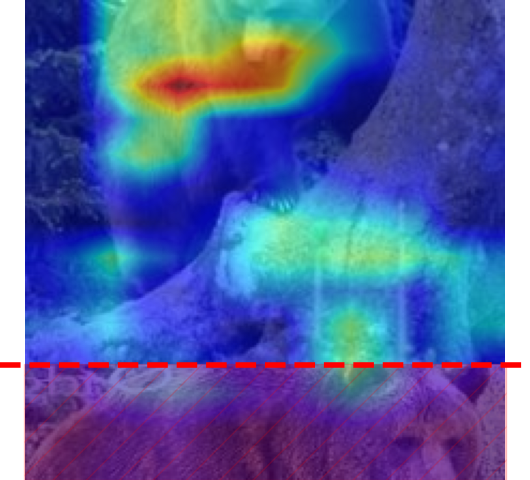} & \includegraphics[width=.2\linewidth,valign=m]{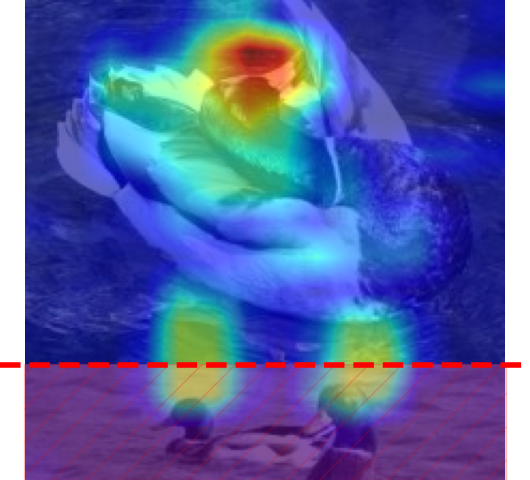} & \includegraphics[width=.2\linewidth,valign=m]{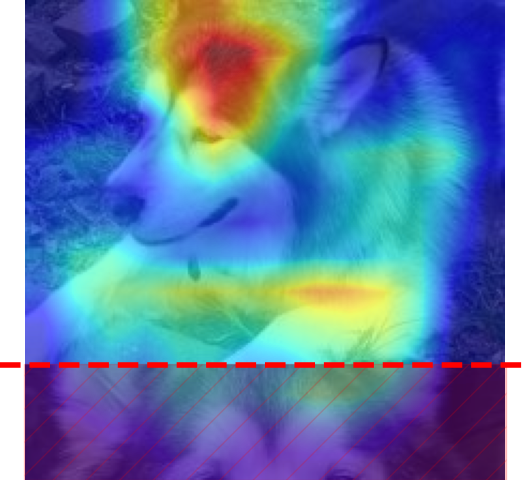}& \includegraphics[width=.2\linewidth,valign=m]{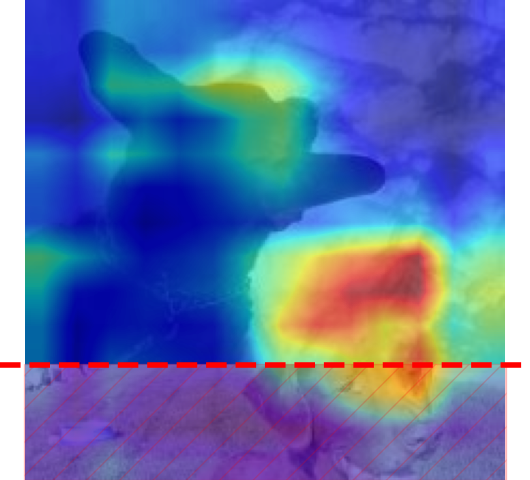}\\[1cm]
\rotatebox[origin=c]{90}{Opti-CAM} & \includegraphics[width=.2\linewidth,valign=m]{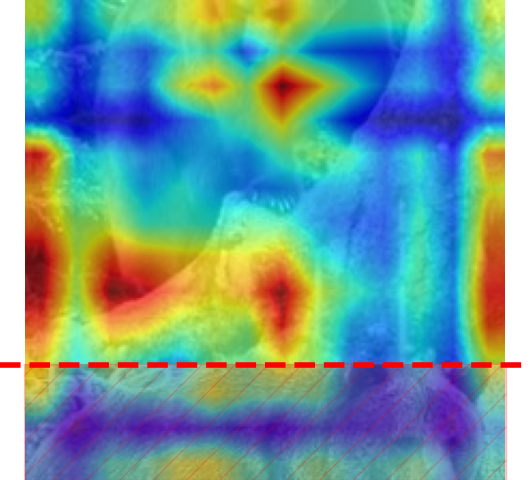} & \includegraphics[width=.2\linewidth,valign=m]{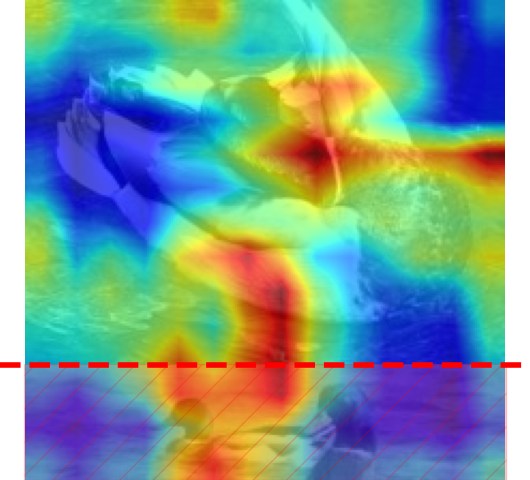} & \includegraphics[width=.2\linewidth,valign=m]{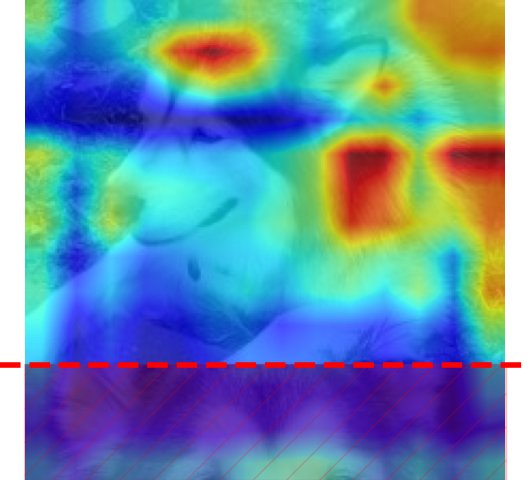}& \includegraphics[width=.2\linewidth,valign=m]{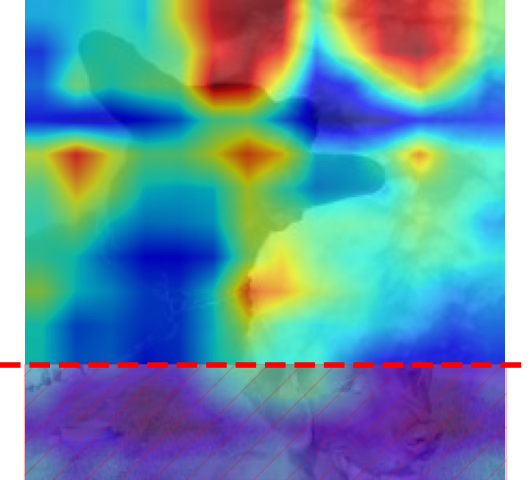}\\[1cm]
\rotatebox[origin=c]{90}{AblationCAM} & \includegraphics[width=.2\linewidth,valign=m]{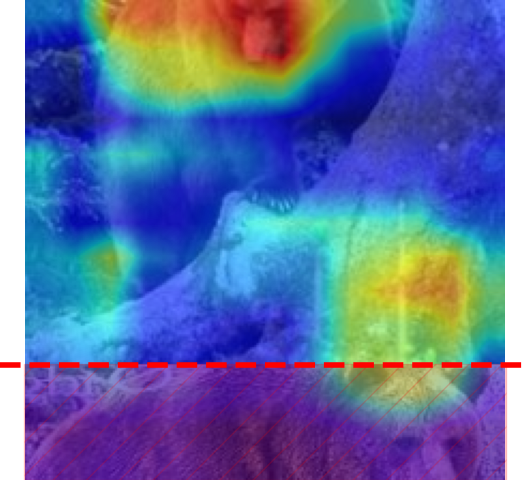} & \includegraphics[width=.2\linewidth,valign=m]{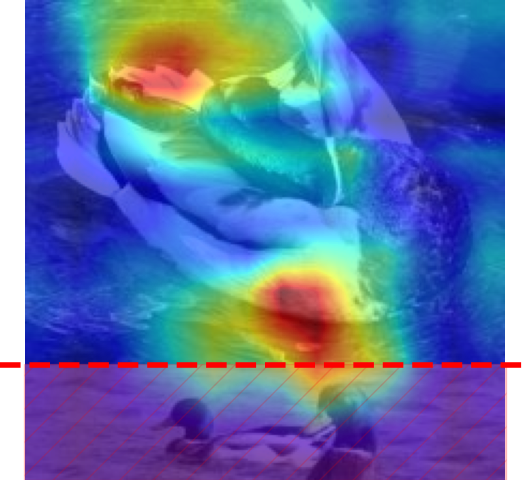} & \includegraphics[width=.2\linewidth,valign=m]{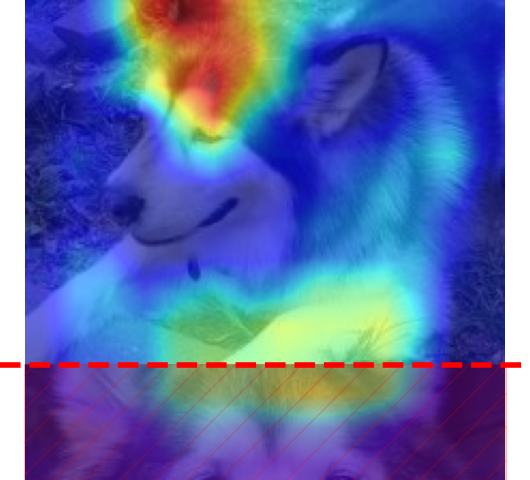}& \includegraphics[width=.2\linewidth,valign=m]{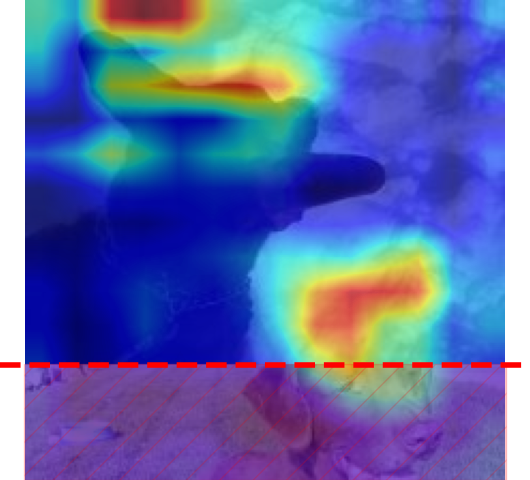}\\[1cm]
\rotatebox[origin=c]{90}{EigenCAM} & \includegraphics[width=.2\linewidth,valign=m]{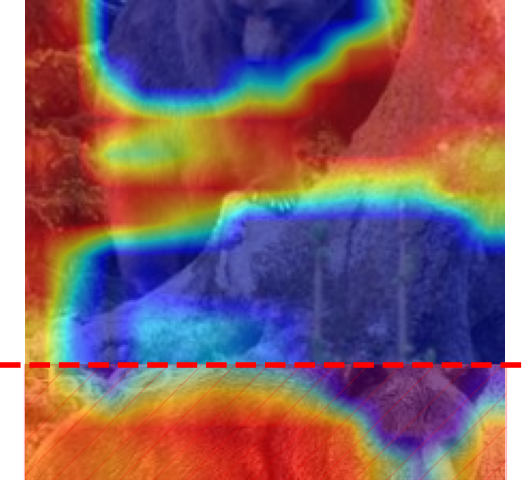} & \includegraphics[width=.2\linewidth,valign=m]{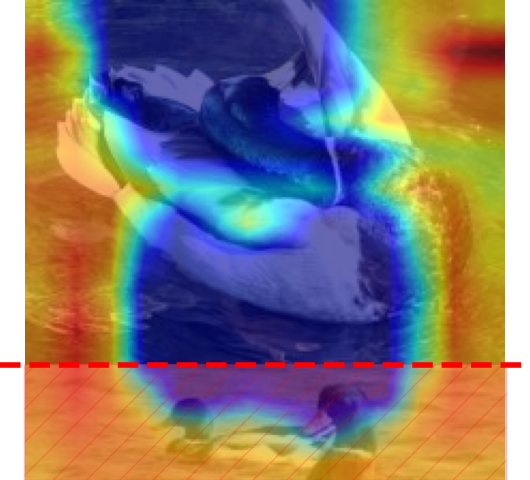} & \includegraphics[width=.2\linewidth,valign=m]{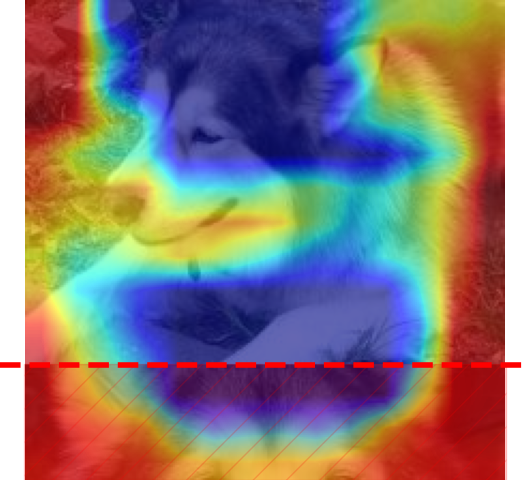}& \includegraphics[width=.2\linewidth,valign=m]{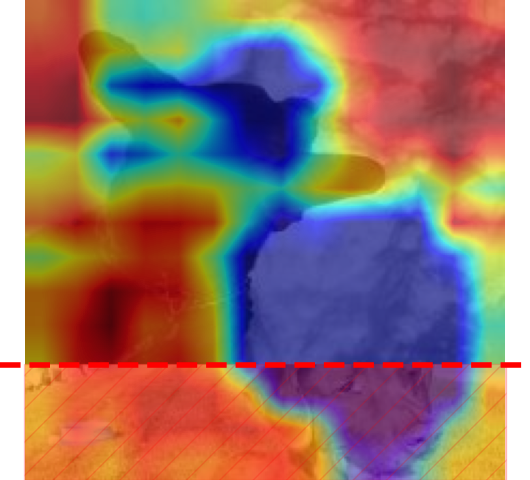}\\[1cm]
\rotatebox[origin=c]{90}{HiResCAM} & \includegraphics[width=.2\linewidth,valign=m]{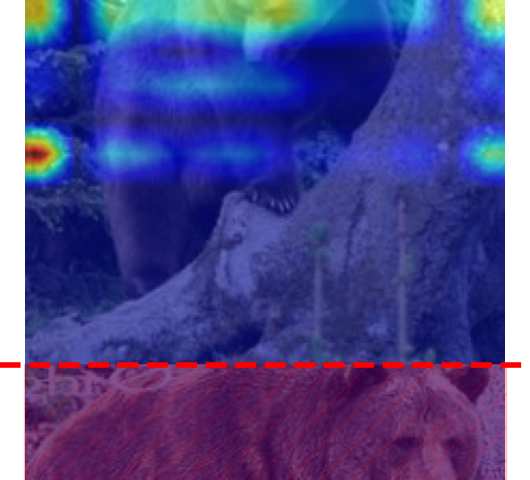} & \includegraphics[width=.2\linewidth,valign=m]{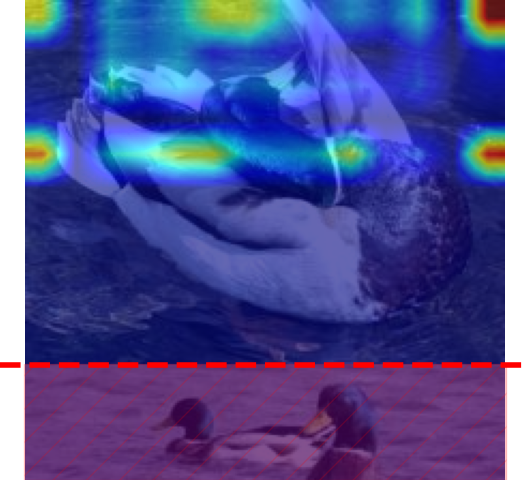} & \includegraphics[width=.2\linewidth,valign=m]{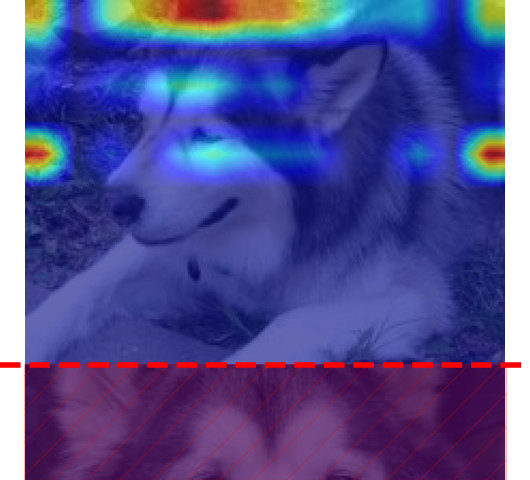}& \includegraphics[width=.2\linewidth,valign=m]{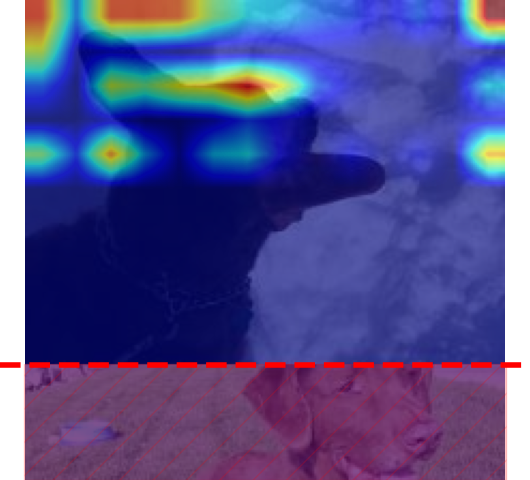}
    \end{tabular}

    \caption{Saliency maps given by the considered CAM-based methods for $\abrev{VGG}$. With the notable exception of HiresCAM, they all highlight parts of images from $\textsc{STACK-MIX}$ which are unseen by the network (this is denoted by the red, rectangular shape in the lower part of the image).}
\label{fig:qualitative-results-2}
\end{figure}
%%%%%%%%%%%%%%%%%%%%%%%%%%%%%%%%%%%%%%%%%%%%%%%%%%%%%%%%%%%%%%%%%%%%%%%%%%%%%%%